\documentclass{article}

 \usepackage[main, final]{neurips_2026}

\usepackage[utf8]{inputenc} 
\usepackage[T1]{fontenc}    
\usepackage{hyperref}       
\usepackage{url}            
\usepackage{booktabs}       
\usepackage{amsfonts}       
\usepackage{nicefrac}       
\usepackage{microtype}      
\usepackage{xcolor}         
\usepackage{amsmath}
\usepackage{amsthm}
\usepackage{graphicx}
\usepackage{bm}
\usepackage{subcaption}
\usepackage{etoc}

\newcommand{\best}[1]{{\boldmath\textbf{#1}}}
\newtheorem{proposition}{Proposition}[section]

\makeatletter
\newcommand{\seq}[1]{%
  \ifcat\noexpand#1a
    \mathbf{#1}%
  \else
    \bm{#1}%
  \fi
}
\makeatother

\usepackage{xspace}

\usepackage{algorithmicx}
\usepackage[ruled,vlined]{algorithm2e}

\usepackage{listings}

\definecolor{codecomment}{RGB}{50,120,120}
\definecolor{codefunction}{RGB}{220,30,110}
\definecolor{codeoperator}{RGB}{20,20,220}

\lstdefinestyle{pytorch}{
    language=Python,
    basicstyle=\ttfamily\small,
    commentstyle=\color{codecomment},
    keywordstyle=\color{codefunction},
    showstringspaces=false,
    columns=fullflexible,
    keepspaces=true,
    breaklines=true,
    aboveskip=0.6em,
    belowskip=0.2em,
    morekeywords={
        sample_t_r,
        randn_like,
        jvp,
        stopgrad,
        metric
    }
}

\usepackage{float}

\newcommand{\EqF}{\textsc{EqF}\xspace}
\newcommand{\ones}{\mathbf{1}}
\newcommand{\zeros}{\mathbf{0}}

\newcommand{\authorblock}[2]{%
  \begin{minipage}[t]{0.23\textwidth}
    \centering
    {\makebox[0pt][c]{\bfseries #1}}\\
    {\normalfont #2}
  \end{minipage}%
}

\title{Equilibrium Forcing: Adaptive Video Generation Without Noise Conditioning} 

\author{%
\makebox[\textwidth][c]{%
  \authorblock{
    Hansen Jin Lillemark\textsuperscript{1}
    \hspace{-0.5em}
    \thanks{Equal Contribution. Correspondence to
    \texttt{hlillemark@ucsd.edu}, \texttt{a5rojas@ucsd.edu}}
  }{}
  \hfill
  \authorblock{
    Alex Rojas\textsuperscript{1}\footnotemark[1]
  }{}
  \hfill
  \authorblock{
    Zachary Novack\textsuperscript{1}
  }{}
  \hfill
  \authorblock{
    Runqian Wang\textsuperscript{2}
  }{}
}\\[0.5em]
\makebox[\textwidth][c]{%
  \authorblock{
    Yilun Du\textsuperscript{3}
  }{}
  \hfill
  \authorblock{
    Yian Ma\textsuperscript{1}
  }{}
  \hfill
  \authorblock{
    Taylor Berg-Kirkpatrick\textsuperscript{1}
  }{}
  \hfill
  \authorblock{
    Rose Yu\textsuperscript{1}
  }{}
}\\[0.5em]
\makebox[\textwidth][c]{%
  \textsuperscript{1}UC San Diego
  \qquad
  \textsuperscript{2}MIT
  \qquad
  \textsuperscript{3}Harvard University
}
}

\begin{document}

\maketitle 

\vspace{-2em}

\begin{abstract}

Standard autoregressive video generation algorithms based on Diffusion and Flow Matching rely on rigid training objectives and static sampling schedules, limiting inference procedures from adapting to the data. We introduce Equilibrium Forcing (\EqF), a simplified framework for video denoising generative models \textit{without} noise level conditioning. \EqF pioneers modular training- and inference-time designs for noise-unconditional generation that decouple learning the denoising field from sampling. This flexibility allows for inference-time algorithms that operate in a closed loop by adapting to feedback from the sample, improving video quality and consistency on challenging autoregressive video generation benchmarks. Extensive analysis elucidates exactly how removing the noise level conditioning enables \EqF's data-dependent inference properties to surpass the performance of standard noise level-conditional denoising video methods. Project page: \href{https://equilibriumforcing.github.io/}{https://equilibriumforcing.github.io/}

\end{abstract}

\section{Introduction}

Video generation has emerged as a central problem in world modeling, simulation, and interactive content creation, where models must produce temporally coherent video rollouts under limited inference budgets. Algorithms for autoregressive inference generate future frames conditional on previously generated segments, leading to accumulation of local errors over time \citep{huang2025selfforcingbridgingtraintest}. We posit that effective samplers should operate in a closed loop with feedback from the state of the generation procedure, adjusting the sampling procedure to balance video refinement and the remaining budget.

Denoising generative models have become the standard approach for autoregressive video generation, but they currently rely on rigid inference schedules for noise level conditioning that preclude adaptive sampling \citep{song2025historyguidedvideodiffusion}. Thought to be necessary for stabilizing autoregressive video generation, these inference schedules predetermine how much noise is removed from the sample at each step through noise conditions passed directly to the model. Because the schedule is followed open loop, errors can accumulate between the scheduled noise conditions and the true noise level of the sample. In fact, we find that existing models trained only on pairs of noisy data and their exact noise levels suffer a systematic train-inference mismatch due to this divergence that compounds over autoregressive rollouts and degrades quality.

In this work, we challenge the prevailing assumption that explicit noise conditioning is necessary for stable, high quality autoregressive video generation. We introduce Equilibrium Forcing (\EqF), a framework for denoising generative video model training and inference unrestricted by noise level conditioning. \EqF improves sampling robustness by learning an \textit{equilibrium} velocity field unified across all noise levels for sampling video data. This contrasts with the brittleness of standard noise-conditional video generation, which uses a chain of nonequilibrium velocity fields calibrated only for samples whose true noise level matches an externally provided condition. 

An equilibrium sampling landscape naturally supports data-adaptive inference. Rather than relying on external schedules to determine sample update, \EqF uses closed loop feedback from the current sample to estimate the progress of denoising and modulate the sampling dynamics accordingly. This feedback mechanism can be leveraged for budget-adaptive inference, where the state of generation is used to reallocate the entire remaining sampling trajectory. By connecting transport fields with fixed-point attractor fields at inference time, \EqF enables gradient-based samplers to adaptively accelerate inference and converge to higher quality samples (see Figure~\ref{fig:figure1}) \citep{nesterov1983method}. 

Prior work on noise-unconditional generation is limited to image generation and suggests that multiple training-time hyperparameters are necessary to remove the reliance on noise level conditions \citep{wang2025equilibrium}. Our approach instead proposes to unify and simplify the design space of denoising models by training with a simpler equilibrium objective, and inducing favorable sampling landscapes at inference time. In addition to removing costly training-time hyperparameter search, we find that prior noise-unconditional methods introduce an unintended loss weighting that is suboptimal in practice. Furthermore, the inference-time decisions for video and image generation are disparate; we study the flexibility uniquely afforded to autoregressive video generation by \EqF.

Despite having a simpler training objective and strictly less information than noise-conditional models, \EqF improves generation quality and consistency on difficult autoregressive video generation datasets. Extensive experiments attribute these gains to the proposed data-adaptive inference mechanisms, including closed loop feedback, gradient-based samplers, and budget adaptivity. We additionally introduce analysis and experiments that elucidate how noise-unconditional models can indeed learn the same denoising field as noise-conditional models. Our modular framework scales to realistic video generation by finetuning pretrained Flow Matching models at 1.3B and 5B parameters, demonstrating that the advantages of equilibrium samplers can be elicited from existing Flow Matching models through the \EqF objective with only a fraction of the pretraining compute \citep{wan2025wanopenadvancedlargescale}.

\begin{figure}[t]
    \centering
    \includegraphics[width=0.8\linewidth]{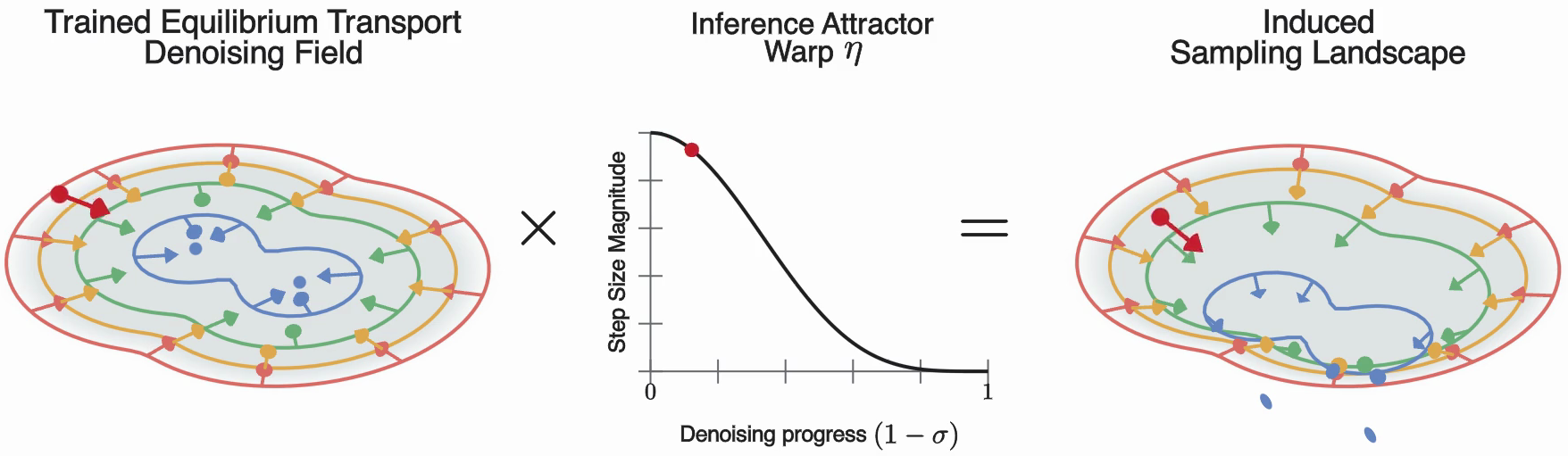}
    \caption{\textbf{Equilibrium Forcing Overview.} \EqF is trained with a simple unmodulated equilibrium denoising objective (left), which can be shaped at inference time through the $\eta$ warp (middle). The sampling landscape here is shaped to be an attractor, with the magnitude decreasing as the sample descends and converges (right).}
    \label{fig:figure1}
\end{figure}

\section{Background}

We begin with an overview of standard denoising generative models for video generation, using Flow Matching as a representative method.
We denote sequence variables that have a time dimension with boldface, e.g. the full video $\seq{x} = [x_1, x_2,\dots,x_T]$ is a concatenation of $T$ individual video frames $x_t \in \mathbb{R}^{C \times H \times W}$.

\subsection{Flow Matching with Noise Level Conditioning for Video Generation}
\label{2.1:flow_matching}

Flow Matching \citep{lipman2022flow, liu2022flow} models a data distribution $p(\seq{x})$ by learning a noise level-dependent velocity field that transports samples from a simple prior distribution $q(\seq{\epsilon})$ to the data distribution. In this paper, we operate within the Diffusion Forcing framework for video denoising, where each video frame has its own noise level $\seq{\sigma} := [\sigma_1, \sigma_2, \dots, \sigma_T] \in [0,1]^T$ \citep{chen2024diffusionforcingnexttokenprediction}. Samples are generated by solving the probability flow ODE 
${d \seq{x}^{\seq{\sigma}}}/{d \seq{\sigma}} = f_{\mathrm{FM}}^{\star}(\seq{x}^{\seq{\sigma}}, \seq{\sigma})$, with the velocity field approximated with the neural network $f_{\mathrm{FM}}$. 
The noising process is a linear interpolation of Gaussian noise $\seq{\epsilon} \sim q(\seq{\epsilon})$ and clean data $\seq{x} \sim p(\seq{x})$ applied framewise (denoted $\odot$):
\begin{align}
\seq{x}^{\seq{\sigma}}
= (\mathbf{1} - \seq{\sigma}) \odot \seq{x}
+ \seq{\sigma} \odot \seq{\epsilon}, 
\qquad {\epsilon}_t \stackrel{\mathrm{i.i.d.}}{\sim} \mathcal{N}(0, I), 
\qquad {\sigma}_t \stackrel{\mathrm{i.i.d.}}{\sim} \mathcal{U}[0,1].
\label{eq:data_gen_process}
\end{align}

The loss restricts the training distribution of the learned field to pairs of noisy data and the exact corresponding noise level. The network $f_{\text{FM}}$ is conditioned on the noise level of the data to directly predict a conditional velocity $\seq{v}$, providing the standard Flow Matching loss:
\begin{align}
\seq{v} := \seq{\epsilon} - \seq{x},
\qquad
\mathcal{L}_{\textrm{FM}}
= \operatorname*{\mathbb{E}}_{\seq{x}, \seq{\epsilon}, \seq{\sigma}}
\left[
\left\|
f_{\textrm{FM}}(\seq{x}^{\seq{\sigma}}, \seq{\sigma})
- \seq{v}
\right\|_2^2
\right].
\label{eq:flow_matching_target_loss}
\end{align}

For simplicity, we explain inference for non-autoregressive video generation where all frames have the same noise level. Inference initializes from pure noise $\seq{x}^{(0)} \sim \mathcal N(0, I)$ at $\seq{\sigma}^{(0)} = \seq{1}$ and integrates along a schedule $\{\seq{\sigma}^{(i)} \}_{i=0,\dots,N}$ until termination at $\seq{\sigma}^{(N)} = \seq{0}$. ODE solvers determine the step size through the difference between noise levels in the schedule \citep{karras2022elucidating, zhao2023unipc}. Data at the $i$-th iteration and noise level $\seq{\sigma}^{(i)}$ are updated to subsequent noise levels as:
\begin{align}
\seq{x}^{(i+1)} = \seq{x}^{(i)} - (\seq{\sigma}^{(i)} - \seq{\sigma}^{(i+1)}) \odot f_{\textrm{FM}}(\seq{x}^{(i)}, \seq{\sigma}^{(i)}).
\label{eq:flow_matching_sampling}
\end{align}

\subsection{Autoregressive Video Denoising}
\label{2.2:diffusion_forcing}

Given ground-truth context or previously generated frames $\seq{c}$, autoregressive video generation utilizes a sliding window approach by repeatedly sampling from the distribution $p(\seq{x} \vert \seq{c})$. Early attempts relied on bespoke fixed-length conditioning schemes to inject $\seq{c}$ as a separate signal \citep{yang2025cogvideoxtexttovideodiffusionmodels}. Instead, Diffusion Forcing-style training covers the case where $\seq{x}^{\seq{\sigma}}$ contains noisy and clean frames simultaneously. This enables context-conditional sampling from $p(\seq{x} \vert \seq{c})$ where the clean $\seq{c}$ and noisy $\seq{x}$ are processed within a unified self-attention window \citep{song2025historyguidedvideodiffusion}.

Autoregressive inference generalizes the non-autoregressive schedule by assigning each sampling iteration $i$ and temporal index $t$ its own noise level, forming a scheduling matrix over sampling iterations and time \citep{chen2024diffusionforcingnexttokenprediction, ruhe2024rolling}. To support temporal causality, these schedules typically place earlier frames at lower noise levels than later frames, as in Figure~\ref{fig:figure2} (right). Once a frame reaches zero noise, it is committed to the context $\seq{c}$, the active window slides forward, and fresh noise is appended for future frames. A full algorithm is given in Appendix~\ref{apd:inference_sched_math}. While such predefined schedules stabilize autoregressive generation, they force the model to follow a fixed denoising path that cannot adapt to the state of the inference procedure \citep{chen2024diffusionforcingnexttokenprediction}.

\begin{figure}[t]
    \centering
    \includegraphics[width=1\linewidth]{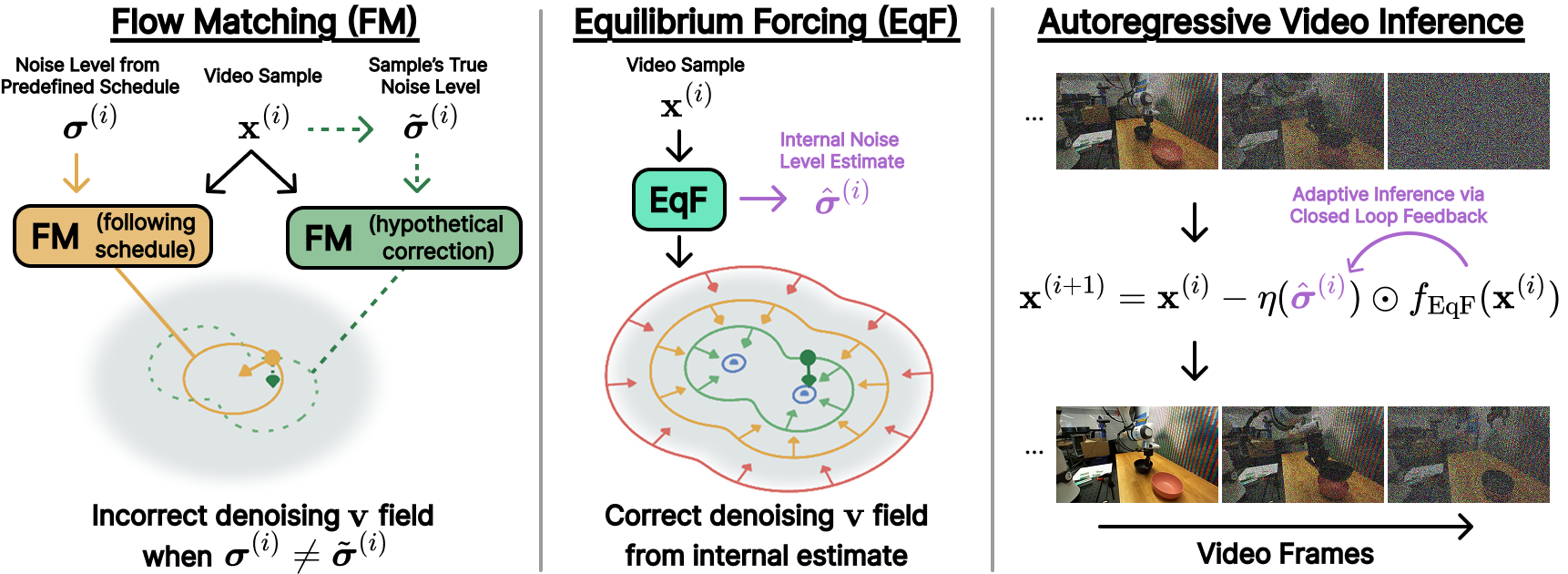}
    \caption{\textbf{Left:} FM follows externally-defined noise schedule conditioning, which diverges from the true denoising field. 
    \textbf{Center:} 
    By training without explicit noise level conditioning, \EqF learns a correct denoising field everywhere by estimating the noise level internally.
    \textbf{Right:} \EqF leverages the internal noise level estimate for adaptive autoregressive video generation inference.}
    \label{fig:figure2}
\end{figure}

\section{Divergent Noise Level Conditions in Flow Matching}
\label{3:flow_analysis}

In this section, we establish that the current design of noise-conditional video denoising models suffers from a lack of data adaptivity, attributable to the open-loop nature of strict noise conditioning schedules. We find that the true noise level diverges from the scheduled noise level during inference, compounding error over autoregressive rollouts. These results suggest that \textit{removing noise conditioning altogether} could unlock data adaptivity that stabilizes inference (see Figure~\ref{fig:figure2}).

\paragraph{Scheduled Noise Levels Diverge During Flow Matching Inference.}
Standard noise-conditional inference algorithms are inherently open loop, unable to respond if the true noise level of the sample $\tilde{\seq{\sigma}}^{(i)}$ diverges from the scheduled noise level $\seq{\sigma}^{(i)}$ over an autoregressive rollout. 
We measure this divergence by training a small framewise noise level predictor $g_\phi: \mathbb{R}^{C \times H \times W} \mapsto [0,1]$ on the task of regressing the noise level $\sigma$ from an individually noised frame $x^\sigma$.
The prediction $\tilde{\sigma} \approx g_\phi(x^\sigma)$ estimates the mode of a sample's true noise level, which we find to estimate the correct value in practice due to posterior concentration of the noise level in high dimensions (see Appendix~\ref{apd:noise_from_data}).

Under our standard baseline Flow Matching setting on the Minecraft dataset (detailed further in Section~\ref{5:experiments}), we generate by autoregressively sliding a 50-frame window consisting of 25 actively predicted frames and 25 clean context frames. Starting with 25 ground truth context frames, we generate 275 frames and compare the scheduled noise level with the predictor-estimated noise level, plotting the absolute residual between the values in Figure~\ref{fig:abs_residual_and_fvd_over_time}~(left). If the noise condition is accurate, the divergence should be $0$ across the trajectory; we find it is nonzero and increasing over time.

\paragraph{Noise Level Divergence Causes Compounding Error Over Autoregressive Video Rollouts.}
For static data, an integration error is confined to one sample. However, for sequence data such as videos, accumulated errors can cause the model to be further out of distribution for the next round of generation. We measure the Frechet Video Distance (FVD, \cite{unterthiner2018towards}) over 25-frame chunks of the same 275-frame rollouts with respect to 25-frame clips from ground-truth videos. Figure~\ref{fig:abs_residual_and_fvd_over_time} (right) demonstrates that autoregressive errors accumulate steeply in noise conditional inference, especially as generated frames become context. We elaborate on the causal relationship between noise level divergence and compounding error, in addition to further analysis, in Appendix~\ref{apd:additional_analysis}.
We further experiment with attempting to close the loop for Flow Matching with feedback from the current sample during noise-conditional inference. To do so, we condition the Flow Matching model on the predictor-estimated noise level $\tilde{\sigma}$ during inference, with results represented by the dotted line. The error is only partially mitigated, suggesting that fully data-adaptive inference may require removing the reliance on \textit{external} noise conditioning altogether. 

\begin{figure}[t]
    \centering
    \begin{subfigure}[t]{0.495\linewidth}
        \centering
        \includegraphics[width=\linewidth]{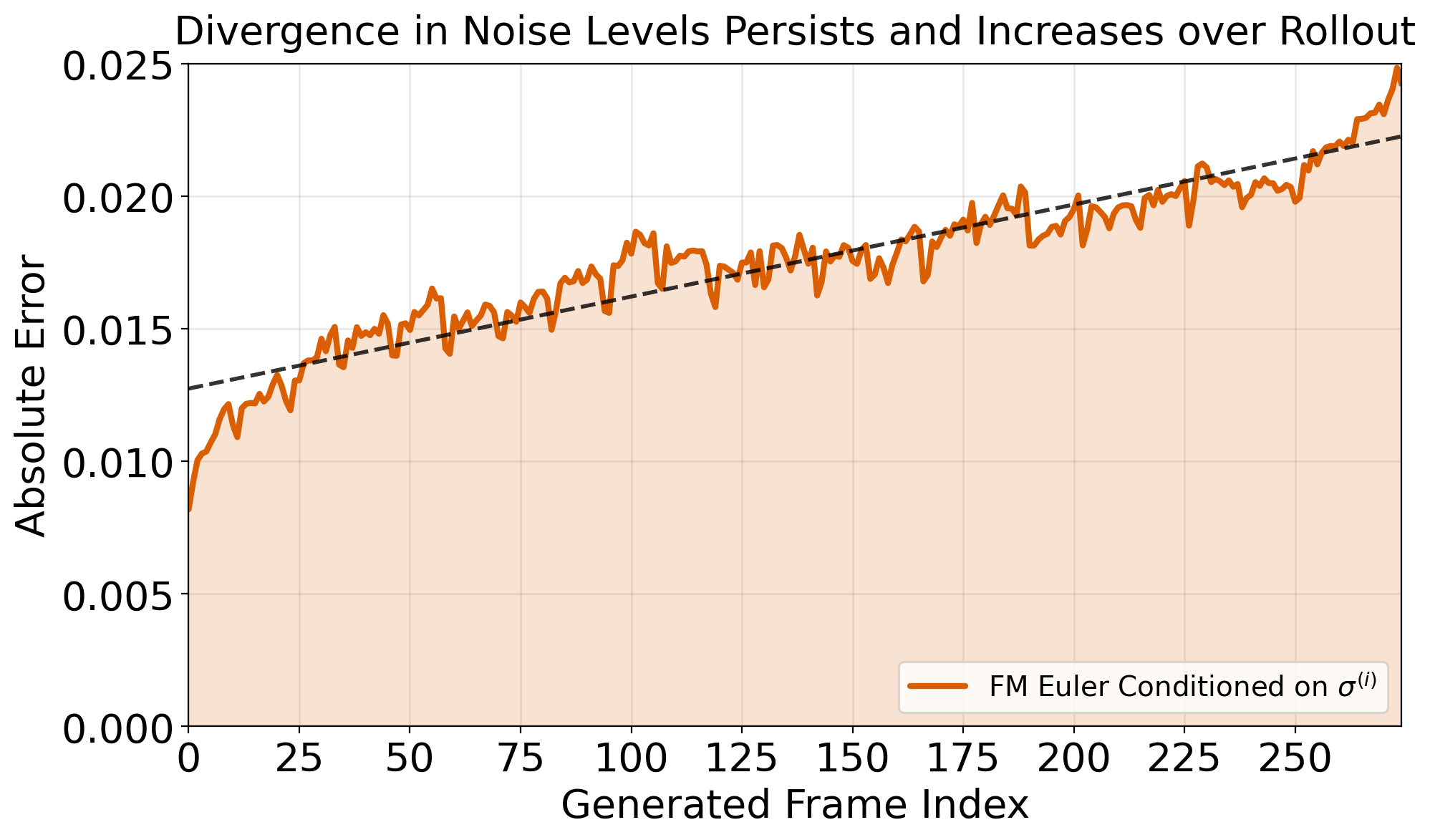}
    \end{subfigure}
    \hfill
    \begin{subfigure}[t]{0.495\linewidth}
        \centering
        \includegraphics[width=\linewidth]{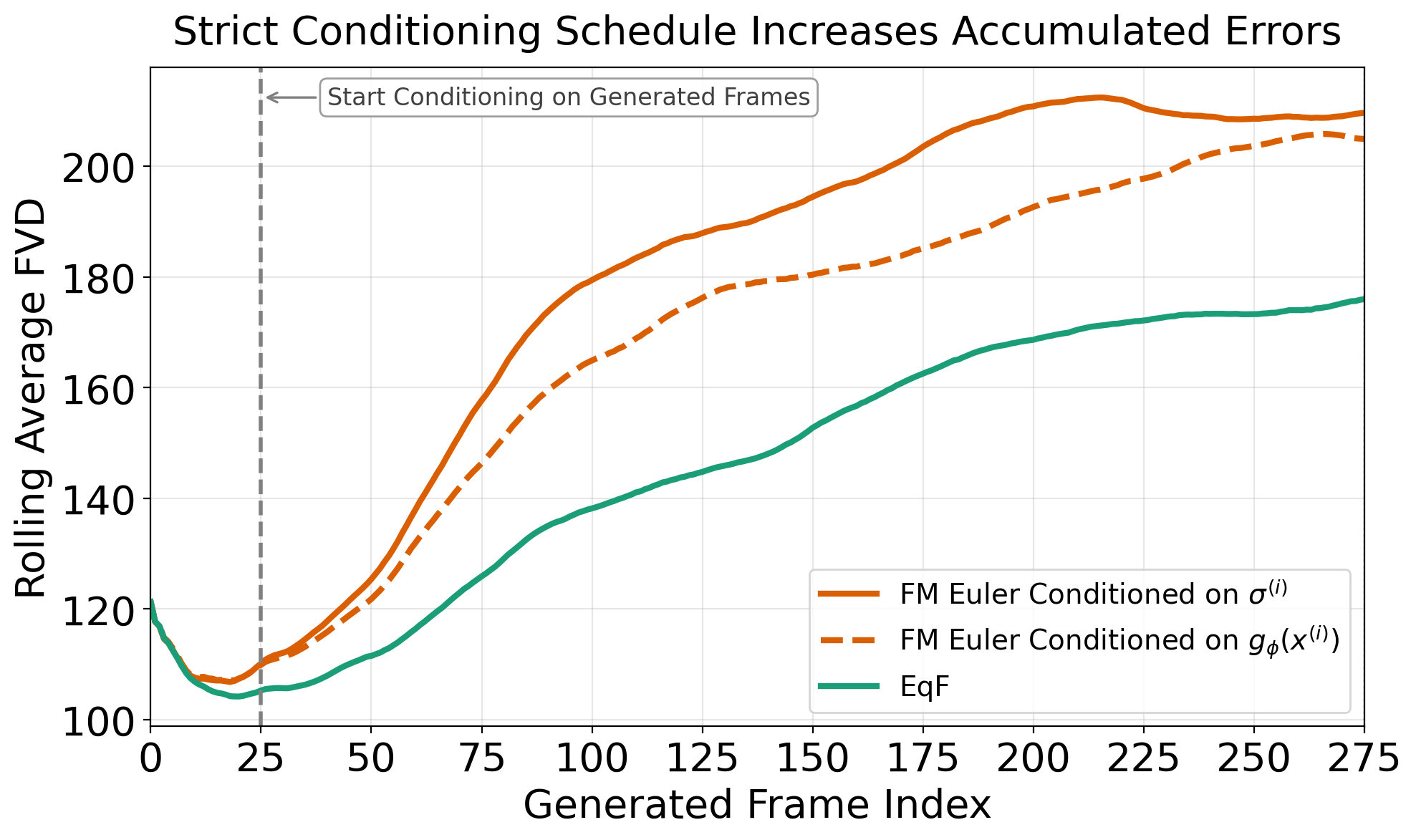}
    \end{subfigure}
    \caption{\textbf{Left}: Absolute noise conditioning error between scheduled $\sigma^{(i)}$ and the predicted $\tilde{\sigma}^{(i)} \approx g_\phi(x^{(i)})$ increases over 256 autoregressive rollouts. \textbf{Right}: FVD worsens over 256 autoregressive rollouts with FM due to externally prescribed noise conditions. Conditioning on estimated noise levels (dotted) helps slightly. Removing the noise condition altogether (\EqF) performs best.}
    \label{fig:abs_residual_and_fvd_over_time}
\end{figure}

\section{Equilibrium Forcing}

The preceding analysis suggests that inference tied to externally prescribed noise conditioning leads to noise level divergence, resulting in degraded performance over long autoregressive rollouts. \EqF addresses this by removing noise conditioning entirely, allowing the denoising process to be adaptively driven by the current sample rather than by a fixed schedule. In this section, we unify recent theory of noise-unconditional denoising with training and inference recipes designed to leverage the unique benefits of an equilibrium sampler for autoregressive video generation.

\subsection{Equilibrium Denoising Objective}

We introduce \EqF, a framework for modeling probability distributions $p(\seq{x})$. Instead of learning a separate velocity field for each value of $\seq{\sigma}$, \EqF learns a single field over the data itself, mitigating the effect of the noise conditioning-based error demonstrated in Section~\ref{3:flow_analysis}. Samples are generated by integrating the noise-autonomous ODE ${d \seq{x}^{\seq{\sigma}} } /{d \seq{\sigma}} = f_\EqF^\star(\seq{x}^{\seq{\sigma}})$, with the field approximated by a neural network $f_\EqF$.
Training examples are generated from the same noising process as Equation~\ref{eq:data_gen_process} with independent noise levels per frame. The loss simply drops the noise level conditioning for video data and regresses against the conditional velocity target $\seq{v} = \seq{\epsilon} - \seq{x}$: 
\begin{align}
\mathcal{L}_{\EqF} &= \operatorname*{\mathbb{E}}_{\seq{x},\seq{\epsilon},\seq{\sigma}}
\left [
\left |
\left |
f_{\EqF}(\seq{x}^{\seq{\sigma}}) - \seq{v}
\right |
\right |^2_2
\right ].
\label{eq:eqf_loss}
\end{align}


\begin{figure}[t]
\centering

\begin{minipage}[t]{0.51\linewidth}
\vspace{0pt}
\begin{algorithm}[H]
\caption{EqF Budget-Adaptive Inference. GD and no autoregressive state for simplicity}
\label{alg:budget_adaptive}

\begin{lstlisting}[style=pytorch]
# f_eqf(x): trained eqf v field
# h_omega(z): noise level estimator
# N: number of sampling steps
sigma_hat = 1
x = randn(sample_shape)
for i in range(N):
    v_hat, activations = f_eqf(x)
    sigma_hat = h_omega(activations)
    # dynamic rescaling
    step_size = sigma_hat / (N - i)
    x = x - step_size * v_hat
return x
\end{lstlisting}
\end{algorithm}
\end{minipage}
\hfill
\begin{minipage}[t]{0.47\linewidth}
\vspace{0pt}
\centering
\includegraphics[width=\linewidth]{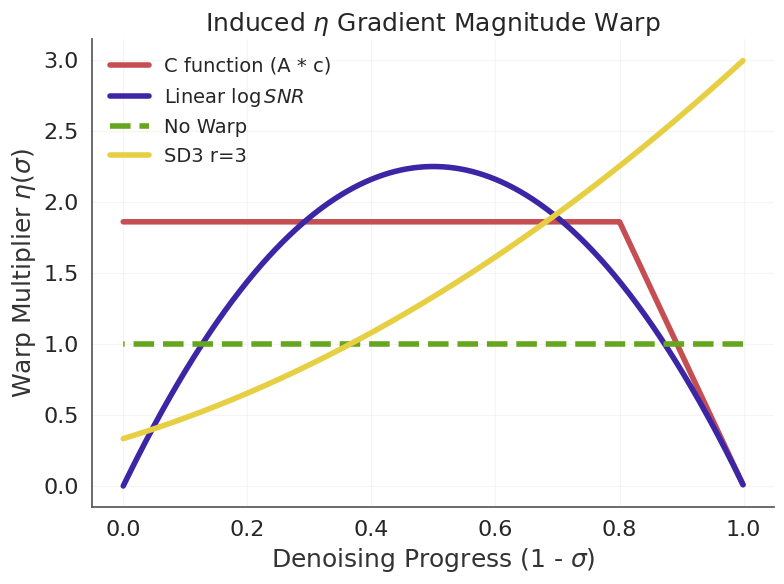}
\caption{$\eta$ warp and induced gradient multiplier. $\eta$ functions that vanish as $\sigma \to 0$ induce an attraction field.}
\label{fig:inference_schedule_vis}
\end{minipage}
\end{figure}

\subsection{Bayes Optimal Denoisers Coincide Under Noise Level Concentration}
\label{4.2:noise_conc}
We next introduce a loss decomposition that establishes how \EqF can learn the same denoising field as Flow Matching without conditioning on the noise level, ameliorating the off-manifold integration error established in Section~\ref{3:flow_analysis}. We denote the minimizer of the Flow Matching loss as $f_{\textrm{FM}}^\star$ and the minimizer of the \EqF loss as $f_{\EqF}^\star$. 

\begin{proposition}
\label{prop:eqf_risk_decomposition}
The optimal \EqF risk decomposes as
\begin{align}
\mathcal L_{\EqF}^\star
=
\mathcal L_{\mathrm{FM}}^\star
+
\operatorname*
{\mathbb{E}}_{\seq{x},\seq{\epsilon},\seq{\sigma}}
\left[
\left\|
f_{\mathrm{FM}}^\star(\seq{x}^{\seq{\sigma}},\seq{\sigma})
-
f_{\EqF}^\star(\seq{x}^{\seq{\sigma}})
\right\|_2^2
\right].
\end{align}
\end{proposition}

The proof is given in Appendix~\ref{loss_decom}. The second term in the decomposition measures the excess risk incurred by replacing the noise-conditional Flow Matching Bayes-optimal target $f_{\mathrm{FM}}^\star(\seq{x}^{\seq{\sigma}},\seq{\sigma})$ with the noise-unconditional \EqF Bayes-optimal target $f_{\EqF}^{\star}(\seq{x}^{\seq{\sigma}})
    =
    \operatorname*{\mathbb{E}}_{\seq{\sigma}}
    \left[
        \operatorname*{\mathbb{E}}_{\seq{x},\seq{\epsilon}}
        \left[
            \seq{v}
            \mid
            \seq{x}^{\seq{\sigma}}, \seq{\sigma}
        \right]
        \mid
        \seq{x}^{\seq{\sigma}}
    \right]$. Since \EqF does not receive $\seq{\sigma}$ explicitly, its Bayes-optimal denoiser averages the fixed-noise targets over the posterior $p(\seq{\sigma}\mid \seq{x}^{\seq{\sigma}})$. 
    
Recent theoretical work shows that, in high dimensions, this posterior concentrates around a single noise level \citep{kadkhodaie2026blind,sahraee2026geometry}. Under this concentration, the posterior average collapses to the target corresponding to the identified noise level, making the squared discrepancy in the decomposition small. This suggests that \EqF can approximate the FM target pointwise for data in high dimensions, which we provide empirical evidence for in Section~\ref{5.5:evidence_for_decomposition}.

\subsection{Closed Loop Inference with Noise Level Estimation and $\eta$ Warp-Induced Dynamics}
\label{4.3:eta_schedule_explanation}

As minimizing the denoising loss in high dimensions requires the \EqF model to internally represent an input's current noise level, an estimate of that noise level provides a natural proxy for sampling progress toward the data manifold. Specifically, we train a lightweight framewise readout model $h_\omega$ to predict the mode of the noise level posterior from the \EqF model's activations $\seq{z}$. This mirrors the prediction task in Section~\ref{3:flow_analysis}, but uses internal model features rather than the raw noisy input $\seq{x}^{\seq{\sigma}}$, which we find to be more performant in closing the loop (see Appendices~\ref{apd:noise_from_data},\ref{apd:fixed_fm_with_g}). Driving integration with step sizes modulated by $h_\omega$'s estimate of the noise level allows \EqF sampling procedures to operate in a closed loop with respect to the status of the sample. We adjust the inference procedure by using a warp function $\eta$, defined as a framewise map from the noise level estimate to a step size:
\begin{align}
(\widehat{\seq{v}}^{(i)}, \seq{z}^{(i)}) &= f_{\EqF}(\seq{x}^{(i)}),
\qquad 
\widehat{\seq{\sigma}}^{(i)} = h_\omega(\seq{z}^{(i)}), \\
\seq{x}^{(i+1)} &= \seq{x}^{(i)} - \eta(\widehat{\seq{\sigma}}^{(i)}) \odot \widehat{\seq{v}}^{(i)}.
\label{eq:sampler_update}
\end{align}

$\eta$ includes a rescaling factor inversely proportional to the total number of sampling steps $N$, while the shape of $\eta$ determines the sampling dynamics.
Choosing $\eta(\seq{\sigma}) \to \seq{0} $ as $\seq{\sigma} \to \seq{0}$ warps the sampling problem into a fixed-point \textbf{attractor} search for $\eta(\seq{\sigma}) \odot f_{\EqF}(\seq{x}) = \seq{0}$.
This contrasts with \textbf{transport-style} sampling (such as standard FM Euler integration), where the general nonzero constraint $\eta(\seq{\sigma}) \odot f_{\EqF}(\seq{x}) = - \eta(\seq{0}) \seq{x}$ holds as $\seq{\sigma} \to \seq{0}$.

\subsection{Choice of $\eta$ Warp Function}
\label{4.4:choices}

Applying attractor-style $\eta$ warp functions on top of \EqF's trained transport field is powerful for unlocking the toolkit of gradient-based optimization solvers that converge to a fixed point. In particular, we implement Nesterov Accelerated Gradient (NAG) within this closed loop sampling framework, with the full algorithm in Appendix~\ref{apd:inference_details}. The usage of these solvers under open-loop, noise-conditional sampling was prohibitive: steps with adaptive, unscheduled step sizes would make the sample jump to unknown noise levels, leading to significantly divergent noise-level conditioning and off-manifold errors. 

DPM-Solver's linear log-SNR schedule and the $c$-function of EqM are examples of $\eta$ warps that induce attractor-style dynamics, while the SD3 reparameterization and identity map (using evenly spaced steps) are transport warps \citep{lu2022dpm, esser2024scaling, wang2025equilibrium}. We visualize these warp function options in Figure~\ref{fig:inference_schedule_vis} and evaluate these warps experimentally in Section~\ref{5:experiments}. The $\eta$ warp function is closely related to established noise level-reparameterizations of sampling ODEs for denoising generative models, differing in that $\eta$ explicitly specifies the derivative of the reparameterization rather than a grid of noise levels; further interpretation is in Appendix~\ref{apd:sampler_warp_schedules}.

\subsection{Budget-Adaptive Sampling Algorithm}
\label{4.5:budget_adaptive_algorithm}

We additionally propose Algorithm~\ref{alg:budget_adaptive}, a sampler designed to leverage \EqF's equilibrium field and noise level estimation capabilities in a budget-aware manner. Standard samplers specify a schedule before inference begins, meaning the update at step $i$ is determined by the precomputed grid in $\eta$-warp space. While \EqF closed loop-only inference (Equation ~\ref{eq:sampler_update}) has the ability to reindex into the predefined $\eta$-function with the noise level estimate, it retains the rescaling implied by the total original budget $N$. Instead, Algorithm~\ref{alg:budget_adaptive} incorporates both the estimated noise level and the remaining sampling budget $N-i$ to reallocate the denoising trajectory at each iteration. This budget-adaptivity can be seen as a version of the $\eta$ function that defines a simple update rule including the number of steps remaining. It avoids overshooting as the sample approaches the data manifold when accelerated with methods like NAG. We present this algorithm as just one instantiation of the myriad possible adaptive algorithms enabled by \EqF; we describe an additional adaptive early stopping algorithm that leverages the noise level to reduce spurious additional denoising steps in Appendix~\ref{apd:adaptive_early_stopping}.

\begin{figure}[t]
    \centering
    \includegraphics[width=0.95\linewidth]{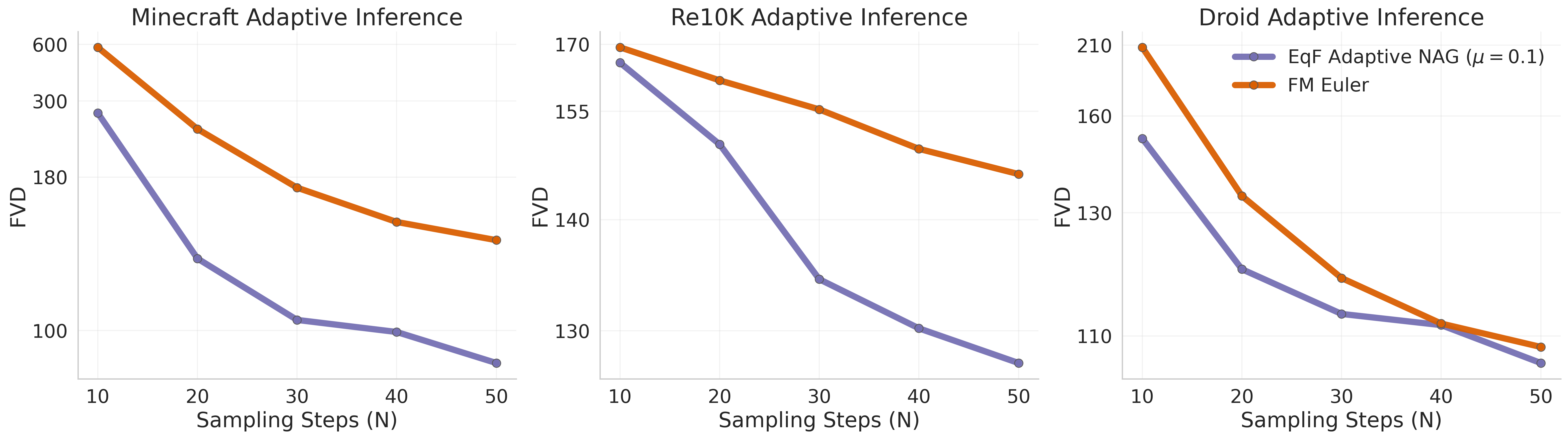}
    \caption{\textbf{Budget-adaptive inference results.} \EqF with budget-adaptive sampling achieves the lowest FVD across Minecraft, Re10K, and Droid.}
    \label{fig:minecraft_budget_adaptive}
\end{figure}

\section{Experiments}
\label{5:experiments}

Here, we present results on autoregressive video generation experiments, demonstrating how \EqF's data adaptivity leads to performance gains over noise-conditional and noise-unconditional baselines.

\subsection{Experimental Setup and Baselines}

For all video generation experiments, we follow standard practice and process the video through a VAE to retrieve compressed video latents, which are then noised independently per-frame (see \cite{rombach2022highresolutionimagesynthesislatent} and Section~\ref{2.2:diffusion_forcing}). We consider a set of baselines representative of current methods for video denoising generative models: \textbf{1) Flow Matching (FM)} is implemented as explained in Section~\ref{2.1:flow_matching}. We use standard Euler integration with warped noise levels. \textbf{2) Diffusion} uses the DDPM training objective from \cite{ho2020denoising} with the DDIM sampler \citep{song2021denoising}. \textbf{3) Equilibrium Matching (EqM)} drops noise conditioning and trains with a modulated velocity magnitude to bake in a warped sampling landscape during training \citep{wang2025equilibrium}. \textbf{4) Noise Unconditional-Flow Matching (NU-FM)} follows \cite{sun2025noise}'s training and inference recipe without noise level conditioning, adapted to video.

All baseline models have matched compute and data budgets. We evaluate with EMA weights and autoregressive context conditioning (see Section~\ref{2.2:diffusion_forcing} and Appendix~\ref{apd:inference_sched_math}). We report FVD \citep{unterthiner2018towards} to measure generated video quality and diversity, and include VBench scores \citep{huang2024vbench} that provide several metrics for video consistency. Extended information on the experimental setting and evaluation is provided in Appendix~\ref{apd:expt_details}. All code and checkpoints are released \href{https://equilibriumforcing.github.io/}{here}.

\subsection{Training Equilibrium Video Denoising Models From Scratch}

We train models from scratch on clips of 50 frames from the Minecraft dataset, a video dataset with per-frame action conditions \citep{yan2023temporally}. All models use the same CogVideoX-based backbone denoising model architecture with an active generation window of 50 frames but with different training objectives \citep{yang2025cogvideoxtexttovideodiffusionmodels}. We drive closed loop inference for \EqF with a readout predictor $h_\omega$ as explained in Section~\ref{4.3:eta_schedule_explanation}. All training details are elaborated upon in Appendix~\ref{apd:expt_details}. Each experiment evaluates on action-conditioned video generation of 256 videos, with 275 future frames generated given 25 ground-truth context frames; strong performance on this long autoregressive rollout requires generations to remain stable and consistent over many shifts of the active window.

The main results on Minecraft are presented in Table~\ref{tab:minecraft_main} with 250 sampling steps per frame and a sliding window of 25 active sampling frames and 25 context frames. Qualitative results are visualized in Figure~\ref{fig:minecraft_vis} and on the project page \href{https://equilibriumforcing.github.io/}{here}. \EqF achieves the best FVD and VBench scores, with the best setting combining closed-loop inference, a vanishing $\eta$ function, and NAG. 

Across various computational budgets with the standard sampling algorithm, \EqF sets a frontier in performance on the same experimental setting in Minecraft as shown by Figure~\ref{fig:fixed_compute_scaling}. We find that larger inference budgets permit increased acceleration through a higher setting of the momentum parameter $\mu$. We hypothesize that with excessively large updates under a coarser discretization, higher momentum carries stale gradient information to low noise levels and risks overshooting the data manifold. 
Additional inference-time decisions, including the use of the budget-adaptive sampling algorithm and the effect of closed-loop acceleration with NAG are studied further below.

\begin{table}[t!]
\centering
\caption{Autoregressive Minecraft 300-frame rollout results across 5 seeds. \EqF with NAG sampling achieves the best FVD and VBench scores.}
\label{tab:minecraft_main}
\resizebox{0.9 \textwidth}{!}{
\begin{tabular}{llccc}
\toprule
\textbf{Method} & \textbf{Sampler} & \textbf{Inference Warp} & \textbf{FVD$\downarrow$} & \textbf{VBench$\uparrow$} \\
\midrule
\EqF (Ours) & NAG, $\mu=0.3$ & $c$-function & \best{$64.34 \pm 1.78$} & \best{$0.7733 \pm 0.0008$} \\
EqM \citep{wang2025equilibrium} & NAG $\mu = 0.1$ & -- & $77.02 \pm 5.21$ & $0.7679 \pm 0.0009$ \\
\EqF & GD & $c$-function & $77.63 \pm 3.31$ & $0.7721 \pm 0.0005$ \\
FM & Euler & $c$-function & $104.13 \pm 4.33$ & $0.7703 \pm 0.0012$ \\
NU-FM \citep{sun2025noise} & GD & -- & $101.16 \pm 3.57$ & $0.7698 \pm 0.0011$ \\
Diffusion & DDIM & -- & $106.00 \pm 1.22$ & $0.7681 \pm 0.0008$ \\
\bottomrule
\end{tabular}
}
\end{table}

\begin{figure}[t]
    \centering
    \includegraphics[width=1.0\linewidth]{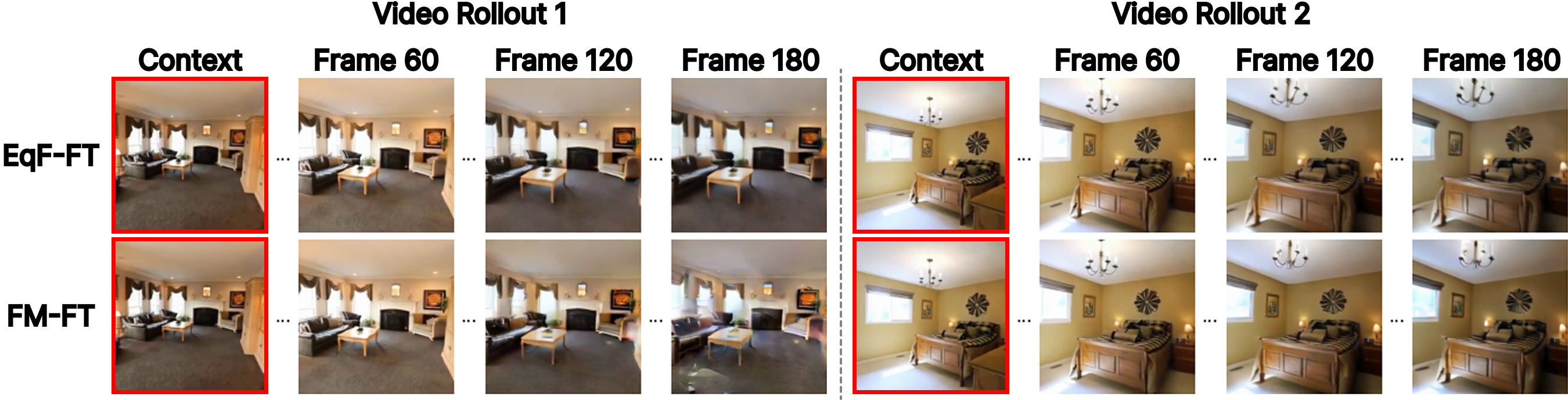}
    \caption{Visualizations of autoregressively generated RealEstate10K videos using \EqF-FT and FM-FT finetuned models from Wan 2.1.}
    \label{fig:re10k_vis}
\end{figure}

\begin{figure}[t]
    \centering
    \includegraphics[width=1\linewidth]{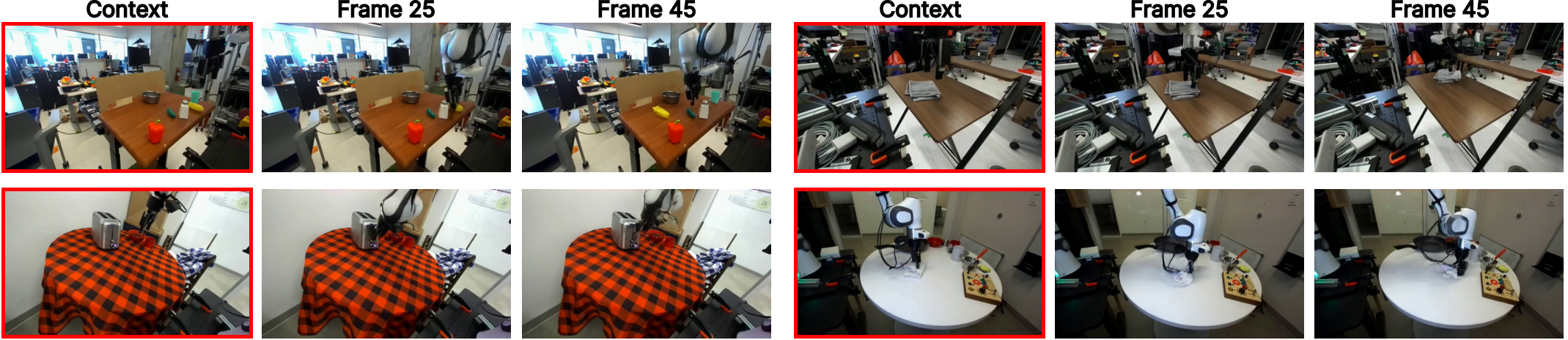}
    \caption{Generated videos from the Droid dataset using \EqF-FT finetuned from Wan 2.2.}
    \label{fig:droid_vis}
\end{figure}

\subsection{Finetuning Large-Scale Flow Matching Models to \EqF}

Beyond training \EqF models from scratch, we investigate whether equilibrium denoising fields are already present inside large pretrained Flow Matching models. As established in Section~\ref{4.2:noise_conc}, noise level posteriors concentrate in high dimensions, implying that the Bayes-optimal \EqF and Flow Matching denoisers coincide pointwise. This suggests that finetuning with the \EqF objective can induce an internal estimator of the noise level, transforming the noise-conditional field into an equilibrium field. We test this hypothesis through \textit{equilibrium finetuning} on two datasets.

\paragraph{Equilibrium Finetuning on the RealEstate10K Dataset.}
Starting from Wan 2.1 T2V 1.3B \citep{wan2025wanopenadvancedlargescale}, a large scale video model trained with the Flow Matching objective, we remove the external noise-conditioning pathway by zeroing the noise level condition and continue training with the \EqF objective. Specifically, we finetune on 49-frame clips from RealEstate10K (Re10K), a realistic video dataset of house tours with camera annotations \citep{zhou2018stereo}. Afterwards, we train a simple $\sigma$-readout predictor following a similar approach to Minecraft. Models converge within 100k steps; we denote the \EqF-finetuned model as \EqF-FT, and the FM finetuned model as FM-FT. Further finetuning information can be found in Appendices~\ref{apd:expt_details},~\ref{apd:eqfft}. 

We evaluate on camera-conditioned autoregressive generation of 256 videos, where 152 future frames are generated from 37 ground-truth context frames. Sampling uses 50 steps per frame with a sliding window of 12 active frames and 37 context frames, with the stride also set to 12. Quantitative results are reported in Table~\ref{tab:re10k_autoregressive} and qualitative results are visualized both in Figure~\ref{fig:re10k_vis} and on the project page \href{https://equilibriumforcing.github.io/}{here}.
\EqF-FT with closed loop sampling and NAG improves upon FM, corroborating our hypothesis about the limiting nature of noise conditioning, and demonstrating that equilibrium finetuning can unlock data-adaptive sampling even at scale.
Results on an additional setting where less initial context is given (29 frames) using the budget-adaptive algorithm are reported in the next section.

\paragraph{Equilibrium Finetuning on the Droid Dataset.}
We perform the same equilibrium finetuning procedure on an additional robot manipulation dataset at higher resolution with a larger scale model to demonstrate these results are not unique to a single model/dataset pair. Specifically, we finetune the Wan 2.2 TI2V 5B model on 49-frame clips from the Droid dataset, a language-conditioned video dataset of robots performing diverse tasks \citep{khazatsky2024droid}. The Droid model represents approximately a 4$\times$ increase in both model parameters and data resolution over the Re10K model. We evaluate on prompt-conditioned generation of 256 videos, where 36 video frames are generated from 13 video frames of context. Results are discussed in the next subsection.

\subsection{Inference-Time Decisions for Equilibrium Forcing}
\label{5.3:inference_minecraft}

\begin{figure}[t!]
    \centering
    \captionsetup{
        font=small,
        justification=raggedright,
        singlelinecheck=false,
        skip=4pt
    }

    \begin{minipage}[t]{0.32\textwidth}
        \vspace{0pt}
        \centering
        \includegraphics[width=\linewidth]{
            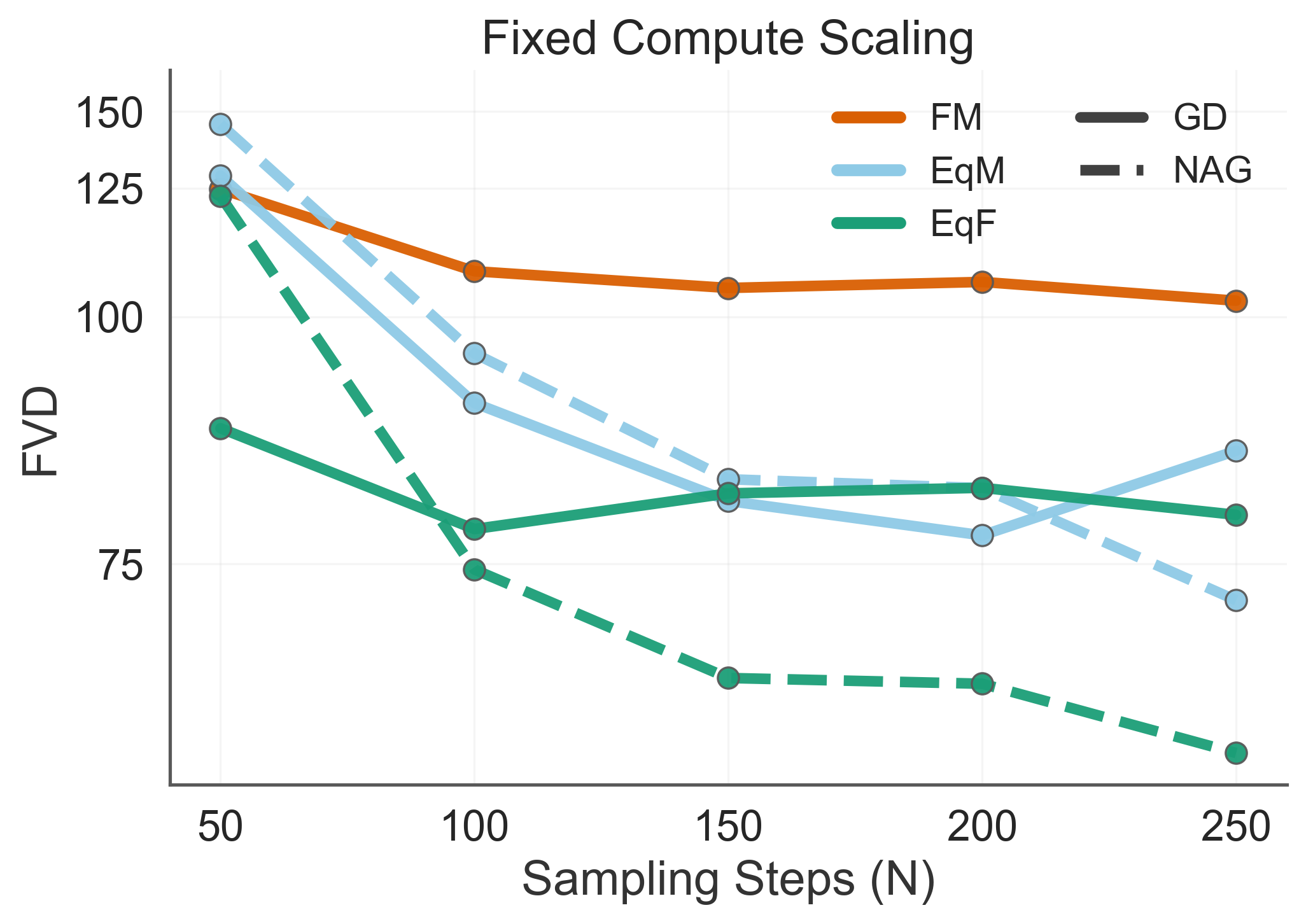
        }
        \captionof{figure}{
            \EqF is optimal across compute scales on Minecraft compared with noise-conditional and -unconditional baselines.
        }
        \label{fig:fixed_compute_scaling}
    \end{minipage}\hfill%
    \begin{minipage}[t]{0.32\textwidth}
        \vspace{0pt}
        \centering
        \includegraphics[width=\linewidth]{
            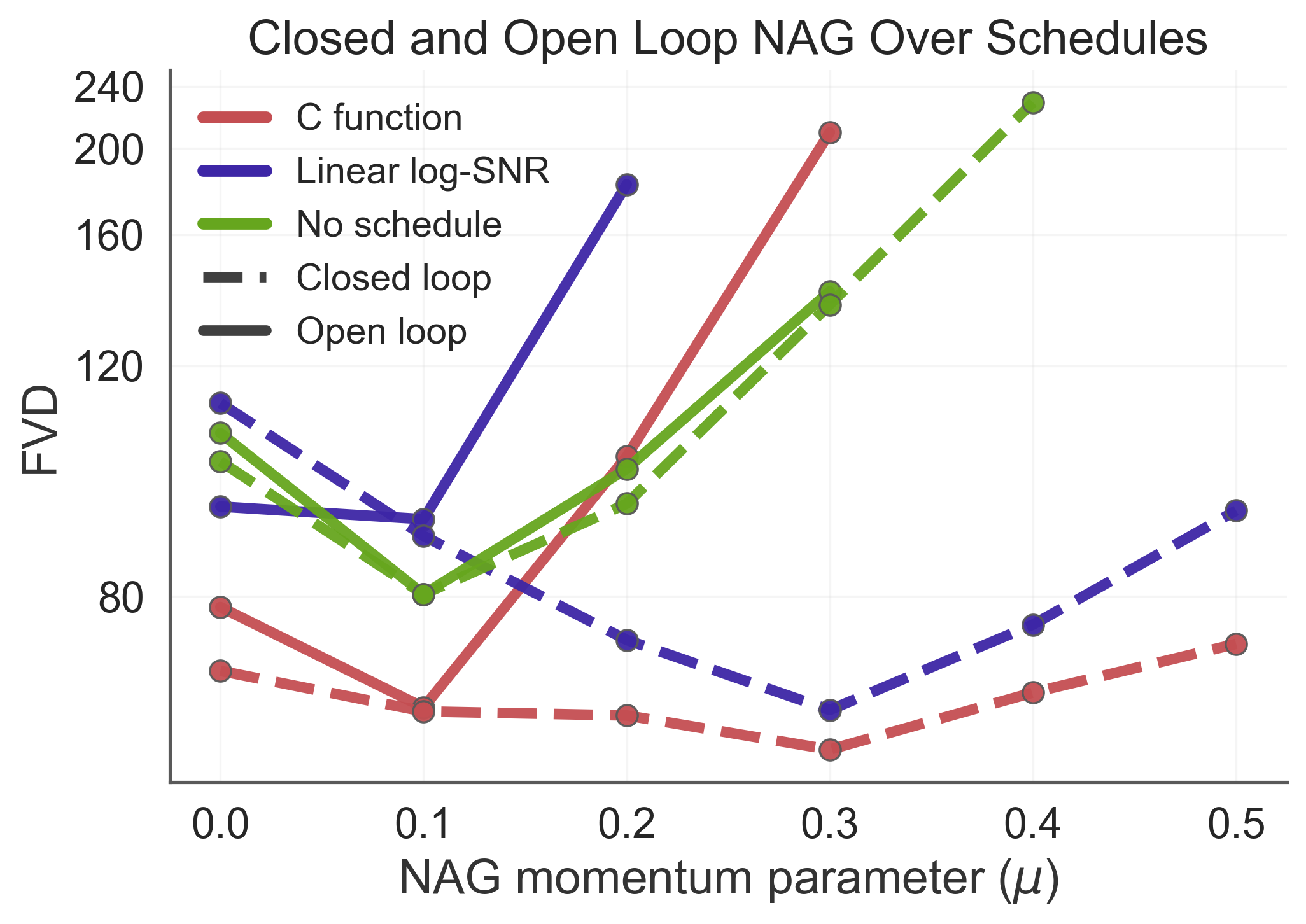
        }
        \captionof{figure}{
            NAG is optimal with closed-loop and attractor ($c$-function) $\eta$ warp, compared to open-loop and transport warps.
        }
        \label{fig:closed_open_nag}
    \end{minipage}\hfill%
    \begin{minipage}[t]{0.32\textwidth}
        \vspace{0pt}
        \centering
        \includegraphics[width=\linewidth]{
        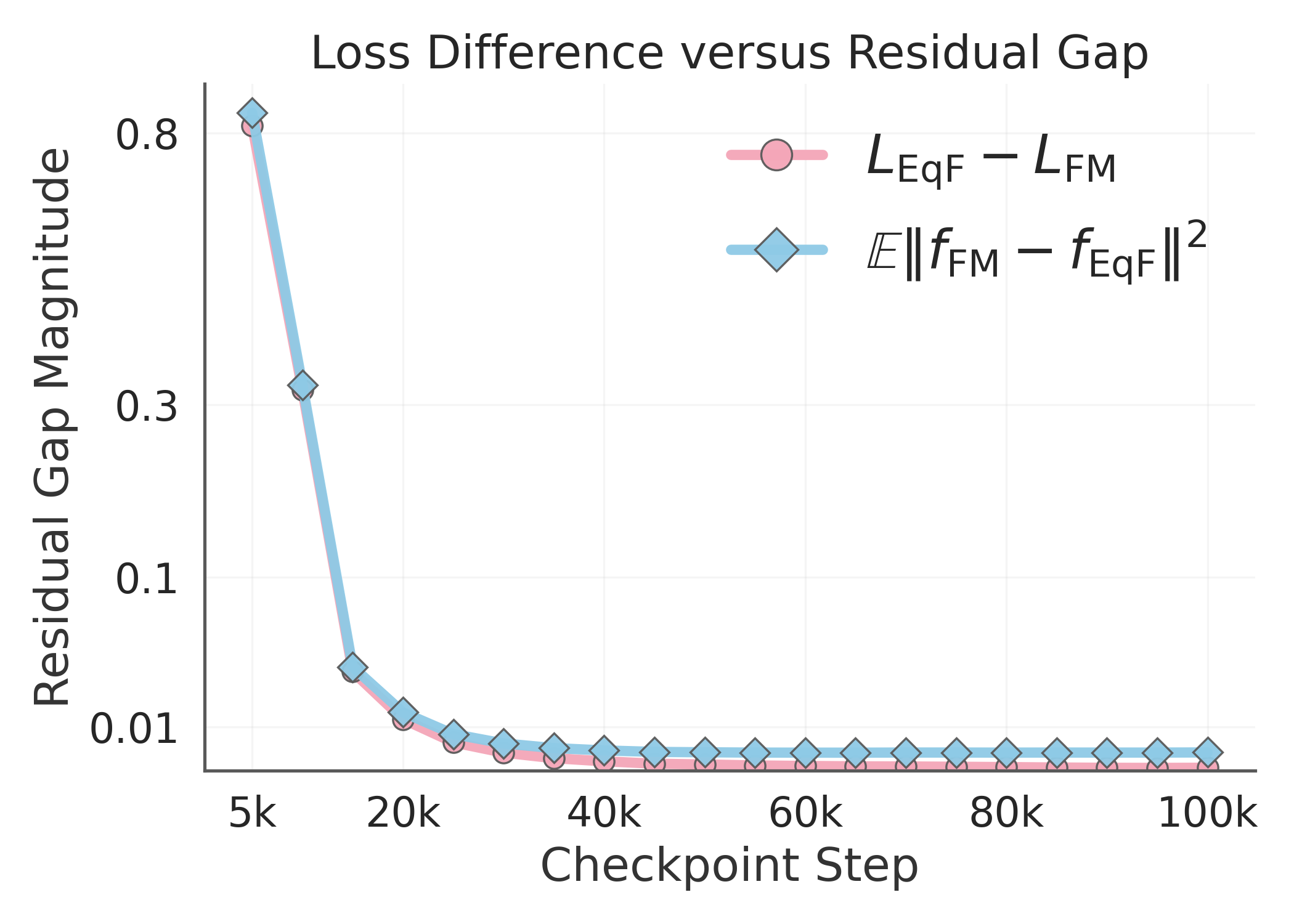
        }
        \captionof{figure}{
            The FM/\EqF loss difference and the expected residual gap between the denoising field matches across Re10K finetuning.
        }
        \label{fig:re10k_ft_residual_gap}
    \end{minipage}

\end{figure}

\paragraph{Budget-Adaptive Inference.}
We evaluate the budget-adaptive sampling algorithm presented in Section~\ref{4.5:budget_adaptive_algorithm} on the three datasets, with the inference setting described in the previous sections. Standard \EqF is run with NAG and $\mu=0.1$; budget-adaptive \EqF uses an identity map $\eta$ function, and is also run with NAG and $\mu=0.1$. Results are available in Table~\ref{tab:main_budget_adaptive} and Figure~\ref{fig:minecraft_budget_adaptive}. Qualitative results for the Droid dataset under this setting are available in Figure~\ref{fig:droid_vis}. Under a low computational budget of 50 sampling steps or fewer, we find that the adaptive algorithm with NAG consistently sets a lower frontier of performance compared to FM and standard \EqF with NAG, demonstrating the practical utility of \EqF's data-adaptive flexibility across model sizes and datasets.

\paragraph{Closed-Loop Acceleration with NAG.}
We study the interplay between inference acceleration, $\eta$ function properties, and closed loop noise level feedback with $h_\omega$ on the same Minecraft dataset setup in Table~\ref{tab:minecraft_inference_expts}. We test three choices for $\eta$ with differing properties in addition to an identity map. Our results confirm that shaping the landscape into an attractor at inference time with $\eta(\seq{\sigma}) \odot f_{\EqF}(\seq{x}) = \seq{0}$ at the data manifold is necessary for NAG to converge due to its minimum-seeking property. The best setting combines attractor-like schedules (either linear log-SNR or $c$-function), closed-loop prediction, and the higher acceleration parameter $\mu$ that closed-loop prediction enables, displayed in Figure~\ref{fig:closed_open_nag}.
Conversely, the SD3 warp is a non-attractor landscape with $\eta(\seq{\sigma}) \odot f_{EqF}(\seq{x}) = -r \seq{x}$ at the data manifold (see Appendix~\ref{apd:sampler_warp_schedules}); as such, NAG-based sampling with SD3 is prone to overshooting the fixed-point, exhibiting unreliable performance and lagging behind other schedules.
 
\paragraph{Online Replanning Inference.}
We design a novel experimental setting to evaluate the robustness of \EqF and the baselines to reutilize previous predictions when new ground truth context is observed, mirroring a realistic world modeling setting. We find that the flexibility of \EqF allows it to better utilize previous predictions without needing to fully recompute the entire future prediction. 
Specifically, to reach the same FVD as Adaptive \EqF for each computation threshold, Flow Matching requires between 13\% and 29\% more compute. The full experimental results and details are in Appendix~\ref{apd:online_replanning_inference}.

\subsection{Empirical Evidence for Vanishing Gap in Proposition 4.1}
\label{5.5:evidence_for_decomposition}

Proposition~\ref{prop:eqf_risk_decomposition} predicts that for \EqF to achieve loss parity with FM, it must reproduce FM’s field from $\seq{x}^{\seq{\sigma}}$ alone. As evidence that the gap closes, we compare the predictions made by FM and \EqF models on identical held-out noisy inputs. We find that across datasets, the predictions are highly correlated ($>0.98$) and the residual gap accounts for as little as $2.4\%$ of the denoiser output magnitude. Figure~\ref{fig:re10k_ft_residual_gap} further shows that over Re10K finetuning, the gap between the FM and \EqF losses closely tracks the empirical difference in the learned denoising field. Thus \EqF approximately recovers the FM
velocity field, meaning any performance gains are from differences in the inference procedure. By doing so without noise level conditioning, this substantiates our hypothesis that \EqF implicitly identifies the relevant noise level from $\seq{x}^{\seq{\sigma}}$. See Appendix~\ref{apd:residual_gap} for more details.

\begin{figure}[tb]
\centering

\begin{minipage}[t]{0.52\textwidth}
\vspace{0pt}
\centering
\captionof{table}{Ablations of the \EqF model on 300-frame Minecraft generation across warp schedules, samplers, and $\mu$ for one seed. The reported run for each row corresponds to the best setting of $\mu$. Closed loop inference is necessary for higher settings of $\mu$.}
\label{tab:minecraft_inference_expts}
\resizebox{\linewidth}{!}{
\begin{tabular}{lllcc}
\toprule
\textbf{$\eta$ Warp} & \textbf{Feedback} & \textbf{Sampler} & \textbf{$\mu$} & \textbf{FVD$\downarrow$} \\
\midrule
$c$ Function & Closed Loop & NAG & 0.3 & \textbf{64.12} \\
 & Open Loop & NAG & 0.1 & 67.94 \\
 & Open Loop & GD & 0 & 78.69 \\
\midrule
Linear $\log$-SNR & Closed Loop & NAG & 0.3 & 67.70 \\
 & Open Loop & NAG & 0.1 & 90.59 \\
 & Open Loop & GD & 0 & 92.54 \\
\midrule
SD3 $r=3$ & Closed Loop & NAG & 0.1 & 203.73 \\
 & Open Loop & NAG & 0.1 & 114.11 \\
 & Open Loop & GD & 0 & 137.85 \\
\midrule
No Schedule & Closed Loop & NAG & 0.1 & 80.25 \\
 & Open Loop & NAG & 0.1 & 80.25 \\
 & Open Loop & GD & 0 & 105.44 \\
\bottomrule
\end{tabular}
}
\end{minipage}
\hfill
\begin{minipage}[t]{0.42\textwidth}
\vspace{0pt}
\centering
\includegraphics[width=0.86\linewidth]{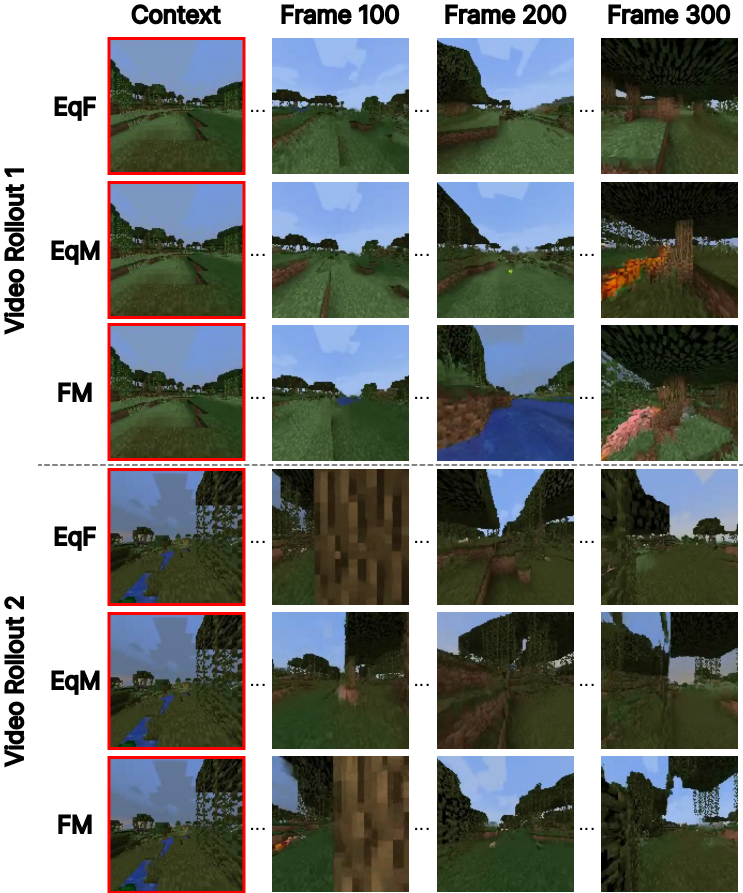}
\captionof{figure}{300-frame autoregressively generated Minecraft videos.}
\label{fig:minecraft_vis}
\end{minipage}

\end{figure}

\section{Related Work}
\label{6:related_work}

\subsection{Autoregressive Video Generation}

Denoising-based autoregressive generative models have already shown tremendous potential for video generation  \citep{brooks2024videogenerationmodelsworldsimulators}. The predominant challenge in autoregressive video generation lies in producing stable rollouts that adhere to the context, currently addressed by introducing strict inference schedules \citep{feng2024matrixinfinitehorizonworldgeneration,song2025historyguidedvideodiffusion}. Our results suggest that in fact, these schedules may not be necessary for stable autoregressive generation, and that the noise conditioning they rely upon may be a key hindrance to the model at inference time. 
Another approach for stabilizing autoregressive rollouts conditioned on generated frames post-trains directly on these generated frames to ameliorate a general train-inference mismatch \citep{huang2025selfforcingbridgingtraintest}. As analyzed in Section~\ref{3:flow_analysis} and Appendix~\ref{sec_bowls}, we have addressed one aspect of train-inference mismatch by removing the noise level condition; we see other post-training approaches as coming from a different, but synergistic angle.

\subsection{Generative Modeling without Noise Level Conditioning}

While a few prior works have studied generative modeling without noise level conditioning, they have primarily focused on image generation as the setting. In addition to being the first work to our knowledge to explore the implications of removing noise level conditioning for temporal data such as video, we clarify the contributions we make over these prior works here.

\paragraph{Noise-Unconditional Flow Matching.}
\cite{sun2025noise} introduced a training recipe for noise-unconditional Flow Matching (NU-FM), though their experiments did not improve over Flow Matching on large-scale image datasets, nor did they explore data-adaptive sampling. NU-FM's training recipe, when adapted to video generation through the framework presented in Section~\ref{2.2:diffusion_forcing}, can be instantiated as a base inference setting of \EqF without any data-adaptive sampling applied on top. Our results with NU-FM substantiate the hypothesis that data-adaptive sampling is the key aspect that allows noise-unconditional models to improve over traditional models.

\paragraph{Equilibrium Matching.}
Equilibrium Matching (EqM) removes explicit noise conditioning from Flow Matching and introduces a modulated training target designed to make the magnitude of the learned field vanish at clean data, enabling attractor-style sampling for image generation \citep{wang2025equilibrium}. However, EqM requires a training-time hyperparameter search over both the shape $c$ and magnitude $\lambda$ of the field. In this work, we demonstrate that separating the noise-unconditional denoising objective from the sampling dynamics is advantageous: \EqF learns an unmodified velocity field and only induces an attractor landscape at inference time through the $\eta$ warp. In addition to offering the simplification of separating training and inference-time decisions, this enables stronger performance via data-adaptive sampling. Our framework further removes an unintended loss weighting introduced by EqM's target modulation, which we find to be suboptimal in practice. In fact, we analyze further how our model subsumes all possible EqM training-time hyperparameters with inference-time decisions in Appendices~\ref{apd:eqmloss},~\ref{apd:eqmsampl}.

\paragraph{Comparison on ImageNet.} Though \EqF is designed for video data, our insights can be applied on ImageNet to compare against prior noise-unconditional image models \citep{5206848}. We train an \EqF model using the same settings as the EqM paper, following the SiT architecture, and train the XL$/2$ size model to 80 epochs \citep{ma2024sit,wang2025equilibrium}. The results in Table~\ref{tab:imagenet_main}, using the inference-time $c$-function $\eta$ warp across 3 seeds, demonstrate that image generation also benefits from data-driven adaptivity and modular training- and inference-time decisions.

\begin{table}[t]
\centering
\caption{
Comparison of FVD and VBench at 50 sampling steps across Minecraft, RealEstate10K, and Droid datasets using the budget-adaptive \EqF algorithm. NAG $\mu=0.1$.
}
\label{tab:main_budget_adaptive}

\renewcommand{\arraystretch}{1.10}

\resizebox{\linewidth}{!}{%
\begin{tabular}{@{}l cc cc cc@{}}
\toprule
& \multicolumn{2}{c}{Minecraft}
& \multicolumn{2}{c}{Re10K}
& \multicolumn{2}{c}{Droid} \\
\cmidrule(lr){2-3}
\cmidrule(lr){4-5}
\cmidrule(lr){6-7}

Method
& FVD $\downarrow$
& VBench $\uparrow$
& FVD $\downarrow$
& VBench $\uparrow$
& FVD $\downarrow$
& VBench $\uparrow$ \\
\midrule

\EqF Adaptive NAG
& $\mathbf{91.45 \pm 0.92}$
& $0.771 \pm 0.000$
& $\mathbf{127.69 \pm 1.57}$
& $\mathbf{0.761 \pm 0.000}$
& $\mathbf{107.38 \pm 0.82}$
& $\mathbf{0.812 \pm 0.002}$ \\

\EqF NAG
& $92.08 \pm 0.30$
& $\mathbf{0.774 \pm 0.000}$
& $166.73 \pm 2.98$
& $0.759 \pm 0.000$
& $113.58 \pm 3.51$
& $0.810 \pm 0.002$ \\

FM
& $135.37 \pm 1.74$
& $0.769 \pm 0.000$
& $145.42 \pm 5.93$
& $\mathbf{0.761 \pm 0.000}$
& $108.69 \pm 3.71$
& $0.809 \pm 0.000$ \\

\bottomrule
\end{tabular}%
}
\end{table}


\begin{table}[t]
\centering

\begin{minipage}[t]{0.595\textwidth}
\centering
\captionof{table}{Autoregressive rollout results on RealEstate10K, finetuned from pretrained Flow Matching model Wan 2.1 1.3B (denoted with -FT). NAG is run with $\mu=0.1$.}
\label{tab:re10k_autoregressive}
\resizebox{\linewidth}{!}{
\begin{tabular}{llcc}
\toprule
\textbf{Method} & \textbf{Sampler} & \textbf{FVD$\downarrow$} & \textbf{VBench$\uparrow$} \\
\midrule
\EqF-FT & Adaptive, NAG  & \best{$139.5 \pm 3.93$} & \best{$0.7472 \pm 0.001$} \\
\EqF-FT & NAG  & {$142.0 \pm 5.80$} & \best{$0.7472 \pm 0.001$} \\
\EqF-FT & GD               & $149.5 \pm 6.06$ & $0.7470 \pm 0.000$ \\
FM-FT   & Euler            & $174.3 \pm 2.17$ & $0.7429 \pm 0.001$ \\
\bottomrule
\end{tabular}
}
\end{minipage}
\hfill
\begin{minipage}[t]{0.38\textwidth}
\centering
\captionof{table}{ImageNet with \EqF's framework ($c$ Function $\eta$)}
\vspace{0.58em}
\label{tab:imagenet_main}
\resizebox{\linewidth}{!}{
\begin{tabular}{llc}
\toprule
\textbf{Method} & \textbf{Sampler} & \textbf{FID$\downarrow$} \\
\midrule
\EqF & NAG, $\mu=0.3$ & \best{$15.15 \pm 0.09$} \\
\EqF & GD & $15.39 \pm 0.09$ \\
EqM & NAG, $\mu=0.3$ & $15.74 \pm 0.10$ \\
EqM & GD & $15.91 \pm 0.08$ \\
FM & Euler & $18.91 \pm 0.10$ \\
\bottomrule
\end{tabular}
}
\end{minipage}

\end{table}

\section{Conclusion}

By training on a simplified noise-unconditional objective, \EqF leverages unique adaptivity and closed-loop inference capabilities to improve quality and consistency on challenging autoregressive video generation tasks. It further elucidates the advantages that come from a modular decoupling of noise-unconditional training and attractor-like inference, surpassing the performance of prior noise-unconditional frameworks. More broadly, our results challenge the assumption that explicit noise conditioning and strict inference schedules are necessary for modern video denoising models. As video becomes an increasingly central modality for world modeling, simulation, and interactive content generation, it is important to distinguish the components that are necessary for stable generation from those adding unnecessary complexity that ultimately constrain performance.

\EqF as presented in this paper carries some limitations. Though \EqF's framework extends to general spatiotemporal denoising generative models, we only consider video in the paper's primary experiments. 
We establish that the error in the noise level readout model is small in Appendix~\ref{apd:noise_from_data}, but it could still contribute a source of error, such as if the partially noised sample is off the training distribution. Finally, while the datasets considered here are large, they are not on the general video domain. Future work pretraining \EqF at a larger scale could support more general conclusions; the purpose of this present work is to compare the performance of \EqF and the baselines under a matched training budget. Training for longer or using data mixtures as is done in \cite{chen2025large} could yield additional improvements in quality and instruction following on our datasets.

\newpage

\begin{ack}

This work is supported in part by the U.S. Army Research Office under Army-ECASE award W911NF-07-R-0003-03; the U.S. Department of Energy, Office of Science, ARPA-H-SOL-24-101 program; IARPA HAYSTAC Program; DARPA YFA; NSF Grants \#2146151, \#2205093, \#2146343, \#2134274, \#2441832, and CCF-2112665; and CDC-RFA-FT-23-0069. This work was also funded by the Minderoo Foundation (project title: eDNAID: Environmental DNA and AI as Tools for Ocean Aid) under award number CLB-3502. 
This work has been made possible in part by a gift from the Chan Zuckerberg Initiative Foundation to establish the Kempner Institute for the Study of Natural and Artificial Intelligence at Harvard University.


\end{ack}


\bibliography{example_paper}
\bibliographystyle{neurips_2026}


\newpage 
\appendix

\addcontentsline{toc}{part}{Appendix}

\begingroup
\etocsettocdepth{section}
\etocsettocstyle{\section*{Appendix Contents}}{}
\localtableofcontents
\endgroup

\newpage

\newpage
\section{Analysis on Noise Unconditional Denoising}
\label{apd:apdx_extended_eqm}

In this section, we elaborate on how \EqF subsumes EqM as a training and inference time framework for noise unconditional denoising. 


\subsection{Effect of EqM Modulation on Target}
\label{apd:eqmloss}

In the denoising generative modeling literature, the objective used during training under finite compute has been observed to impact the perceptual quality of samples generated by the learned denoiser \citep{karras2022elucidating, esser2024scaling}. We demonstrate here that the EqM objective explored in \cite{wang2025equilibrium} induces a weighting over noise levels in the corresponding clean-data prediction objective due to the extra modulation term. For $c(\sigma)$ considered in EqM, the weighting places comparatively more emphasis on high-noise samples than the \EqF objective.

\begin{proposition}[Induced $x$-prediction loss weightings]
\label{prop:eqm_eqf_xpred_weightings}
The EqM and \EqF velocity losses are equivalent to the following weighted clean data prediction losses:
\begin{align}
\mathcal{L}_{\mathrm{EqM}}
&=
\operatorname*{\mathbb{E}}_{x,\sigma,\epsilon}
\left[
\frac{c(\sigma)^2}{\sigma^2}
\left\|
x_{\mathrm{EqM}}(x^\sigma,\sigma) - x
\right\|_2^2
\right],
\\
\mathcal{L}_{\EqF}
&=
\operatorname*{\mathbb{E}}_{x,\sigma,\epsilon}
\left[
\frac{1}{\sigma^2}
\left\|
x_{\EqF}(x^\sigma,\sigma) - x
\right\|_2^2
\right].
\end{align}
\end{proposition}

\begin{proof}
For simplicity, we analyze the case where there is only a single noise level in the sample, $\sigma$, though the argument generalizes to when $\seq{\sigma}$ is $T$-dimensional. Let
\begin{align}
    v := \epsilon - x,
    \qquad
    x^\sigma := (1-\sigma)x + \sigma \epsilon.
\end{align}

Using the definition of $v$, the noised sample can equivalently be written as
\begin{align}
    x^\sigma
    &= (1-\sigma)x + \sigma \epsilon \notag \\
    &= x + \sigma(\epsilon - x) \notag \\
    &= x + \sigma v.
\end{align}

Therefore, the clean data can be recovered from the noised sample and the velocity target:
\begin{align}
    x = x^\sigma - \sigma v.
    \label{eq:x_from_velocity}
\end{align}

EqM trains a noise-unconditional network to predict a modulated velocity target $c(\sigma)v$:
\begin{align}
\mathcal{L}_{\mathrm{EqM}}
=
\operatorname*{\mathbb{E}}_{x,\sigma,\epsilon}
\left[
\left\|
f_{\mathrm{EqM}}(x^\sigma)
-
c(\sigma)v
\right\|_2^2
\right].
\end{align}

To understand the effect that $c(\sigma)$ has on the loss weighting, we can convert this objective to $x$-prediction. Since $f_{\mathrm{EqM}}(x^\sigma)$ is trained to estimate $c(\sigma) v$, the corresponding estimate of $v$ is $f_{\mathrm{EqM}}(x^\sigma)/c(\sigma)$. We can then define the clean-data prediction induced by EqM by substituting back into Equation~\ref{eq:x_from_velocity}. To convert, $x_{\mathrm{EqM}}$ becomes reliant on $\sigma$ as well as the noised data; $x_{\mathrm{EqM}}$ is an induced $x$-prediction model based on the trained $v$-prediction function. 
\begin{align}
    x_{\mathrm{EqM}}(x^\sigma,\sigma)
    :=
    x^\sigma
    -
    \frac{\sigma}{c(\sigma)}
    f_{\mathrm{EqM}}(x^\sigma).
    \label{eq:eqm_x_prediction_def}
\end{align}

Using Equations~\ref{eq:x_from_velocity}~and~\ref{eq:eqm_x_prediction_def}, we can now rewrite the EqM loss in terms of the error between the induced $x$-prediction and the true $x$:
\begin{align}
    x_{\mathrm{EqM}}(x^\sigma,\sigma) - x
    &=
    \left(
    x^\sigma
    -
    \frac{\sigma}{c(\sigma)}
    f_{\mathrm{EqM}}(x^\sigma)
    \right)
    -
    \left(
    x^\sigma - \sigma v
    \right) \notag \\
    &=
    \sigma v
    -
    \frac{\sigma}{c(\sigma)}
    f_{\mathrm{EqM}}(x^\sigma) \notag \\
    &=
    -\frac{\sigma}{c(\sigma)}
    \left(
    f_{\mathrm{EqM}}(x^\sigma)
    -
    c(\sigma)v
    \right).
\end{align}

Rearranging gives
\begin{align}
    f_{\mathrm{EqM}}(x^\sigma)
    -
    c(\sigma)v
    =
    -\frac{c(\sigma)}{\sigma}
    \left(
    x_{\mathrm{EqM}}(x^\sigma;\sigma) - x
    \right).
\end{align}

The original loss for the modulated target prediction becomes the following reweighted $x$-loss
\begin{align}
\mathcal{L}_{\mathrm{EqM}}
&=
\operatorname*{\mathbb{E}}_{x,\sigma,\epsilon}
\left[
\left\|
f_{\mathrm{EqM}}(x^\sigma)
-
c(\sigma)v
\right\|_2^2
\right] \notag \\
&=
\operatorname*{\mathbb{E}}_{x,\sigma,\epsilon}
\left[
\frac{c(\sigma)^2}{\sigma^2}
\left\|
x_{\mathrm{EqM}}(x^\sigma,\sigma) - x
\right\|_2^2
\right].
\label{eq:eqm_general_xpred_weight}
\end{align}

The noise-level weight is given by $w_{\mathrm{EqM}} = \frac{c(\sigma)^2}{\sigma^2}$. For example, if $c(\sigma) = \sigma$, then the weighting term becomes $w_{\mathrm{EqM}}=1$, yielding an unweighted clean-data prediction objective:
\begin{align}
\mathcal{L}_{\mathrm{EqM}}
&=
\operatorname*{\mathbb{E}}_{x,\sigma,\epsilon}
\left[
\left\|
x_{\mathrm{EqM}}(x^\sigma) - x
\right\|_2^2
\right].
\end{align}

In contrast, \EqF trains directly on the unmodulated velocity target:
\begin{align}
\mathcal{L}_{\EqF}
=
\operatorname*{\mathbb{E}}_{x,\sigma,\epsilon}
\left[
\left\|
f_{\EqF}(x^\sigma)
-
v
\right\|_2^2
\right].
\end{align}

The corresponding clean-data estimate is
\begin{align}
    x_{\EqF}(x^\sigma,\sigma)
    :=
    x^\sigma - \sigma f_{\EqF}(x^\sigma).
\end{align}

Repeating the same calculation,
\begin{align}
    x_{\EqF}(x^\sigma,\sigma) - x
    &=
    \left(x^\sigma - \sigma f_{\EqF}(x^\sigma)\right)
    -
    \left(x^\sigma - \sigma v\right) \notag \\
    &=
    -\sigma
    \left(
    f_{\EqF}(x^\sigma) - v
    \right),
\end{align}

which implies
\begin{align}
\mathcal{L}_{\EqF}
&=
\operatorname*{\mathbb{E}}_{x,\sigma,\epsilon}
\left[
\left\|
f_{\EqF}(x^\sigma)
-
v
\right\|_2^2
\right] \notag \\
&=
\operatorname*{\mathbb{E}}_{x,\sigma,\epsilon}
\left[
\frac{1}{\sigma^2}
\left\|
x_{\EqF}(x^\sigma,\sigma) - x
\right\|_2^2
\right].
\label{eq:eqf_xpred_weight}
\end{align}
\end{proof}

\vspace{-1em}

Comparing Equations~\ref{eq:eqm_general_xpred_weight} and~\ref{eq:eqf_xpred_weight}, EqM and \EqF differ by the factor $c(\sigma)^2$ in their induced $x$-prediction weighting. \EqF's loss weighting more heavily weights low-noise regions, similar to Flow Matching. 

Theoretically, all such weightings across noise levels $\sigma$ do not change the pointwise Bayes-optimal denoising target for noise conditional models \citep{kingma2023understandingdiffusionobjectiveselbo}. The same theoretical conclusion holds for noise-unconditional models under the assumption that the noise level posterior concentrates in high dimensions and is thus directly identifiable (studied more in Appendix~\ref{apd:eqmsampl}). However, in practice, finite capacity models are sensitive to how the training signal is distributed across noise levels; state-of-the-art denoising generative models for images have found weightings with increased emphasis on high noise levels (such as the one in EqM) to be less performant \citep{karras2022elucidating,esser2024scaling}. Similar results have been found under close examination of noise-sampling distributions in autoregressive denoising generative modeling for complex dynamical systems \citep{cachay2025elucidated}. Our experimental results in Section~\ref{5:experiments} comparing samples from EqM to \EqF corroborate these findings: the \EqF loss places less relative emphasis on the highest-noise regions compared to EqM and achieves better results.

Concurrent work shows that velocity modulation in the training objective, as used in EqM, can lead to biased terminal distributions after integrating the noise-unconditional flow \citep{cheukkit2026beckmanntransportmodelsautonomous}. This provides complementary theoretical support for EqF's unmodulated training objective from a different motivation.


\subsection{EqM $c$-Function Inference Relationship to \EqF}
\label{apd:eqmsampl}

Here, we continue the discussion from Section~\ref{4.3:eta_schedule_explanation} on how \EqF modularly decouples training and inference. We show that, under concentration of the noise-level posterior established in \cite{sahraee2026geometry,kadkhodaie2026blind}, an EqM model trained with target modulation $c(\seq{\sigma})$ is equivalent to applying the same $c$-function as an inference-time warp to the \EqF velocity field. Thus, in the Bayes-optimal and posterior concentration limit, a trained \EqF model can subsume EqM-style target modulation by moving the $c$-function to the sampler via the warp $\eta$, as in Equation~\ref{eq:sampler_update}.

\begin{proposition}[EqM training-time modulation can be rewritten as an inference-time \EqF warp]
\label{prop:eqm_modulation_as_eqf_warp}
The EqM Bayes-optimal denoiser satisfies
\begin{align}
f_{\textnormal{EqM}}^\star(\seq{x}^{\seq{\sigma}})
&\approx
c(\widehat{\seq{\sigma}}) \odot 
f_{\EqF}^{\star}
\left(
\seq{x}^{\seq{\sigma}}
\right)  \notag \\
&\propto \eta(\widehat{\seq{\sigma}}) \odot 
f_{\EqF}^{\star}
\left(
\seq{x}^{\seq{\sigma}}
\right),
\label{eq:eqm_to_eqf_concentration}
\end{align}
where the final proportionality holds when the inference-time warp $\eta$ is chosen to have the same shape as $c$, up to normalization.
\end{proposition}

\begin{proof}
Recall the noising process and velocity target from the main text:
\begin{align}
    \seq{v}
    :=
    \seq{\epsilon} - \seq{x},
    \qquad
    \seq{x}^{\seq{\sigma}}
    :=
    (1-\seq{\sigma}) \odot \seq{x}
    +
    \seq{\sigma} \odot \seq{\epsilon}. \notag 
\end{align}

For a squared loss regression objective, the Bayes-optimal denoiser is the conditional expectation of the target given the model input. For \EqF, the model input is only $\seq{x}^{\seq{\sigma}}$, and the training target is the unmodulated velocity $\seq{v}$, yielding the following Bayes-optimal denoiser:
\begin{align}
\mathcal{L}_{\EqF} &= \operatorname*{\mathbb{E}}_{\seq{x},\seq{\epsilon},\seq{\sigma}}
\left [
\left |
\left |
f_{\EqF}(\seq{x}^{\seq{\sigma}}) - \seq{v}
\right |
\right |^2_2
\right ], \qquad
    f_{\EqF}^{\star}(\seq{x}^{\seq{\sigma}})
    =
    \operatorname*{\mathbb{E}}_{\seq{x},\seq{\epsilon},\seq{\sigma}}
    \left[
        \seq{v}
        \mid
        \seq{x}^{\seq{\sigma}}
    \right].
\end{align}

Applying the law of total expectation by first conditioning on $\seq{\sigma}$ gives an outer expectation over possible noise levels under $p(\seq{\sigma} | \seq{x}^{\seq{\sigma}})$, and an inner expectation over possible clean data and noise pairs under $p(\seq{x}, \seq{\epsilon} | \seq{x}^{\seq{\sigma}}, \seq{\sigma})$:
\begin{align}
    f_{\EqF}^{\star}(\seq{x}^{\seq{\sigma}})
    &=
    \operatorname*{\mathbb{E}}_{\seq{\sigma}}
    \left[
        \operatorname*{\mathbb{E}}_{\seq{x},\seq{\epsilon}}
        \left[
            \seq{v}
            \mid
            \seq{x}^{\seq{\sigma}}, \seq{\sigma}
        \right]
        \mid
        \seq{x}^{\seq{\sigma}}
    \right].
    \label{eq:eqf_bayes_expanded}
\end{align}

EqM has the same model input, but the training target is the modulated velocity $c(\seq{\sigma})\odot \seq{v}$.
\begin{align}
    f_{\textrm{EqM}}^{\star}(\seq{x}^{\seq{\sigma}})
    &=
    \operatorname*{\mathbb{E}}_{\seq{x},\seq{\epsilon},\seq{\sigma}}
    \left[
        c(\seq{\sigma}) \odot \seq{v}
        \mid
        \seq{x}^{\seq{\sigma}}
    \right] \notag \\
    &=
    \operatorname*{\mathbb{E}}_{\seq{\sigma}}
    \left[
        c(\seq{\sigma}) \odot 
        \operatorname*{\mathbb{E}}_{\seq{x},\seq{\epsilon}}
        \left[
            \seq{v}
            \mid
            \seq{x}^{\seq{\sigma}}, \seq{\sigma}
        \right]
        \mid
        \seq{x}^{\seq{\sigma}}
    \right].
    \label{eq:eqm_bayes_expanded}
\end{align}

Under concentration in high dimensions \citep{sahraee2026geometry, kadkhodaie2026blind}, the posterior estimate of the noise level given noisy data collapses to a Dirac delta function at a point mass at $\widehat{\seq{\sigma}}$. Thus, any outer expectation over $\seq{\sigma}$ for any function $g$ when conditioned on $\seq{x}^{\seq{\sigma}}$ collapses to evaluation at $\widehat{\seq{\sigma}}$:
\begin{align}
    p(\seq{\sigma} \mid \seq{x}^{\seq{\sigma}})
    \approx
    \delta(\seq{\sigma} - \widehat{\seq{\sigma}}), \qquad 
    \operatorname*{\mathbb{E}}_{\seq{\sigma}}
    \left[
        g(\seq{\sigma})
        \mid
        \seq{x}^{\seq{\sigma}}
    \right]
    \approx
    g(\widehat{\seq{\sigma}}).
    \label{eq:noise_post} 
\end{align}

Applying this to Equations~\ref{eq:eqf_bayes_expanded} and \ref{eq:eqm_bayes_expanded} and combining, with a final assumption of the \EqF sampler using an inference-time warp $\eta$ with the same shape as $c$, gives
\begin{align}
    f_{\EqF}^{\star}(\seq{x}^{\seq{\sigma}})
    &\approx
    \operatorname*{\mathbb{E}}_{\seq{x},\seq{\epsilon}}
    \left[
        \seq{v}
        \mid
        \seq{x}^{\seq{\sigma}}, \widehat{\seq{\sigma}}
    \right],
    \label{eq:eqf_concentrated} \\
    f_{\textrm{EqM}}^{\star}(\seq{x}^{\seq{\sigma}})
    &\approx
    c(\widehat{\seq{\sigma}}) \odot
    \operatorname*{\mathbb{E}}_{\seq{x},\seq{\epsilon}}
    \left[
        \seq{v}
        \mid
        \seq{x}^{\seq{\sigma}}, \widehat{\seq{\sigma}}
    \right]
     \notag \\
    &\approx
    c(\widehat{\seq{\sigma}}) \odot 
    f_{\EqF}^{\star}
    \left(
    \seq{x}^{\seq{\sigma}}
    \right)  \notag \\
    &\propto
    \eta(\widehat{\seq{\sigma}}) \odot 
    f_{\EqF}^{\star}
    \left(
    \seq{x}^{\seq{\sigma}}
    \right).
    \label{eq:eqm_concentrated} 
\end{align}
\end{proof}

As described in Appendix~\ref{apd:sampler_warp_schedules}, a valid transport warp must integrate to 1, so the proportionality between $c$ and $\eta$ should not be interpreted as saying that every EqM $c$-function directly defines a valid probability-flow transport. The EqM paper chose to define the $c$-function generally such that $c(\sigma) \to 0$ as $\sigma \to 0$, which does not in general define a valid transport. EqM introduced another training-time hyperparameter $\lambda$ to modulate the magnitude of the $c$-function. When taken together with the inference-time step size, EqM approximates this constraint empirically, falling short of identifying the theoretically correct value and requiring large, expensive training-time sweeps over both $\lambda$ and the $c$-function that are unnecessary in our framework.

Together with Appendix~\ref{apd:eqmloss}, this demonstrates that EqM's training-time target modulation is \textit{theoretically unnecessary} under the stated noise-level concentration and introduces an undesirable weighting over noise levels during training, as we confirm experimentally in Section~\ref{5:experiments}. In high dimensions, this suggests that an equilibrium landscape together with attractor-style sampling are better treated as modular inference time decisions rather than being baked into the training objective.


\subsection{Proof of Proposition~\ref{prop:eqf_risk_decomposition}}
\label{loss_decom}

In this section we expand the loss decomposition in Proposition~\ref{prop:eqf_risk_decomposition}. This decomposition demonstrates that the loss for a denoiser without noise level conditioning is lower bounded by the one with noise level conditioning, and that if the \EqF denoiser can minimize the pointwise denoiser error for any particular noise level, it can achieve the same optima.

\begin{proof}
The Bayes-optimal denoiser for Flow Matching is the conditional expectation of the target given the model input:
\begin{align*}
    f_{\textrm{FM}}^\star(\seq{x}^{\seq{\sigma}},\seq{\sigma})
    =
    \operatorname*{\mathbb{E}}_{\seq{x},\seq{\epsilon}}
    \left[
        \seq{v}
        \mid
        \seq{x}^{\seq{\sigma}},\seq{\sigma}
    \right].
\end{align*}
For brevity, throughout this proof we write
\begin{align*}
    f_{\textrm{FM}}^\star
    &:=
    f_{\textrm{FM}}^\star(\seq{x}^{\seq{\sigma}},\seq{\sigma}),
    \qquad
    f_{\EqF}^\star
    :=
    f_{\EqF}^\star(\seq{x}^{\seq{\sigma}}).
\end{align*}
By the definition of $f_{\textrm{FM}}^\star$, the Flow Matching residual has zero conditional mean:
\begin{align}
    \operatorname*{\mathbb{E}}_{\seq{x},\seq{\epsilon}}
    \left[
        \seq{v}
        -
        f_{\textrm{FM}}^\star
        \mid
        \seq{x}^{\seq{\sigma}},\seq{\sigma}
    \right]
    &=
    \operatorname*{\mathbb{E}}_{\seq{x},\seq{\epsilon}}
    \left[
        \seq{v}
        \mid
        \seq{x}^{\seq{\sigma}},\seq{\sigma}
    \right]
    -
    f_{\textrm{FM}}^\star = \seq{0}
\end{align}

We now expand the optimal \EqF risk by adding and subtracting $f_{\textrm{FM}}^\star$:
\begin{align}
\mathcal L_{\EqF}^\star
&=
\operatorname*{\mathbb{E}}_{\seq{x},\seq{\epsilon},\seq{\sigma}}
\left[
\left\|
\seq{v}
-
f_{\EqF}^\star
\right\|_2^2
\right]
\notag
\\
&=
\operatorname*{\mathbb{E}}_{\seq{x},\seq{\epsilon},\seq{\sigma}}
\left[
\left\|
\left(\seq{v}-f_{\textrm{FM}}^\star\right)
+
\left(f_{\textrm{FM}}^\star-f_{\EqF}^\star\right)
\right\|_2^2
\right]
\notag
\\
&=
\operatorname*{\mathbb{E}}_{\seq{x},\seq{\epsilon},\seq{\sigma}}
\left[
\left\|
\seq{v}-f_{\textrm{FM}}^\star
\right\|_2^2
\right]
+
\operatorname*{\mathbb{E}}_{\seq{x},\seq{\epsilon},\seq{\sigma}}
\left[
\left\|
f_{\textrm{FM}}^\star-f_{\EqF}^\star
\right\|_2^2
\right]
\notag
\\
&\qquad
+
2\operatorname*{\mathbb{E}}_{\seq{x},\seq{\epsilon},\seq{\sigma}}
\left[
\left\langle
\seq{v}-f_{\textrm{FM}}^\star,
f_{\textrm{FM}}^\star-f_{\EqF}^\star
\right\rangle
\right]
\notag
\\
&=
\mathcal L_{\textrm{FM}}^\star
+
\operatorname*{\mathbb{E}}_{\seq{x},\seq{\epsilon},\seq{\sigma}}
\left[
\left\|
f_{\textrm{FM}}^\star-f_{\EqF}^\star
\right\|_2^2
\right]
\notag
\\
&\qquad
+
2\operatorname*{\mathbb{E}}_{\seq{x}^{\seq{\sigma}},\seq{\sigma}}
\left[
\operatorname*{\mathbb{E}}_{\seq{x},\seq{\epsilon}}
\left[
\left\langle
\seq{v}-f_{\textrm{FM}}^\star,
f_{\textrm{FM}}^\star-f_{\EqF}^\star
\right\rangle
\mid
\seq{x}^{\seq{\sigma}},\seq{\sigma}
\right]
\right]
\notag
\\
&=
\mathcal L_{\textrm{FM}}^\star
+
\operatorname*{\mathbb{E}}_{\seq{x},\seq{\epsilon},\seq{\sigma}}
\left[
\left\|
f_{\textrm{FM}}^\star-f_{\EqF}^\star
\right\|_2^2
\right]
\notag
\\
&\qquad
+
2\operatorname*{\mathbb{E}}_{\seq{x}^{\seq{\sigma}},\seq{\sigma}}
\left[
\left\langle
\operatorname*{\mathbb{E}}_{\seq{x},\seq{\epsilon}}
\left[
\seq{v}-f_{\textrm{FM}}^\star
\mid
\seq{x}^{\seq{\sigma}},\seq{\sigma}
\right],
f_{\textrm{FM}}^\star-f_{\EqF}^\star
\right\rangle
\right]
\notag
\\
&=
\mathcal L_{\textrm{FM}}^\star
+
\operatorname*{\mathbb{E}}_{\seq{x},\seq{\epsilon},\seq{\sigma}}
\left[
\left\|
f_{\textrm{FM}}^\star-f_{\EqF}^\star
\right\|_2^2
\right]
+
2\operatorname*{\mathbb{E}}_{\seq{x}^{\seq{\sigma}},\seq{\sigma}}
\left[
\left\langle
\seq{0},
f_{\textrm{FM}}^\star-f_{\EqF}^\star
\right\rangle
\right]
\notag
\\
&=
\mathcal L_{\textrm{FM}}^\star
+
\operatorname*{\mathbb{E}}_{\seq{x},\seq{\epsilon},\seq{\sigma}}
\left[
\left\|
f_{\textrm{FM}}^\star(\seq{x}^{\seq{\sigma}},\seq{\sigma})
-
f_{\EqF}^\star(\seq{x}^{\seq{\sigma}})
\right\|_2^2
\right].
\label{eq:eqf_risk_decomposition}
\end{align}
\end{proof}

Considering the form of the Bayes-optimal denoisers for Flow Matching and \EqF (Equation~\ref{eq:eqf_bayes_expanded}), recall that \EqF marginalizes Flow Matching's Bayes-optimal denoisers over $p(\seq{\sigma} \mid \seq{x}^{\seq{\sigma}})$. As in Equation~\ref{eq:noise_post} if this posterior concentrates \citep{sahraee2026geometry, kadkhodaie2026blind}, the Bayes-optimal denoiser is approximately: 

\begin{align*}
    p(\seq{\sigma} \mid \seq{x}^{\seq{\sigma}})
    \approx
    \delta(\seq{\sigma} - \widehat{\seq{\sigma}})
    \implies
    f_{\EqF}^\star(\seq{x}^{\seq{\sigma}})
    \approx
    \operatorname*{\mathbb{E}}_{\seq{x},\seq{\epsilon}}
    \left[
        \seq{v}
        \mid
        \seq{x}^{\seq{\sigma}}, \widehat{\seq{\sigma}}
    \right]
\end{align*}

Taking into account that the gap $\mathcal L_{\EqF}^\star - \mathcal L_{\textrm{FM}}^\star$ is the expected difference in regression targets, under concentration the \EqF model must drive the gap to zero to minimize the objective, suggesting that to do so it must learn to implicitly estimate the noise level posterior.

\subsection{FM-to-\EqF Residual Gap}
\label{apd:residual_gap}

Here, we provide the full table for the experimental analysis provided in Section~\ref{5.5:evidence_for_decomposition}. We evaluate the paired FM and \EqF predictions on identical held-out noisy inputs. 
Table~\ref{tab:fm_eqf_field_agreement} shows that \EqF obtains a denoising loss only slightly larger than FM, matching the lower bound prediction from the main text. The comparatively small FM/\EqF residual explains the gap in the loss, especially so for Re10K and Droid. 
Their predicted velocity fields obtain low normalized discrepancy and high Pearson correlations ($>0.98$) across all three datasets. We additionally show the convergence of the output denoising field Pearson correlation across training of the Re10K model in Figure~\ref{fig:re10k_train_pearson_correlations}. Droid is evaluated on only 1,417 samples since that is the size of the validation dataset. These results support the conclusion that \EqF approximately recovers the FM velocity field without explicit noise conditioning, meaning the gains in performance come from getting rid of the drifted noise conditions and from the application of adaptive algorithms.

\begin{figure}[t]
    \centering
    \includegraphics[width=0.5\linewidth]{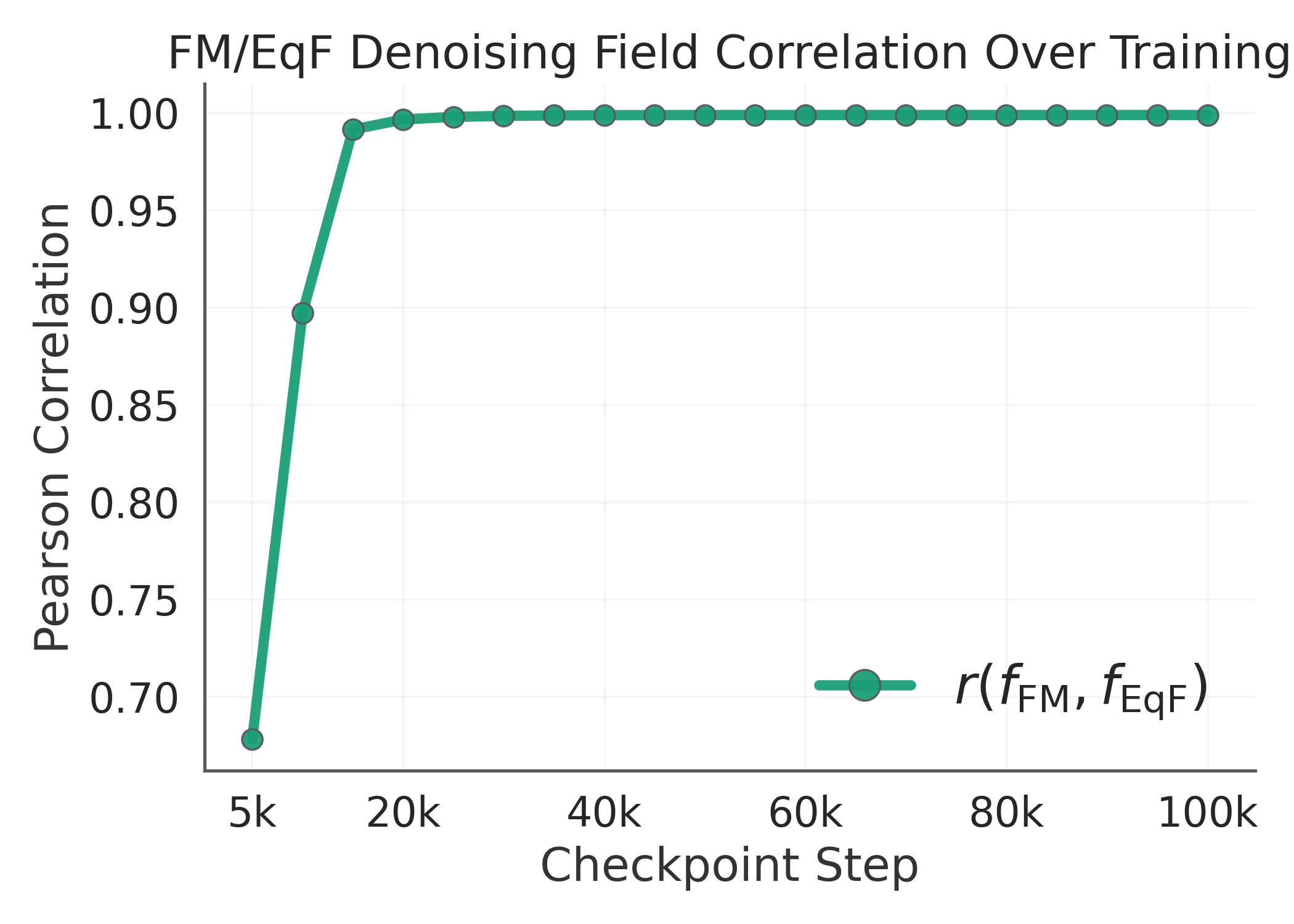}
    \caption{Pearson Correlation convergence across training of the Re10K FM and \EqF models.}
    \label{fig:re10k_train_pearson_correlations}
\end{figure}

\begin{table}[t!]
\centering
\caption{
\textbf{Empirical agreement between trained FM and \EqF denoisers.}
All quantities are evaluated on matched held-out samples using the same
clean inputs, noise levels, and Gaussian noise.
Relative field MSE denotes the magnitude of 
$\mathbb{E}\|f_{\textrm{FM}}-f_{\EqF}\|_2^2$
normalized by
$\tfrac{1}{2}
\bigl(
\mathbb{E}\|f_{\textrm{FM}}\|_2^2+
\mathbb{E}\|f_{\EqF}\|_2^2
\bigr)$.
Across all three domains, the trained fields have small normalized
discrepancy and high Pearson correlation despite \EqF receiving no
explicit noise-level condition.
}
\label{tab:fm_eqf_field_agreement}
\resizebox{\textwidth}{!}{
\begin{tabular}{@{}lcccccc@{}}
\toprule
\textbf{Dataset}
& \textbf{\# Samples}
& \shortstack{\textbf{FM/\EqF}\\\textbf{Field MSE}}
& \multicolumn{2}{c}{\textbf{Validation $v$-Loss}}
& \shortstack{\textbf{Relative}\\\textbf{Field MSE} $(\%)$}
& \shortstack{\textbf{Pearson}\\\textbf{Correlation}} \\
\cmidrule(lr){4-5}
&
&
$\mathbb{E}\|f_{\textrm{FM}}-f_{\EqF}\|_2^2$
&
$\mathbb{E}\|f_{\textrm{FM}}-\seq{v}\|_2^2$
&
$\mathbb{E}\|f_{\EqF}-\seq{v}\|_2^2$
&
&
\\
\midrule
Minecraft
& $2{,}084$
& $0.0425 \pm 0.00505$
& $0.576 \pm 0.00270$
& $0.576 \pm 0.00271$
& $7.378$
& $0.984$
\\
Re10K
& $4{,}096$
& $0.00347 \pm 0.00003$
& $0.142 \pm 0.00146$
& $0.143 \pm 0.00147$
& $2.432$
& $0.999$
\\
Droid
& $1{,}417$
& $0.0135 \pm 0.000385$
& $0.307 \pm 0.00972$
& $0.318 \pm 0.00993$
& $4.325$
& $0.996$
\\
\bottomrule
\end{tabular}
}
\end{table}



\section{Additional Analysis on the Effect of Divergent Noise Levels in Flow Matching}
\label{apd:additional_analysis}

In this section, we provide further analysis on the effect of noise conditioning divergence during autoregressive inference in Flow Matching. Our analysis suggests a causal relationship between this noise level divergence and resulting degraded performance at inference time.

\subsection{Flow Matching Predictions Frequently Terminate Scheduled Sampling at Nonzero Noise}
\label{sec_nonzero}

In Section~\ref{3:flow_analysis}, we observe a persistent residual between the estimated true noise level of the sample and the scheduled noise level during Flow Matching inference. We use the same experimental setup over 300-frame Minecraft autoregressive rollouts here and visualize the progress of the predicted noise level of the sample during inference with a fixed number of denoising steps. The solid line in Figure~\ref{fig:unfinished_denoising} represents the mean of the distribution of framewise noise levels over inference, and the dotted line is the scheduled noise level at a particular solver step using the $c$-function as $\eta$. Any gap between the lines at end of the figure demonstrates that sampling with Flow Matching ends with a certain amount of noise remaining, and the miss rates plotted indicate the frequency with which a sample terminates below a specified noise level. Zooming into the end of the denoising process, we see that scheduled inference consistently terminates with a nonzero noise level, indicating that insufficient denoising is a source of error during Flow Matching inference. This further impacts autoregressive inference in that the conditioning signal for \textit{generated context} is also inaccurate, not just over the horizon as analyzed in Section~\ref{3:flow_analysis}.

\begin{figure}[tb]
    \centering
    \includegraphics[width=\linewidth]{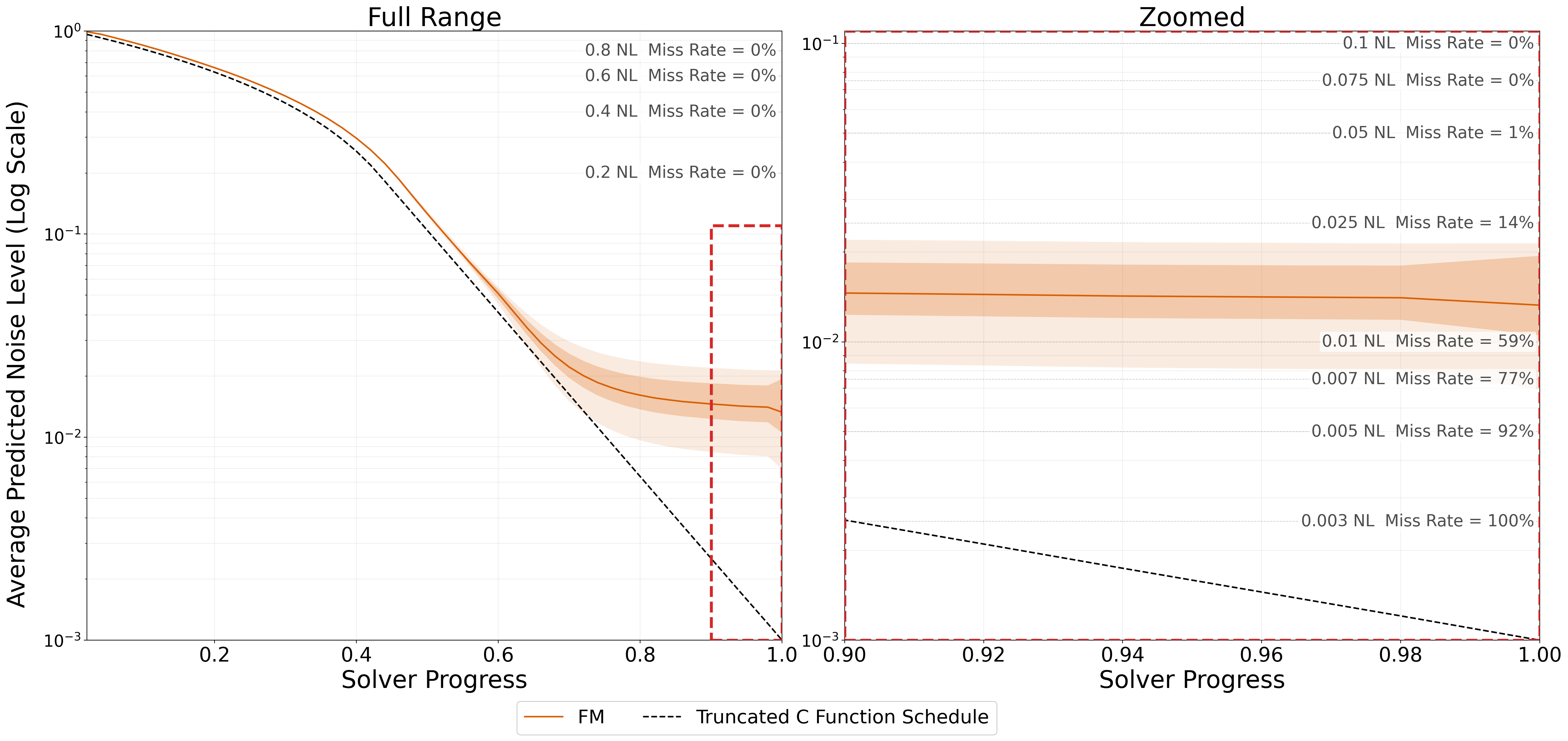}
    \caption{Noise level readout across sampling steps. Ideally, sampling ends at noise level $0$, but Flow Matching samples terminate at a significant, nonzero noise level, which is a potential contributing factor to autoregressive error accumulation. The right shows a zoomed view of the red box from the left plot.}
    \label{fig:unfinished_denoising}
\end{figure}

\subsection{Noise Conditioning Divergence Directly Causes Increased Denoising Loss}
\label{sec_bowls}

To measure the direct effect of divergent noise level conditions per inference step on Flow Matching, we noise data to a particular level using the forward process, and then condition the model on an artificially perturbed noise level that is not the same as the true noise level of the data point. Matching our autoregressive inference setting (Section~\ref{3:flow_analysis} and Section~\ref{5:experiments}), we consider denoising windows of 50 frames from the Minecraft dataset. We apply the forward process with noise level $\sigma_{\text{stable}}$ to the latter 25 frames, constituting the prediction horizon, while retaining a clean context of length 25 frames. We study the effect of incorrect noise levels on the denoising loss when the noise level is correct on different combinations of the 25 active horizon frames and the 25 context frames.

\paragraph{Noise Level Divergence on the Active Horizon Causes Denoising Error.} We first simulate the noise level divergence only over the prediction horizon. The presence of such errors, where the scheduled noise level is different than the estimated true noise level during inference, is established by Section~\ref{3:flow_analysis}. We give the correct noise level of $\seq{0}$ for the 25 context frames, and we give a simulated diverged noise level condition $({\sigma_{\text{stable}}} + \delta) \cdot \ones$ for the prediction horizon. Specifically, the true noise level of the data is $\tilde{\seq{\sigma}} = [\seq{0}, \sigma_\text{stable} \cdot \ones]$, and the actual conditioning signal is $\seq{\sigma} = [\seq{0}, (\sigma_\text{stable} + \delta) \cdot \ones ]$. 

The discrepancy between $\tilde{\seq{\sigma}}$ and $\seq{\sigma}$ causes a significant increase in the denoising $v$-loss for the Flow Matching model, as shown in Figure~\ref{fig:bowls} (left). At the bottom of each `bowl' is the denoising loss when the condition matches the data, and the surrounding perturbed points are normalized as the increase in the loss caused by the mismatch. We thus see that as the conditioned noise level diverges further from the truth, the error monotonically increases. This discrepancy, attributable to rigid noise conditioning during Flow Matching training, is one reason for the decreased downstream performance once noise levels diverge.

\paragraph{Inaccurate Context Noise Level Causes Further Denoising Loss.} In Subsection~\ref{sec_nonzero}, we observe that when following a scheduled path of noise levels with Flow Matching, a sample often terminates at nonzero noise. During autoregressive rollouts, the inference procedure relies on the conditioning signal of generated frames to be accurate. This suggests that any further inaccuracy in the context conditioning signal may further harm performance. 

To simulate the effect of generated context terminating at nonzero noise, we measure the $v$-loss when the noise level condition is corrupted for both the context and the prediction horizon. We measure the loss of a Flow Matching model when paired with the incorrect noise level of $\delta \cdot \ones$ for the context and $({\sigma_{\text{stable}}} + \delta) \cdot \ones$ for the prediction horizon. Specifically, the true noise level of the noisy data is $\tilde{\seq{\sigma}} = [\seq{0}, \sigma_\text{stable} \cdot \ones  ]$, and the actual conditioning signal is $\seq{\sigma} = [\delta \cdot \ones, (\sigma_\text{stable} + \delta) \cdot \ones ]$.

When comparing the loss from this experiment to the previous case when only the prediction horizon's signal is corrupted, we find that the loss with mismatched noise level conditions for both the context and prediction horizon is strictly higher than the case where only the horizon is corrupted, as visualized in Figure~\ref{fig:bowls} (right). 

\begin{figure}[tb]
    \centering
    \includegraphics[width=0.99\linewidth]{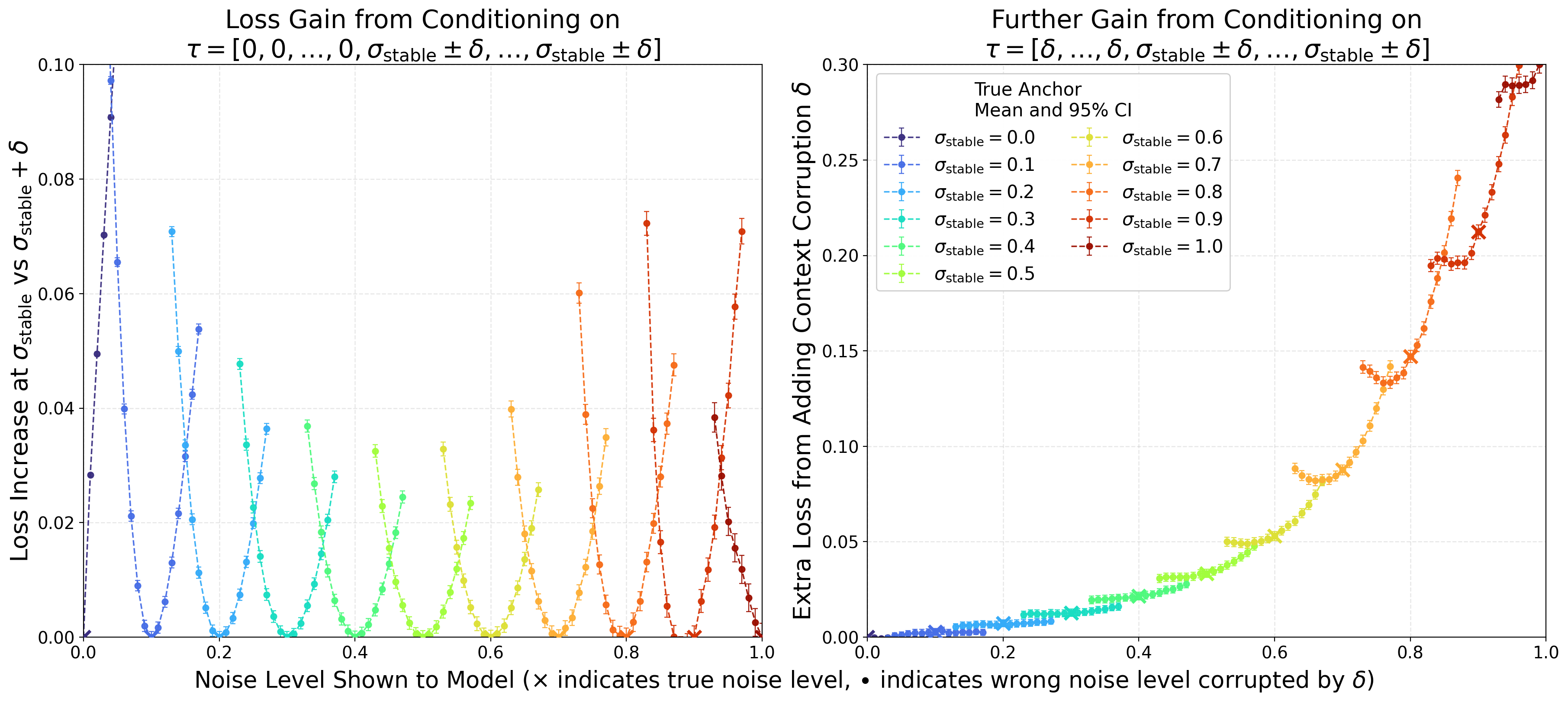}
    \caption{Noise conditioning error introduced artificially, with the loss calculated only on the active sampling horizon. Left: horizon only. Right: gain from adding context corruption. Noise conditions are clamped to $[0,1]$, and errors are standardized with respect to $\delta \in [-0.07, 0.07]$.}
    \label{fig:bowls}
\end{figure}

\subsection{Simulated Closed-Loop Flow Matching Sampling Improves Sample Quality}
\label{apd:fixed_fm_with_g}

Here, we demonstrate that using the predicted noise level from $\tilde{\sigma} \approx g_\phi(x^\sigma)$ for Flow Matching inference can improve performance of Flow Matching, but that directly conditioning on noise levels still produces worse quality compared to \EqF. More details on the noise level estimator $g_\phi$ can be found in Appendix~\ref{apd:noise_from_data}. The evaluation setting is identical to Sections~\ref{3:flow_analysis} and~\ref{5:experiments}, with 300-frame rollouts on Minecraft. 

Accelerated gradient-based samplers such as NAG typically do not work for noise-conditional models, since the sample would step to unknown noise levels, causing severe conditioning divergence and highly inaccurate denoising. As shown in Table~\ref{tab:fm_conditioning_ablation}, using the predicted noise levels to close the sampling loop leads to significant improvements. This is especially the case for NAG, where the sample diverges completely without the closed-loop correction. While these results for closed-loop sampling with Flow Matching demonstrate improvements over naive open-loop sampling schedules, we find that these fixes are just extra steps that can be removed when using \EqF's framework. Specifically, closing the loop by \textit{removing noise conditioning altogether} and relying on the model's internal estimates of progress produces much stronger results, as our main results in Section~\ref{5:experiments} demonstrate \EqF's improvements in the same evaluation setting.

\begin{table}[tb]
\centering
\caption{Flow Matching autoregressive generation results under different inference schedules and conditioning choices. Conditioning on the inferred noise level $g_\phi(x^{(i)})$ substantially improves FVD, especially when combined with NAG under the $c$ Function inference schedule.}

\label{tab:fm_conditioning_ablation}
\resizebox{0.75 \textwidth}{!}{
\begin{tabular}{ccccc}
\toprule
\textbf{Method} & \textbf{Sampler} & \textbf{Inference Schedule} & \textbf{Conditioning} & \textbf{FVD$\downarrow$} \\
\midrule
Flow Matching & Euler & Identity & $\sigma^{(i)}$ & $126.633$ \\
Flow Matching & Euler & Identity & $\tilde{\sigma}^{(i)} \approx g_\phi(x^{(i)})$ & $121.145$ \\
\midrule
Flow Matching & Euler & $c$-Function & $\sigma^{(i)}$ & $103.610$ \\
Flow Matching & Euler & $c$-Function & $\tilde{\sigma}^{(i)} \approx g_\phi(x^{(i)})$ & $94.488$ \\
\midrule
Flow Matching & NAG, $\mu=0.2$ & $c$-Function & $\sigma^{(i)}$ & $1045.633$ \\
Flow Matching & NAG, $\mu=0.2$ & $c$-Function & $\tilde{\sigma}^{(i)} \approx g_\phi(x^{(i)})$ & \best{$81.193$} \\
 \bottomrule
\end{tabular}
}
\end{table}

\subsection{Closed Loop Sampling with \EqF Produces the Lowest Final Noise Level}
\label{sec_nonzero_comparison}

In Appendix~\ref{sec_nonzero}, we observe that Flow Matching samples consistently terminate at nonzero noise. Here, we provide additional comparisons between the noise level termination values of samples from closed-loop \EqF sampling, and of samples from Flow Matching and from NU-FM \citep{sun2025noise}. In Figure~\ref{fig:unfinished_denoising_comparison}, the black dashed line represents the scheduled noise level, and the result for each sampler is shown by different colors.

The closed-loop samples with \EqF terminate at lower noise levels (almost an order of magnitude lower than Flow Matching), exhibiting the robustness of closing the loop with a noise-unconditional sampler via its internal estimator. Furthermore, we can see that sampling with NAG lowers the noise level that the denoised sample reaches, as well as accelerates the denoising progress overall (the noise level readout curve falls significantly below the schedule). This analysis provides further intuition for how closed-loop sampling with NAG produces higher quality samples by converging better to $\eta(\widehat{\seq{\sigma}}) f_{\EqF}(\seq{x}) = \seq{0}$, as demonstrated experimentally in Section~\ref{5:experiments}.

\begin{figure}[tb]
    \centering

    \includegraphics[width=\linewidth]{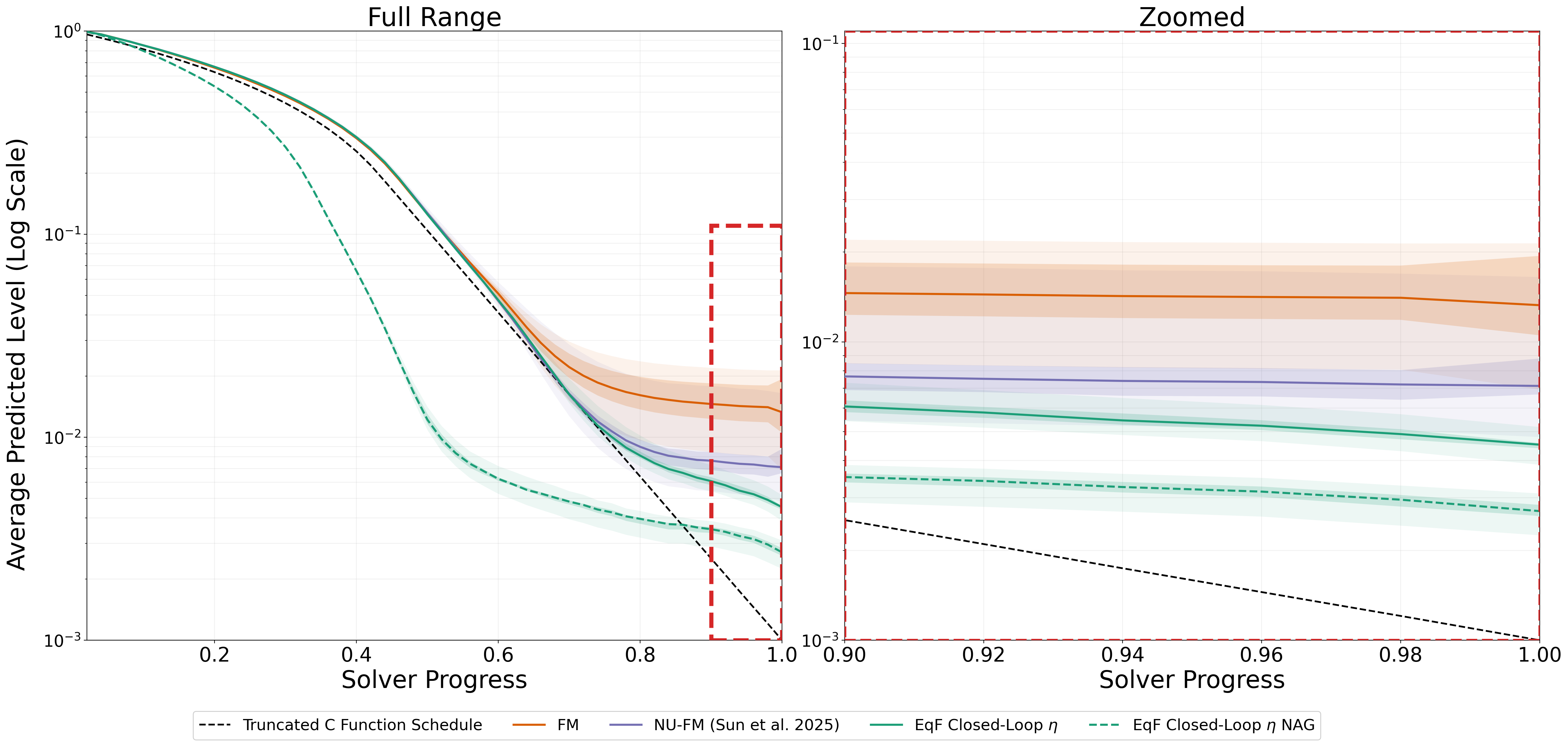}
    \caption{Comparison of the noise level that sampling terminates at across FM, NU-FM, \EqF, and \EqF with NAG. Ideally, sampling ends at a noise level of $0$. Closed-loop sampling with \EqF terminates at a lower noise level, and NAG accelerates the schedule and ends at an even lower noise level. The right shows a zoomed view of the red box from the left plot.}
    \label{fig:unfinished_denoising_comparison}
\end{figure}

\section{Additional Results}
\label{apd:additional_results}

In this section, we describe an additional data-adaptive early stopping algorithm and results on Minecraft. Further, we present full tables for the budget-adaptive results in the main text, and an additional inpainting task on the Re10K dataset.

\subsection{Adaptive Early Stopping Algorithm}
\label{apd:adaptive_early_stopping}
The noise level estimation capabilities of \EqF provide a natural convergence criterion for a data-adaptive early stopping algorithm, presented in Algorithm~\ref{alg:adaptive_early_stopping}. Once the estimated noise level falls below a threshold, $\widehat{\sigma}< \sigma_\text{thresh}$, the sampler enters a short refinement period with a decaying step size before completing denoising for those indices. This refinement period implements the budget-adaptive algorithm presented in Section~\ref{4.5:budget_adaptive_algorithm} for a few additional `cleanup' steps, removing the final amount of noise. Noise-driven adaptive early stopping contrasts with how prior work defines stopping criteria for noise-unconditional inference based on the gradient magnitude falling below a threshold \citep{wang2025equilibrium}. Previously, it was unclear how this threshold should be chosen with relation to the convergence of the data; the magnitude of the update itself is relatively arbitrary, and would have to be tuned for each dataset.

On the Minecraft dataset, we demonstrate that the adaptive early stopping algorithm effectively reduces the number of steps taken under different values of the stopping criteria $\sigma_\text{thresh}$, and in some cases also improves the performance over constant-budget samples. The results are presented in Figure~\ref{fig:minecraft_adaptive_early_stopping} and Table~\ref{tab:minecraft_adaptive_early_stopping}. These results highlight the flexibility of \EqF's inference framework: compute requirements can be derived from the sample itself rather than from a fixed hyperparameter of the sampler. We hypothesize that it can perform better than standard \EqF algorithm in some cases due to the ability to ``finish'' the denoising based on the noise estimate and the number of steps left, whereas the standard \EqF can leave too much noise in the sample, especially when taking a lower number of steps.

The inference setting is described in Appendix~\ref{apd:minecraft_dataset_details}; it is different than the other experiments, so the FVD numbers are not directly comparable. We use a number of cleanup steps $C_s=8$, and the value of $\sigma_\text{thresh}$ ranges from $0$ (no early stopping) to $0.7$. The detection of $\sigma_\text{thresh}$ triggers based on the highest noise level of a frame within the next chunk of frames to be emitted, and the number of cleanup steps is run for that chunk of frames; multiple chunks of frames can be in `cleanup mode' simultaneously.

\begin{figure}[t]
    \centering
    \includegraphics[width=0.5\linewidth]{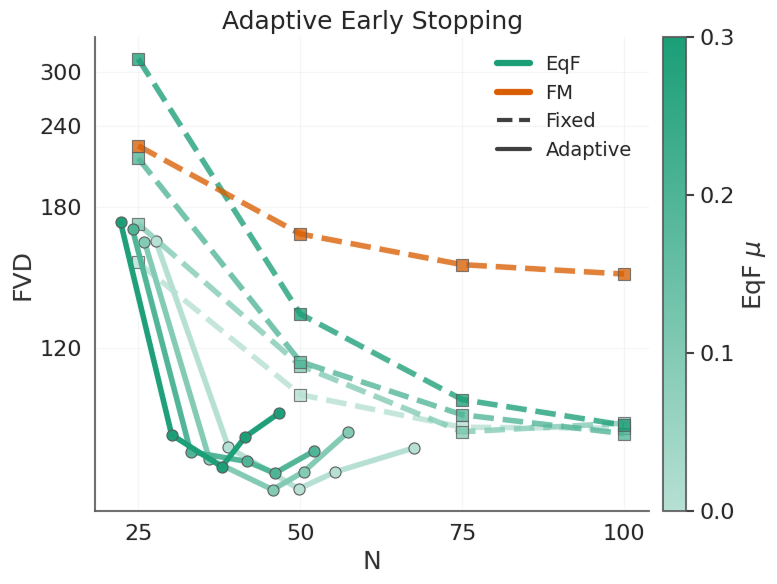}
    \caption{\textbf{Early Stopping on Minecraft with \EqF.} The early stopping algorithm using noise level readouts optimizes the compute/performance tradeoff. Each value of $\mu$ is evaluated with stopping thresholds $\sigma_\text{thresh}=0.1,0.2,0.3,0.5,0.7$}
    \label{fig:minecraft_adaptive_early_stopping}
\end{figure}

\begin{algorithm}[t]
\caption{\EqF Adaptive Early Stopping Inference}
\label{alg:adaptive_early_stopping}

\begin{lstlisting}[style=pytorch]
# f_eqf(x): trained equilibrium v field
# h_omega(z): noise level readout estimator
# sigma_thresh: stopping threshold
# C_s: number of cleanup steps
# eta(sigma): sampling warp
sigma_hat = ones(num_frames)
x = randn(sample_shape)

# adaptive early stopping
while max(sigma_hat) > sigma_thresh:
    v_hat, activations = f_eqf(x)
    sigma_hat = h_omega(activations)
    x = x - eta(sigma_hat) * v_hat

# cleanup refinement
for j in range(C_s):
    v_hat, activations = f_eqf(x)
    sigma_hat = h_omega(activations)
    x = x - sigma_hat / (C_s - j) * v_hat

return x
\end{lstlisting}
\end{algorithm}

\begin{table}[p]
\centering
\caption{\textbf{Adaptive early stopping experiments on Minecraft with \EqF.} For adaptive runs, the stopping noise level induces an empirical number of function evaluations $N$. Higher setting of $\mu$ for NAG makes the adaptive stopping condition trigger sooner. All runs use $c$-function as the $\eta$ warp. These are the values plotted in Figure~\ref{fig:minecraft_adaptive_early_stopping}.}
\label{tab:minecraft_adaptive_early_stopping}
\resizebox{0.85\textwidth}{!}{
\begin{tabular}{lllcc}
\toprule
\textbf{Inference Mode} & \textbf{Sampler} & \textbf{Stopping Noise Level $\sigma_{\text{thresh}}$} & \textbf{$N$} & \textbf{FVD$\downarrow$} \\
\midrule
Adaptive Early Stopping & \EqF GD & 0.7 & $27.7$ & $161.67$ \\
 & & 0.5 & $38.9$ & $94.87$ \\
 & & 0.3 & $49.8$ & \best{$86.70$} \\
 & & 0.2 & $55.4$ & $89.88$ \\
 & & 0.1 & $67.6$ & $94.75$ \\
\midrule
Adaptive Early Stopping & \EqF NAG $\mu$=0.1 & 0.7 & $26.0$ & $161.59$ \\
 & & 0.5 & $36.0$ & $92.41$ \\
 & & 0.3 & $45.7$ & \best{$86.55$} \\
 & & 0.2 & $50.5$ & $89.94$ \\
 & & 0.1 & $57.4$ & $98.12$ \\
\midrule
Adaptive Early Stopping & \EqF NAG $\mu$=0.2 & 0.7 & $24.1$ & $167.80$ \\
 & & 0.5 & $33.1$ & $93.92$ \\
 & & 0.3 & $41.8$ & $91.95$ \\
 & & 0.2 & $46.1$ & \best{$89.65$} \\
 & & 0.1 & $52.1$ & $94.04$ \\
\midrule
Adaptive Early Stopping & \EqF NAG $\mu$=0.3 & 0.7 & $22.3$ & $171.65$ \\
 & & 0.5 & $30.1$ & $97.53$ \\
 & & 0.3 & $37.9$ & \best{$90.90$} \\
 & & 0.2 & $41.5$ & $97.11$ \\
 & & 0.1 & $46.7$ & $102.44$ \\
\midrule
Fixed & \EqF GD & 0 & 25 & $151.78$ \\
 & & 0 & 50 & $107.05$ \\
 & & 0 & 75 & $99.22$ \\
 & & 0 & 100 & $98.99$ \\
\midrule
Fixed & \EqF NAG $\mu$=0.1 & 0 & 25 & $170.55$ \\
 & & 0 & 50 & $114.91$ \\
 & & 0 & 75 & $98.20$ \\
 & & 0 & 100 & $100.16$ \\
\midrule
Fixed & \EqF NAG $\mu$=0.2 & 0 & 25 & $212.75$ \\
 & & 0 & 50 & $116.03$ \\
 & & 0 & 75 & $102.12$ \\
 & & 0 & 100 & $97.79$ \\
\midrule
Fixed & \EqF NAG $\mu$=0.3 & 0 & 25 & $318.11$ \\
 & & 0 & 50 & $131.33$ \\
 & & 0 & 75 & $105.81$ \\
 & & 0 & 100 & $99.70$ \\
\midrule
Fixed & Flow Matching & 0 & 25 & $222.67$ \\
 & & 0 & 50 & $165.41$ \\
 & & 0 & 75 & $150.82$ \\
 & & 0 & 100 & $146.77$ \\
\bottomrule
\end{tabular}
}
\end{table}

\subsection{Budget-Adaptive Full Results}

The full results for the budget-adaptive experiments from Section~\ref{5:experiments} on the Minecraft, Re10K, and Droid datasets are detailed in Table~\ref{tab:apdx_full_budget_adaptive}.

\subsection{Inpainting Results on RealEstate10K}
\label{apd:inpainting}
We present an additional set of results on an inpainting task on the RealEstate10K dataset. The model is given frame 1 and the final 4 frames (corresponding to the first and last latent frames). It must denoise the intermediate frames, constituting 49 total frames. The given frames are treated as context in the same context-conditioned style, where they have zero noise and the remaining frames do have noise. The task is not autoregressive, and the model has all the information about the end of the denoising window, making it much easier for the denoising model. We find that on this easier task, the inpainting is essentially perfect, and the models have comparable metric scores (PSNR, FVD, VBench) over 256 generated videos per seed with 5 seeds. This is in contrast to the more difficult open-ended and autoregressive task presented in the main text. The results are presented in Table~\ref{tab:re10k_inpaint}.

\begin{figure}[p!]
\centering

\begin{minipage}{\textwidth}
\centering
\captionof{table}{
Full results table for FVD across different numbers of sampling steps, using the budget-adaptive sampling algorithm for \EqF on Minecraft, Re10K, and Droid. \EqF Adaptive NAG and \EqF NAG are run with $\mu=0.1$. FM is run with Euler sampling. Runs are across 3 seeds.
}
\label{tab:apdx_full_budget_adaptive}

\scriptsize
\setlength{\tabcolsep}{5pt}
\renewcommand{\arraystretch}{1.03}

\textbf{Minecraft: FVD $\downarrow$}

\vspace{0.15em}
\resizebox{\linewidth}{!}{%
\begin{tabular}{l ccccc}
\toprule
Sampling steps & 10 & 20 & 30 & 40 & 50 \\
\midrule
\EqF Adaptive NAG
& $\mathbf{271.22 \pm 2.19}$
& $\mathbf{126.29 \pm 4.23}$
& $\mathbf{103.16 \pm 6.33}$
& $\mathbf{99.59 \pm 1.92}$
& $\mathbf{91.45 \pm 0.92}$ \\
\EqF NAG
& $952.46 \pm 11.91$
& $246.89 \pm 6.65$
& $144.55 \pm 4.64$
& $109.50 \pm 5.19$
& $92.08 \pm 0.30$ \\
FM
& $570.30 \pm 3.72$
& $240.81 \pm 11.47$
& $170.53 \pm 6.55$
& $145.77 \pm 5.74$
& $135.37 \pm 1.74$ \\
\bottomrule
\end{tabular}%
}

\vspace{0.55em}

\textbf{Minecraft: VBench $\uparrow$}

\vspace{0.15em}
\resizebox{\linewidth}{!}{%
\begin{tabular}{l ccccc}
\toprule
Sampling steps & 10 & 20 & 30 & 40 & 50 \\
\midrule
\EqF Adaptive NAG
& $\mathbf{0.770 \pm 0.001}$
& $0.772 \pm 0.002$
& $0.772 \pm 0.002$
& $0.772 \pm 0.001$
& $0.771 \pm 0.000$ \\
\EqF NAG
& $0.750 \pm 0.001$
& $\mathbf{0.773 \pm 0.001}$
& $\mathbf{0.774 \pm 0.001}$
& $\mathbf{0.774 \pm 0.001}$
& $\mathbf{0.774 \pm 0.000}$ \\
FM
& $0.755 \pm 0.001$
& $0.766 \pm 0.000$
& $0.767 \pm 0.001$
& $0.768 \pm 0.001$
& $0.769 \pm 0.000$ \\
\bottomrule
\end{tabular}%
}

\vspace{1.55em}

\textbf{Re10K: FVD $\downarrow$}

\vspace{0.15em}
\resizebox{\linewidth}{!}{%
\begin{tabular}{l ccccc}
\toprule
Sampling steps & 10 & 20 & 30 & 40 & 50 \\
\midrule
\EqF-FT Adaptive NAG
& $\mathbf{165.15 \pm 0.53}$
& $\mathbf{149.58 \pm 3.54}$
& $\mathbf{134.20 \pm 2.20}$
& $\mathbf{130.19 \pm 1.23}$
& $\mathbf{127.69 \pm 1.57}$ \\
\EqF-FT NAG
& $244.41 \pm 2.58$
& $184.11 \pm 1.19$
& $172.56 \pm 1.54$
& $168.67 \pm 2.62$
& $166.73 \pm 2.98$ \\
FM-FT
& $169.13 \pm 4.47$
& $161.09 \pm 3.16$
& $155.34 \pm 3.96$
& $148.92 \pm 5.56$
& $145.42 \pm 5.93$ \\
\bottomrule
\end{tabular}%
}

\vspace{0.55em}

\textbf{Re10K: VBench $\uparrow$}

\vspace{0.15em}
\resizebox{\linewidth}{!}{%
\begin{tabular}{l ccccc}
\toprule
Sampling steps & 10 & 20 & 30 & 40 & 50 \\
\midrule
\EqF-FT Adaptive NAG
& $\mathbf{0.757 \pm 0.000}$
& $\mathbf{0.759 \pm 0.000}$
& $\mathbf{0.760 \pm 0.000}$
& $\mathbf{0.761 \pm 0.001}$
& $\mathbf{0.761 \pm 0.000}$ \\
\EqF-FT NAG
& $0.748 \pm 0.000$
& $0.757 \pm 0.000$
& $0.758 \pm 0.001$
& $0.758 \pm 0.001$
& $0.759 \pm 0.000$ \\
FM-FT
& $\mathbf{0.757 \pm 0.000}$
& $\mathbf{0.759 \pm 0.000}$
& $\mathbf{0.760 \pm 0.000}$
& $0.760 \pm 0.000$
& $\mathbf{0.761 \pm 0.000}$ \\
\bottomrule
\end{tabular}%
}

\vspace{1.55em}

\textbf{Droid: FVD $\downarrow$}

\vspace{0.15em}
\resizebox{\linewidth}{!}{%
\begin{tabular}{l ccccc}
\toprule
Sampling steps & 10 & 20 & 30 & 40 & 50 \\
\midrule
\EqF-FT Adaptive NAG
& $\mathbf{150.89 \pm 6.32}$
& $\mathbf{119.40 \pm 3.20}$
& $\mathbf{112.82 \pm 3.70}$
& $\mathbf{111.36 \pm 4.74}$
& $\mathbf{107.38 \pm 0.82}$ \\
\EqF-FT NAG
& $223.89 \pm 4.50$
& $145.92 \pm 4.96$
& $125.47 \pm 4.26$
& $116.96 \pm 2.72$
& $113.58 \pm 3.51$ \\
FM-FT
& $207.59 \pm 6.70$
& $133.86 \pm 4.66$
& $117.99 \pm 6.09$
& $111.56 \pm 5.25$
& $108.69 \pm 3.71$ \\
\bottomrule
\end{tabular}%
}

\vspace{0.55em}

\textbf{Droid: VBench $\uparrow$}

\vspace{0.15em}
\resizebox{\linewidth}{!}{%
\begin{tabular}{l ccccc}
\toprule
Sampling steps & 10 & 20 & 30 & 40 & 50 \\
\midrule
\EqF-FT Adaptive NAG
& $\mathbf{0.805 \pm 0.001}$
& $\mathbf{0.810 \pm 0.001}$
& $\mathbf{0.810 \pm 0.003}$
& $\mathbf{0.811 \pm 0.003}$
& $\mathbf{0.812 \pm 0.002}$ \\
\EqF-FT NAG
& $0.791 \pm 0.002$
& $0.805 \pm 0.000$
& $0.807 \pm 0.001$
& $0.810 \pm 0.002$
& $0.810 \pm 0.002$ \\
FM-FT
& $0.789 \pm 0.003$
& $0.805 \pm 0.002$
& $0.808 \pm 0.000$
& $0.809 \pm 0.001$
& $0.809 \pm 0.000$ \\
\bottomrule
\end{tabular}%
}
\end{minipage}

\vspace{1.8em}

\begin{minipage}{\textwidth}
\centering
\captionof{table}{Inpainting Results with \EqF on Re10K}
\label{tab:re10k_inpaint}

\scriptsize
\resizebox{0.75\linewidth}{!}{%
\begin{tabular}{llccc}
\toprule
\textbf{Method}
& \textbf{Sampler}
& \textbf{PSNR$\uparrow$}
& \textbf{FVD$\downarrow$}
& \textbf{VBench$\uparrow$} \\
\midrule
\EqF-FT & NAG, $\mu=0.05$
& \best{$20.34 \pm 0.06$}
& $37.28 \pm 1.14$
& \best{$0.7738 \pm 0.0002$} \\

\EqF-FT & GD
& $20.32 \pm 0.06$
& $37.82 \pm 1.05$
& $0.7737 \pm 0.0003$ \\

FM-FT & Euler
& $20.30 \pm 0.04$
& \best{$35.29 \pm 0.65$}
& $0.7736 \pm 0.0004$ \\
\bottomrule
\end{tabular}%
}
\end{minipage}

\vspace{1.8em}

\begin{minipage}{\textwidth}
\centering
\captionof{table}{
Equilibrium Finetuning on Minecraft.
Autoregressive Minecraft generation results across Equilibrium
Finetuning settings compared to a Flow Matching model with continued
training. All share an equivalent budget of 100k extra steps.
All are run with 250 sampling steps.
}
\label{tab:mc_eq_finetune}

\scriptsize
\resizebox{0.5\linewidth}{!}{%
\begin{tabular}{lllcc}
\toprule
\textbf{Method}
& \textbf{$\eta$ Warp}
& \textbf{Sampler}
& \textbf{$\mu$}
& \textbf{FVD$\downarrow$} \\
\midrule
\EqF & $c$ Function & NAG   & 0.2 & \best{$74.54$} \\
EqM  & $c$ Function & NAG   & 0.1 & $75.98$ \\
FM   & $c$ Function & Euler & --  & $100.48$ \\
\bottomrule
\end{tabular}%
}
\end{minipage}

\end{figure}

\section{Experiment and Dataset Details}
\label{apd:expt_details}

In this section, we provide details on the experimental testbeds for the three video datasets Minecraft, RealEstate10K, and Droid. All experiment hyperparameters are provided in Table~\ref{tab:training_details}. We also provide additional experiment details for the ImageNet experiments. All model variants are evaluated on comparable training budgets; FLOP and runtime analysis are included at the end of this section, Subsection~\ref{D.6:compute_analysis}. For more information on the geometry of autoregressive generation schedules (including stride length described below), see Appendix~\ref{apd:inference_sched_math}. 

\begin{table*}[!p]
\centering
\small
\caption{Training settings for Minecraft, Re10K, and Droid across models. The Rectified Flow schedule is from \cite{liu2022flow} and the Cosine schedule is from \cite{nichol2021improved}. Minecraft is trained from scratch. For Re10K, we finetune Wan T2V 1.3B \citep{wan2025wanopenadvancedlargescale}. For Droid, we finetune Wan TI2V 5B.}
\label{tab:training_details}
\renewcommand{\arraystretch}{1.2}
\resizebox{\textwidth}{!}{
\begin{tabular}{@{}lccc@{}}
\toprule
\textbf{Config}
& \textbf{Minecraft}
& \textbf{Re10K}
& \textbf{Droid} \\
\midrule

\textbf{Training} \\
\quad Effective Batch Size
& 8
& 8
& 16 \\
\quad Learning Rate
& 2e-4 (Linear Warmup)
& 1e-4 (Linear Warmup)
& 1e-5 (Linear Warmup) \\
\quad Warmup Steps
& 2{,}000
& 2{,}000
& 2{,}000 \\
\quad Weight Decay
& 1e-3
& 1e-3
& 1e-3 \\
\quad Training Steps
& 580k
& 100k
& 60k \\
\quad GPU Usage
& 1$\times$H200
& 2$\times$H200
& 4$\times$H200 \\
\quad Optimizer
& Adam, betas=(0.9, 0.99)
& Adam, betas=(0.9, 0.99)
& Adam, betas=(0.9, 0.99) \\
\quad Training Strategy
& Distributed Data Parallel
& Distributed Data Parallel
& Distributed Data Parallel \\
\quad Precision
& Bfloat16
& Bfloat16
& Bfloat16 \\
\quad EMA for Eval
& 0.9999
& 0.9999
& 0.9996 \\
\quad Noise/Context Strategy
& Random Independent
& Random Independent
& Random Length Context \\

\midrule
\textbf{Equilibrium Forcing (\EqF)} \\
\quad Objective Target
& $\seq v$
& $\seq v$
& $\seq v$ \\
\quad Noise Schedule
& Rectified Flow
& Rectified Flow
& Rectified Flow \\

\midrule
\textbf{Flow Matching} \citep{liu2022flow} \\
\quad Objective Target
& $\seq v$
& $\seq v$
& $\seq v$ \\
\quad Noise Schedule
& Rectified Flow
& Rectified Flow
& Rectified Flow \\

\midrule
\textbf{Equilibrium Matching} \citep{wang2025equilibrium} \\
\quad Objective Target
& $c(\seq \sigma) \odot \seq v$
& --
& -- \\
\quad Noise Schedule
& Rectified Flow
& --
& -- \\
\quad Equilibrium Multiplier $\lambda$
& 4
& --
& -- \\

\midrule
\textbf{Diffusion} \citep{sohldickstein2015diffusion} \\
\quad Objective Target
& $\seq v$
& --
& -- \\
\quad Noise Schedule
& Cosine
& --
& -- \\
\quad Loss Weighting
& Sigmoid
& --
& -- \\

\midrule
\textbf{VAE Details} \\
\quad Input Dimension
& $50 \times 256 \times 256 \times 3$
& $49 \times 256 \times 256 \times 3$
& $49 \times 384 \times 640 \times 3$ \\
\quad Latent Dimension
& $50 \times 32 \times 32 \times 4$
& $13 \times 32 \times 32 \times 16$
& $13 \times 24 \times 40 \times 48$ \\
\quad Temporal Downsampling
& 1
& 4
& 4 \\
\quad Spatial Downsampling
& 8
& 8
& 16 \\
\quad VAE Type
& Framewise
& Causal
& Causal \\
\quad VAE Parameters
& 83 M
& 126 M
& 705 M \\

\midrule
\textbf{Model} \\
\quad Total Parameters
& 95.3 M
& 1.3 B
& 5 B \\
\quad \# Attention Heads
& 12
& 12
& 24 \\
\quad Head Dimension
& 64
& 128
& 128 \\
\quad \# Layers
& 10
& 30
& 30 \\
\quad Time Embed Dimension
& 256
& 256
& 256 \\
\quad Condition Embed Type
& Action
& Camera Pose
& UMT5 Language Prompt \\
\quad Condition Embed Dimension
& 768
& --
& 4096 (UMT5) / 3072 (Model) \\
\bottomrule
\end{tabular}
}
\end{table*}

\subsection{Evaluation Metrics}

\paragraph{Frechet Video Distance (FVD).} FVD is the primary metric that we report to measure the quality of autoregressively generated video, with low scores being better \citep{unterthiner2018towards}. FVD first extracts features from generated and true samples through an image feature extractor I3D \citep{carreira2017quo}, and then calculates the Frechet distance in the latent space between the sets of samples \citep{unterthiner2018towards}. Autoregressive video generation does not have a ground truth to compare to once the video is long enough, so FVD is particularly helpful for assessing long horizon consistency and stability over generation.

\paragraph{VBench.} We report VBench to evaluate the semantic quality of generated video on axes such as frame quality, temporal consistency, and nontrivial motion \citep{huang2024vbench}. VBench provides a granular, human-aligned metric to assess the tradeoff of consistency and video dynamics, allowing us to quantify whether \EqF improves visual quality and long range consistency while still capturing the motion of the video. We utilize the same mix of VBench sub-metrics as  \cite{song2025historyguidedvideodiffusion}. The overall metric we report is computed as the weighted average of the normalized VBench submetrics: aesthetic quality, imaging quality, subject consistency, background consistency, temporal flickering, motion smoothness, and dynamic degree, where dynamic degree has weight 0.5 and all other submetrics have weight 1.

\paragraph{Frechet Inception Distance (FID).} FID \citep{heusel2017gans} is a distributional metric similar to FVD, widely adopted to measure generation quality for images. We use this metric for the additional ImageNet experiments. 

\subsection{Minecraft}
\label{apd:mc_expt_details_settings}

\paragraph{Dataset Details.} 
\label{apd:minecraft_dataset_details}
The Minecraft Dataset \citep{yan2023temporally} is a collection of 200k videos collected from the Minecraft video game, each with a length of 300 frames. We employ a similar setup to \cite{song2025historyguidedvideodiffusion}, where each frame is upsampled to 256 $\times$ 256 pixels. Every video is coupled with an action sequence, making this an action-conditioned generation task. Each frame's action is one of three choices: forward, turn left, or turn right. All videos were collected in the same biome, so the data is relatively homogeneous with respect to the scene, compared to all possible variants of video from the Minecraft game. The egocentric controller rotates along with the turning actions, causing parts of the scene to come in and out of view, necessitating a strong generative model to adhere to the changes and changing context.

Following standard practice, all denoising models are trained in the VAE latent space \citep{rombach2022highresolutionimagesynthesislatent}. We use the pretrained VAE on the Minecraft dataset from \cite{song2025historyguidedvideodiffusion}, which is a framewise VAE without temporal downsampling. 

\paragraph{Standard Inference Setting.} The standard inference setting is a 300 frame rollout. We use 25 ground truth frames to provide context for the generation of 275 future frames. The context window and active sampling horizon are both 25 frames long, forming a 50-frame sliding window. We report FVD and VBench scores over 256 generated videos per seed with 5 seeds for runs in our primary experiments (Section~\ref{5:experiments}). We use 250 sampling steps and an emission stride of 1 frame at a time. The inference-time variant experiments in Section~\ref{5.3:inference_minecraft} use 1 seed.

\paragraph{Budget-Adaptive Inference Setting.}
The adaptive inference setting uses a 300 frame rollout with a sliding window length of 50, with 25 context frames and 25 active sampling horizon frames. The emission stride is 5, and the number of sampling steps ranges from 10 to 50 with increments of 10. The results are averaged across 3 seeds.

\paragraph{Adaptive Early Stopping Inference Setting.}
The adaptive early stopping inference setting length is also a 300-frame rollout. The sliding window is still 50 frames, but there are 40 context frames and 10 in the active horizon. The maximum possible number of sampling steps is 100, and the emission stride is 10 frames. We report FVD scores over 256 generated videos per seed with 1 seed per run. Due to the different inference setting, the metric scores are not easily comparable with the other inference settings.

\paragraph{Online Replanning Inference Setting.}
The online replanning inference setting uses a 300 frame rollout with a sliding window length of 50, with 25 context frames and 25 active sampling horizon frames. The emission stride is 5 frames at a time and we vary the threshold for replanning $\tau$, the range of which was calibrated through hyperparameter search. We report the FVD between the evaluation tensor and true videos.  We use a maximum of 50 inference steps per inner denoising loop, though the algorithm may take fewer steps if it decides to not renoise or to renoise existing predictions to an intermediary noise level.

\subsection{RealEstate10K}

\paragraph{Dataset Details.} 
\label{apd:re10k_dataset_details}
The Re10K dataset consists of curated real estate property tour video clips introduced by \cite{zhou2018stereo}. The camera moves through 3D space, exposed to the model through camera pose annotations. We employ the dataset as a testbed for finetuning a large open source Flow Matching model Wan 2.1 T2V 1.3B \citep{wan2025wanopenadvancedlargescale} into an \EqF model, dropping the noise conditioning as we finetune. Details on equilibrium finetuning are available in Appendix~\ref{apd:eqfft}. This dataset serves as an ideal finetuning benchmark for Wan because the benchmark demands more realism and heterogeneity than Minecraft over an autoregressive rollout. The dataset does not contain captions, so we use a null caption for every clip as input to the T2V model. The dataset contains 58k videos after filtering out videos shorter than 49 frames; the dataset is downsampled to $256\times256$ resolution. We train on clips by choosing a starting frame in the video, resulting in 5.9 million clips for training.

We directly use the VAE from the Wan 2.1 model, which has temporal padding to handle the first frame, meaning the conversion turns $N$ video frames into $\lfloor (N-1) / 4 \rfloor + 1$ latent frames.

\paragraph{Standard Inference Setting.} The standard inference setting is a 189-frame rollout, using 37 frames as the context to generate 152 frames. Due to the temporal downsampling of the VAE, we clarify the number of latent frames here in addition to video frames. The sliding window is 49 video frames (13 latent frames), and there are 37 video frames (10 latent frames) in the context and 12 video frames (3 latent frames) in the active window at any given time. We report FVD and VBench scores over 256 generated videos per seed with 5 seeds for runs in our primary experiments. We follow the default setting of Wan 2.1 and use 50 sampling steps per generation round \citep{wan2025wanopenadvancedlargescale}. The emission stride is 12 video frames (3 latent frames) at a time.

\paragraph{Budget-Adaptive Inference Setting with Less Context.} The adaptive inference setting is a 189-frame rollout, using 29 context frames (8 latent frames) and 20 frames (5 latent frames) in the active horizon. The sliding window is 49 video frames (13 latent frames) and the emission stride is 4 video frames (1 latent frame) at a time. The number of sampling steps varies from 10 to 50 with increments of 10.

\subsection{Droid Robotics Dataset}
\paragraph{Dataset Details.}
The Droid Robotics dataset is a large scale video dataset of robots executing language-specified manipulation of objects in diverse settings \citep{khazatsky2024droid}. See the original paper for more details on how the video dataset was collected. We use the text prompts from the Large Video Planner paper, which are preprocessed with Gemini Flash to describe the motion of the robot in the video with more detail than the original prompts \citep{chen2025large}. This dataset provides us with a realistic benchmark where spatial consistency and object interactions must be predicted with accurate robot dynamics, in contrast with the other two datasets where the environment is static and the camera is dynamic. The resolution is resampled to $384 \times 640$, and we resample video to be 49 frames long so that the text prompt corresponds to a completed action. We use the set of videos that were already filtered by \cite{chen2025large}, resulting in 140k training clips of length 49 each.

We finetune the large scale open source Flow Matching model Wan 2.2 TI2V 5B into an \EqF model on the Droid dataset, following the same finetuning strategy as for the RealEstate10K dataset (details on equilibrium finetuning are available in Appendix~\ref{apd:eqfft}). We directly utilize the UMT5 encoder for text prompt encoding and the VAE from Wan 2.2. The TI2V model has a native image conditioning pathway, but we condition on context using the unified self-attention window. To train without the image conditioning pathway, we turn off the image cross attention layers during training and inference. We also note that the 5B model is not a mixture-of-experts for different noise levels, unlike the Wan 2.2 14B model. Following \cite{chen2025large}, instead of random independent noise levels per frame, we train with random length clean context, to match the inference setting better. A random context length from 0 to 6 temporal latent tokens is sampled. Two noise levels are sampled, one for the context length and one for the rest of the tokens. With probability 0.5, the context is set to noise level 0 instead of the sampled noise level; if it is set to 0, then there is no loss applied on the context tokens.

\paragraph{Adaptive Inference Setting.}
The inference setting for Droid uses a single generation round with 49 frames in the window. The context is 13 video frames (4 latent frames), and the active window is 36 video frames (9 latent frames). The number of sampling steps varies from 10 to 50 with increments of 10. The emission stride is 36 video frames, since there is just one round. For visualization we use 21 video frames as context (6 latent frames) and an active window of 28 video frames (7 latent frames).

\subsection{ImageNet}
\label{apd:imagenet_details} 
We train on the ImageNet-1K dataset for 80 epochs with the XL$/2$ size with center crop and random horizontal flip \citep{5206848}. The model is trained with the SD VAE \citep{rombach2022highresolutionimagesynthesislatent}, with $v$-prediction, 10\% class label dropout, rectified flow path, total batch size 256, AdamW with learning rate 1e-4 and betas (0.9, 0.999), weight decay of 0, EMA weight of 0.9999, and with tf32 precision. The model is trained to directly predict the unmodulated velocity. The EqM and FM models are trained with identical settings. 

We use an identical inference and evaluation setup to the Equilibrium Matching paper \citep{wang2025equilibrium}. Different from EqM, we apply \EqF's inference-time $\eta$ warp. We report FID on 50k generated images with 250 sampling steps and no classifier free guidance.

\subsection{Denoising Model and Noise Level Predictor Compute Comparison}
\label{D.6:compute_analysis}

In this section, we provide wall clock and GFLOP compute comparisons across models on the Minecraft dataset. Table~\ref{tab:forward_compute} shows that FM and \EqF backbones have identical GFLOP requirements, since \EqF is implemented as the same model as FM with the noise level condition always set to zero. The readout $h_\omega$ requires only $0.30$ GFLOPs, less than $1/10{,}000$ of a denoiser forward pass, while $g_\phi$ requires approximately $0.89\%$ of the backbone compute.

We additionally measure wall clock sampling time on Minecraft in Table~\ref{tab:sampling_runtime}. The parameter count ratio between $h_\omega$ and the denoiser is largest on Minecraft compared with Re10K and Droid, and thus provides the most conservative setting. Inference procedures like NAG do not require additional model forward passes, so \EqF GD and \EqF NAG are identical. The noise level prediction models add only negligible additional runtime requirements. 

\begin{table*}[t]
    \centering
    \begin{minipage}[t]{0.50\textwidth}
        \centering
        \caption{Backbone and noise level predictor computational cost per forward pass on Minecraft.}
        \label{tab:forward_compute}
        \begin{tabular}{@{}lr@{}}
            \toprule
            \textbf{Model} & \textbf{GFLOPs/forward} \\
            \midrule
            FM backbone  & 6895.50 \\
            \EqF backbone & 6895.50 \\
            Noise level predictor $g_\phi$     &   61.39 \\
            Noise level readout $h_\omega$   &    0.30 \\
            \bottomrule
        \end{tabular}
    \end{minipage}
    \hfill
    \begin{minipage}[t]{0.46\textwidth}
        \centering
        \caption{Wall clock time per sampling step on Minecraft, 1xH200 with 50 frames.}
        \label{tab:sampling_runtime}
        \begin{tabular}{@{}lr@{}}
            \toprule
            \textbf{Method} & \textbf{Runtime (ms/step)} \\
            \midrule
            FM Euler              & 34.41 \\
            \EqF GD                & 34.46 \\
            \EqF NAG               & 34.46 \\
            \EqF NAG $+\,h_\omega$ & 34.86 \\
            FM Euler $+\,g_\phi$  & 35.10 \\
            \bottomrule
        \end{tabular}
    \end{minipage}
\end{table*}

\section{Details on Equilibrium Finetuning}
\label{apd:eqfft}

In this section, we discuss details for finetuning pretrained Flow Matching models into models such as \EqF that do not condition on the noise level. To the best of our knowledge, equilibrium models have only been trained from scratch in prior work. This makes equilibrium finetuning a novel approach, also applicable to other frameworks like EqM, and thus it is relevant to further expose implementation details. We report equilibrium finetuning across multiple model scales and on each of the datasets. Details of training noise-level predictors on top of finetuned models are available in Appendix~\ref{apd:noise_from_data}.

\subsection{Minecraft Equilibrium Finetuning}
\label{apd:dropout_eq_mc}
We use the Minecraft dataset as a proof of concept for testing the ability to finetune a noise-conditioned model into an equilibrium model. As the initialization for this task, we use the Flow Matching checkpoint trained for 580k steps that served as the baseline for Minecraft experiments. Our results suggest that the \EqF objective alone is sufficient for finetuning, not requiring engineering tricks, curricula, or EqM-style decomposing the target based on the noise level to incentivize the equilibrium landscape.

\paragraph{Equilibrium Finetuning Proof of Concept.}
We perform equilibrium finetuning by continuing training on a Minecraft Flow Matching checkpoint for 100k more steps. We induce an equilibrium landscape by simply zeroing out the noise conditioning signal $\seq{\sigma} = \seq{0}$ and using the \EqF objective from Equation~\ref{eq:eqf_loss}. We further experiment with equilibrium finetuning by using the EqM objective for 100k steps, decomposing the target as $c(\seq{\sigma})\seq{v}$ using the $c$-function used by the EqM paper and explained in Appendix~\ref{apd:sampler_warp_schedules}. Note the simplicity of the \EqF objective in the Minecraft task: the finetuned model only needs to learn the same denoising function without noise level conditioning, while the EqM objective induces a new loss weighting and requires the model to modulate its output magnitude based on the noise level. For fair comparison, we continue training the Flow Matching model for an additional 100k steps as well. After continued training ends, for \EqF we repeat the process of training a readout predictor $h_\omega$ on top of the finetuned models' activations in order to enable closed loop sampling.

\paragraph{Results on Minecraft.}
We report inference results on the identical experimental setting from Section~\ref{5:experiments} on Minecraft in Table~\ref{tab:mc_eq_finetune}. We find that both NAG (reported with the best setting of $\mu$) and closed loop sampling improve results, consistent with our results training from scratch.

\paragraph{Other Finetuning Strategies.}
While we found that simply dropping out the noise level as a finetuning method to be the most effective, we experimented with other strategies as well. We document these other strategies for the sake of the community.

\begin{enumerate}
    \item The simple dropout method explained above has noise level dropout probability $p=1$. We experiment with dropout probabilities $p<1$, where $p$ decides whether we condition on the noise level as in Flow Matching, or pass in zero. We found that in this case, the model took longer to converge and was further unable to fully ``forget'' the noise conditioning signal. Inference carried out by conditioning the model on $\seq{\sigma}=\seq{0}$, (as is standard in \EqF and EqM implementations), produced divergent samples because the model failed to learn an equilibrium field. When running inference by passing in scheduled noise levels (as in standard Flow Matching inference), sample quality was comparatively better, which is evidence of the failure to forget the conditioning signal. Since standard Flow Matching models rely deeply on the noise conditioning, these results suggest that the loss landscape basin of Flow Matching models is shaped by their noise conditioning, relying on substantial changes at finetuning time to escape the local minimum of the new \EqF objective. We further experimented with curricula for the dropout probability, noticing similar results where the model is unable to forget the nonequilibrium landscape of its original training objective.
    \item We also experimented with annealing an auxiliary distribution over a noise conditioning signal $\seq{\sigma}'$ such that at the first training iteration $\seq{\sigma}'=\seq{\sigma}$ with probability 1 and at the end of training they are fully independent. At the end of training, since the condition and the true noise level are independent, the model may learn to ignore the noise condition. We implement this by resampling the condition signal $\seq{\sigma}'$ from a uniform distribution centered on the true noise level $\seq{\sigma}$ such that the uniform anneals into the standard on the entire interval. We experimented with similar annealing schedules for the Beta distribution to encourage resampling concentration on $\seq{\sigma}=0$. While more effective than partial dropout, the results were inferior to simple full dropout training, not justifying the added complexity of the curriculum.

\end{enumerate}

\subsection{Re10K Equilibrium Finetuning}
\label{apd:dropout_eq_re10k}

In this section, we discuss details of finetuning Wan 2.1 T2V 1.3B into a noise unconditional model using the \EqF objective. The dataset is camera-conditioned and we adopt pose processing from \cite{song2025historyguidedvideodiffusion}, introducing the conditioning signal into the AdaLN stream \citep{peebles2023scalable}.

With insights from Minecraft finetuning in Appendix~\ref{apd:dropout_eq_mc} we proceed to employ simple $p=1$ dropout of the noise conditioning signal to finetune Wan into \EqF on the Re10K dataset; we refer to this as the \EqF-FT model. We also fine tune Flow Matching as a noise-conditional baseline with the same computational budget, denoted FM-FT. 

We elucidate training dynamics by plotting the training and validation loss (validating an EMA model with an EMA weight of $0.9999$) against training iteration in Figure~\ref{fig:re10k_train}. The loss decreases rapidly for the baseline Flow Matching model (FM-FT), while generative quality with standard inference-time noise conditioning lags behind (partially due to the EMA). On the other hand, the \EqF-FT model must forget the noise conditioning signal at the beginning of training, leading to the model incurring a higher training and validation loss at the start of training. The generation quality of \EqF-FT after only 10k steps is substantially worse because we run inference on the model with zero noise conditioning, and the model has not yet learned the equilibrium landscape. However, by 20k steps, the model has learned to effectively drop its reliance on noise conditioning and catches up to the Flow Matching baseline, after which they have similar training dynamics. Training is carried out for 100k steps, over which validation FVD continues to improve. Note that the inference results in Figure~\ref{fig:re10k_train}, run for the purpose of validation during training, are strictly open loop and not adaptive.

\begin{figure}[t]
    \centering
    \includegraphics[width=\linewidth]{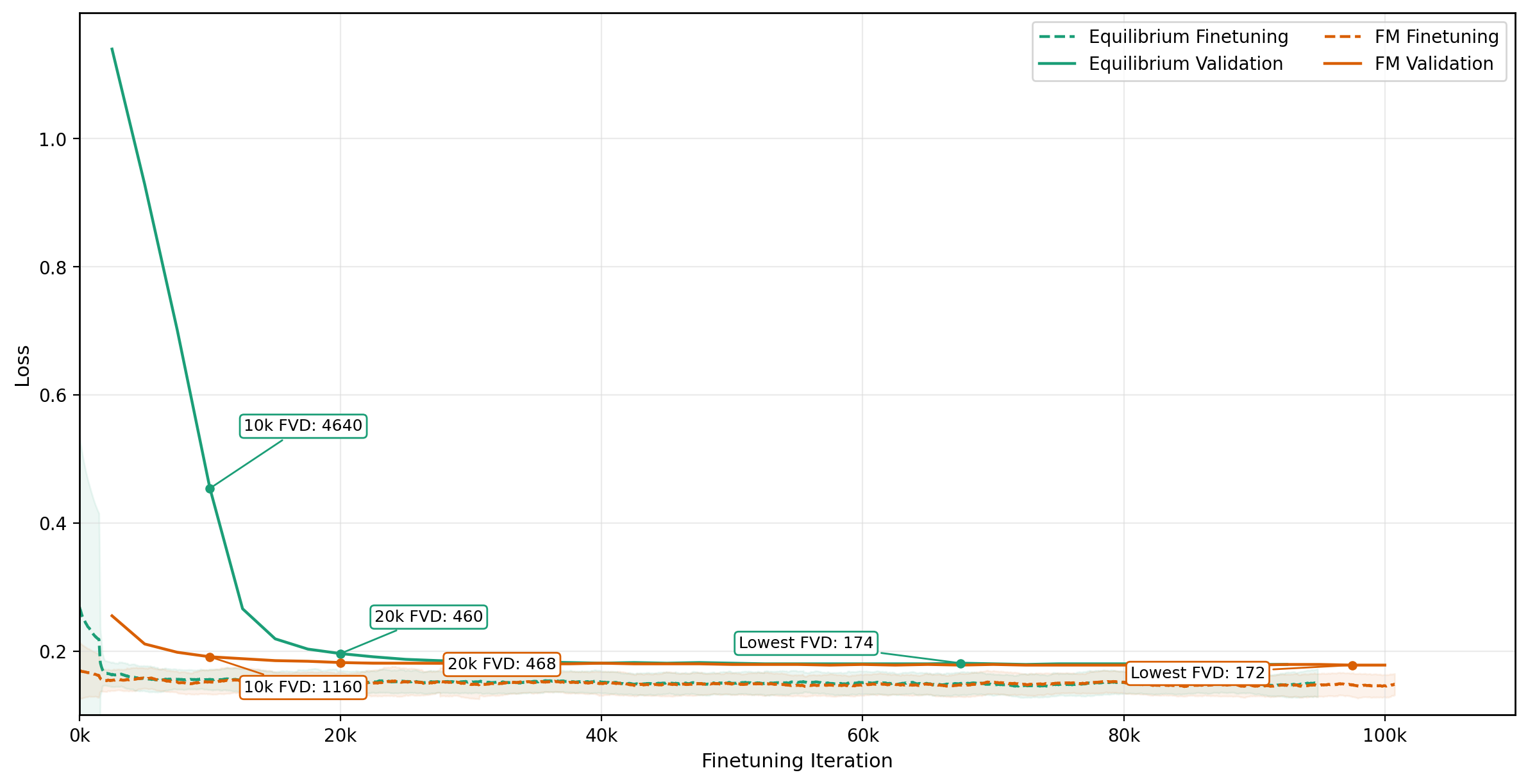}
    \caption{Re10K Validation Loss and FVD across equilibrium finetuning steps.}
    \label{fig:re10k_train}
\end{figure}

\subsection{Droid Equilibrium Finetuning}
\label{apd:dropout_eq_droid}
To finetune Wan 2.2 TI2V 5B into a noise unconditional model using the \EqF objective, we follow the same procedure as with Minecraft and Re10K. Namely, we employ simple $p = 1$ dropout of the noise conditioning signal to create the \EqF-FT model, and finetune the Flow Matching model with the same settings to yield the FM-FT model. Models are trained for 60k steps; it has been observed that denoising quality continues to improve with continued training for denoising models even after the model loss converges, and it is likely that further continuing training would yield qualitative and quantitative improvements in this setting as well.

\section{Details on Estimating Noise Levels from Data or \EqF Activations}
\label{apd:noise_from_data}

In this section, we provide neural network architecture and training details for models that predict noise levels from the data or from \EqF activations on the Minecraft dataset. We refer to auxiliary networks $g_\phi$ as \textit{noise level predictors} when they estimate the noise level directly from a noisy frame, and we refer to \textit{readout predictors} $h_\omega$ when they bootstrap off an \EqF model's activations when evaluated on a noisy frame. Noise level predictors $g_\phi$ are used in our noise level divergence analyses as well as in an effort to close the loop on noise level divergence for standard Flow Matching inference. On the other hand, readout predictors $h_\omega$ are primarily used by \EqF for closed-loop inference via the warp $\eta$.

\subsection{Architecture Design}

\paragraph{Prediction of Noise Level Directly from Noisy Data.}
The noise level predictor $g_\phi$ on the Minecraft dataset is a framewise CNN with 600k parameters trained to predict the noise level of a frame that has been noised under the Flow Matching noising process. We employ framewise training because autoregressive inference algorithms typically use a rolling denoising window in which contiguous chunks of frames appear at different noise levels. To train, we draw a single frame
$x_t \in \mathbb{R}^{C \times H \times W}$ from the data, sample
$\epsilon \sim \mathcal{N}(0,I)$ and $\sigma \sim \mathcal{U}[0,1]$, and construct the noised frame
\begin{align}
    x_t^\sigma
    &= (1 - \sigma) x_t + \sigma \epsilon .
    \label{eq:single_frame_noise}
\end{align}
The network predicts the logit of the noise level, with objective
\begin{align}
    \mathcal{L}_{g}
    &=
    \mathbb{E}_{x_t, \epsilon, \sigma}
    \left[
    \left\|
    g_\phi(x_t^\sigma) - \operatorname{logit}(\sigma)
    \right\|_2^2
    \right].
\end{align}
We apply the loss in logit space because it provides a more balanced supervision signal near $\sigma \approx 0$ and $\sigma \approx 1$, compared with directly applying squared error on $[0,1]$. We convert the logit prediction to a valid noise level by applying a
sigmoid,
\begin{align}
    \hat{\sigma}_\phi(x_t^\sigma)
    &=
    \operatorname{sigmoid}\!\left(g_\phi(x_t^\sigma)\right).
\end{align}
In Section~\ref{3:flow_analysis}, we slightly abuse notation by writing the
estimated true noise level as $\tilde{\sigma} \approx g_\phi(x_t^\sigma)$, where the
sigmoid transformation is implicit. We evaluate the model compared to the readout predictor $h_\omega$ below. 

The model is able to train accurately on the noising process because the noise level is identifiable from the noisy data, as Gaussian shells theoretically concentrate in high dimensions \citep{kadkhodaie2026blind}. In other words, this concentration removes the ambiguity in the data distribution $p(\sigma | x^\sigma)$ that the model was trained on; though that distribution for different values of $\sigma$ technically has support everywhere, the model is trained to predict the mode which matches the distribution in practice. Further, though intermediate samples at a particular noise level during inference do not necessarily match the marginal distribution from the forward noising process on which $g_\phi$ is trained, they overlap sufficiently when sampling with a well-trained flow matching backbone such that the noise level predictor remains accurate during the inference process. Evidence of this is established by the result that closed loop sampling using the $g_\phi$-estimated noise level leads to performance benefits for Flow Matching inference in Section~\ref{3:flow_analysis} and Appendix~\ref{apd:fixed_fm_with_g}. Finally, since $g_\phi$'s training prediction error is low (as explained in the next subsection), we may assume that the error reported in Section~\ref{3:flow_analysis} represents a consequential difference in the noise level condition, rather than the error of $g_\phi$. Training dynamics of the noise level prediction model are elaborated upon in Appendix~\ref{train_dynamics_nlp}.

\paragraph{Prediction of the Noise Level from \EqF Activations.}

The readout predictor $h_\omega$ is a framewise CNN with 260k parameters. Like the noise level predictor, it is trained in a framewise manner to predict the noise levels under a Flow Matching noising process. However, the readout predictor relies on a pretrained \EqF model's ability to implicitly estimate $p(\seq{\sigma} \vert \seq{x}^{\seq{\sigma}})$ to minimize the training objective (Section~\ref{4.3:eta_schedule_explanation}). Given a pretrained $f_\EqF$ denoiser, we train the readout predictor by extracting $f_\EqF$'s activations $z_t^\sigma$ on noisy data constructed as in Equation~\ref{eq:single_frame_noise} and bootstrapping from model activations as follows: 
\begin{align}
    \mathcal{L}_{h}
    &=
    \mathbb{E}_{x_t, \epsilon, \sigma}
    \left[
    \left\|
    h_\omega(z_t^\sigma)
    -
    \operatorname{logit}(\sigma)
    \right\|_2^2
    \right],
    \qquad
    \text{where } (\hat{v}_t^\sigma, z_t^\sigma)
    =
    f_{\EqF}(x_t^\sigma).
    \label{eq:readout_noise_level_loss}
\end{align}

The readout prediction is also converted to a valid noise level by applying the sigmoid.

\subsection{Noise Level Predictor and Readout Training Dynamics on Minecraft}

We compare the training dynamics of noise level predictors with readout predictors. The results support the hypothesis that \EqF models must internally estimate the noise level of a sample to minimize the objective, as we find it far easier to train an estimator of the noise level of data by bootstrapping a readout predictor $h_\omega$ from \EqF activations compared to training a noise level predictor $g_\phi$ from the raw data directly.

\label{train_dynamics_nlp}

On the Minecraft dataset, readout predictor $h_\omega$ converges in only 10k steps, whereas the noise level predictor $g_\phi$ is trained for 100k steps due to significantly slower convergence and doesn't reach as low of a loss. The training loss is displayed in Figure~\ref{fig:nle}. We further found that $g_\phi$ requires larger model capacity to adequately solve the prediction task, using 600k parameters, while $h_\omega$ is less than half the size. These results, especially the fast convergence of the noise level estimator, support our hypothesis that the activations of the \EqF model must already contain a latent estimate of the noise level in order to minimize the objective (Section~\ref{4.2:noise_conc}).

\begin{figure}[tb]
    \centering
    \includegraphics[width=0.99\linewidth]{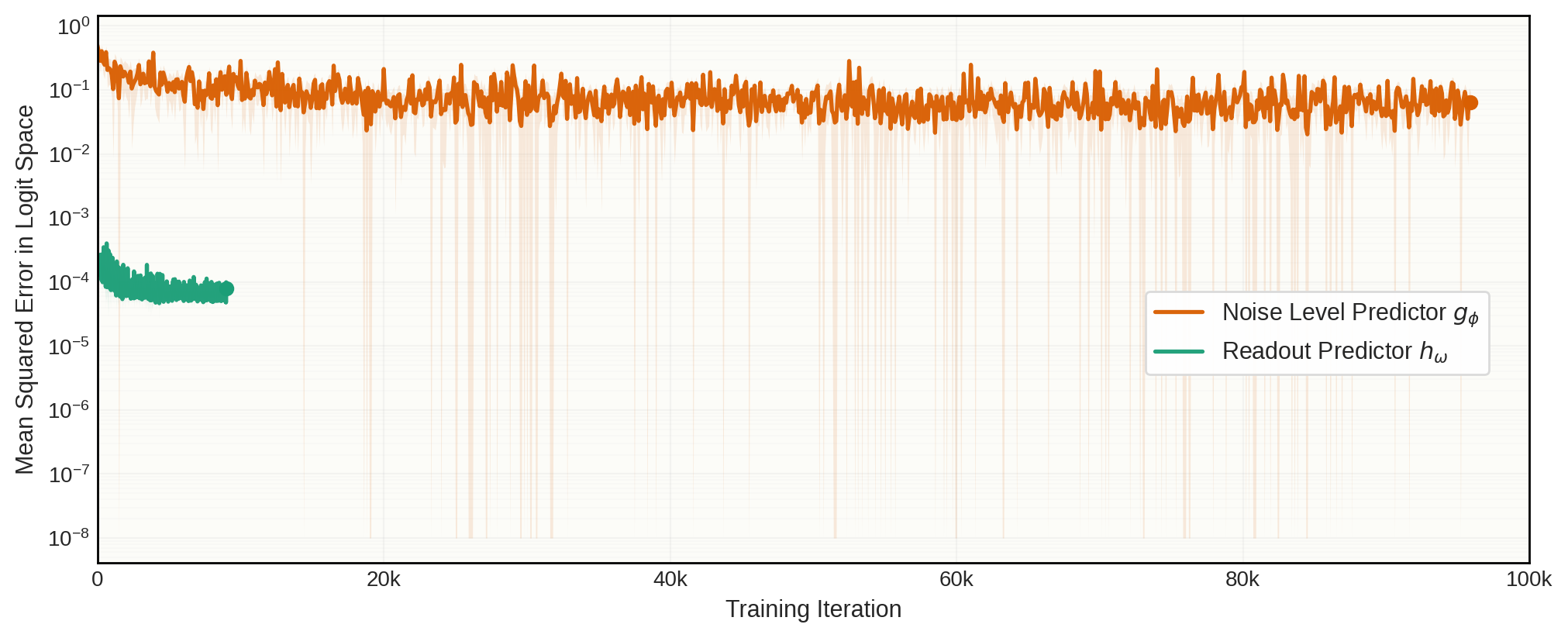}
    \caption{Training loss for noise level estimation. The readout predictor $h_\omega$, trained on \EqF activations, converges substantially faster than the framewise predictor $g_\phi$ trained directly from noisy data.}
    \label{fig:nle}
\end{figure}

The order of magnitude difference in loss is primarily attributable to accuracy near the low noise boundary $\sigma \approx 0$. Since the predictors are trained in logit space, small absolute errors in $\sigma$ near the boundary are amplified by the logit transform. The corresponding error in noise-level space still means they are quite accurate. This effect is visible in the marginalized residuals in Figure~\ref{fig:margred}: both predictors are accurate enough to serve their intended roles in our analysis, with $g_\phi$ providing a reliable external estimate of the noise level directly from noisy data and $h_\omega$ providing a reliable estimate from \EqF activations. The remaining difference is concentrated near $\sigma \approx 0$, where the readout predictor has noticeably smaller residuals and where logit-space errors are most heavily magnified. This improved accuracy at the boundary is especially relevant for our use of $h_\omega$ in closed-loop sampling and adaptive stopping, since those mechanisms depend on accurately detecting proximity to the clean-data manifold. 

\begin{figure}[t]
    \centering
    \begin{minipage}{0.49\linewidth}
        \centering
        \includegraphics[
            width=\linewidth,
            trim=0 0 0 2cm,
            clip
        ]{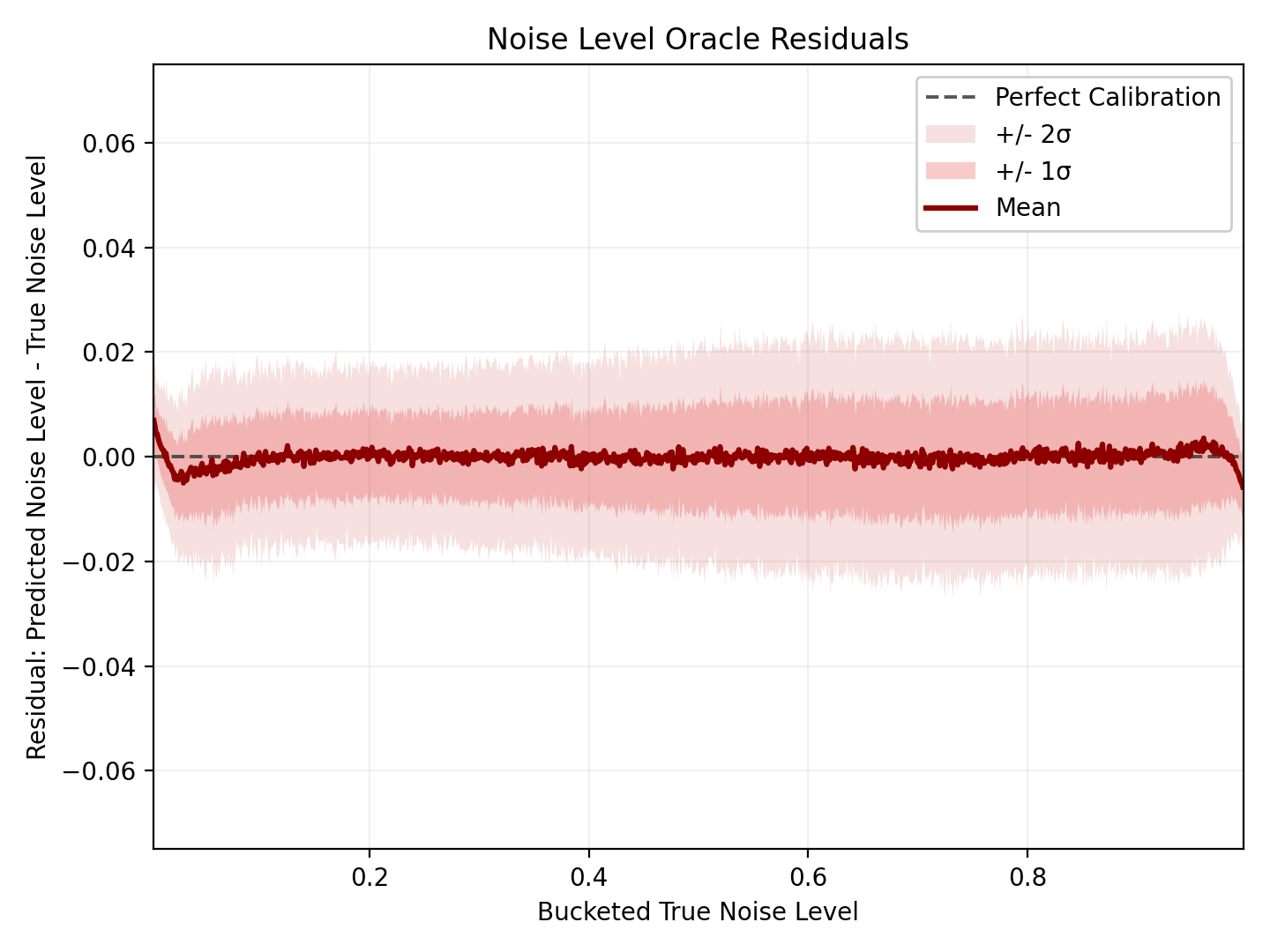}
    \end{minipage}
    \hfill
    \begin{minipage}{0.49\linewidth}
        \centering
        \includegraphics[
            width=\linewidth,
            trim=0 0 0 2cm,
            clip
        ]{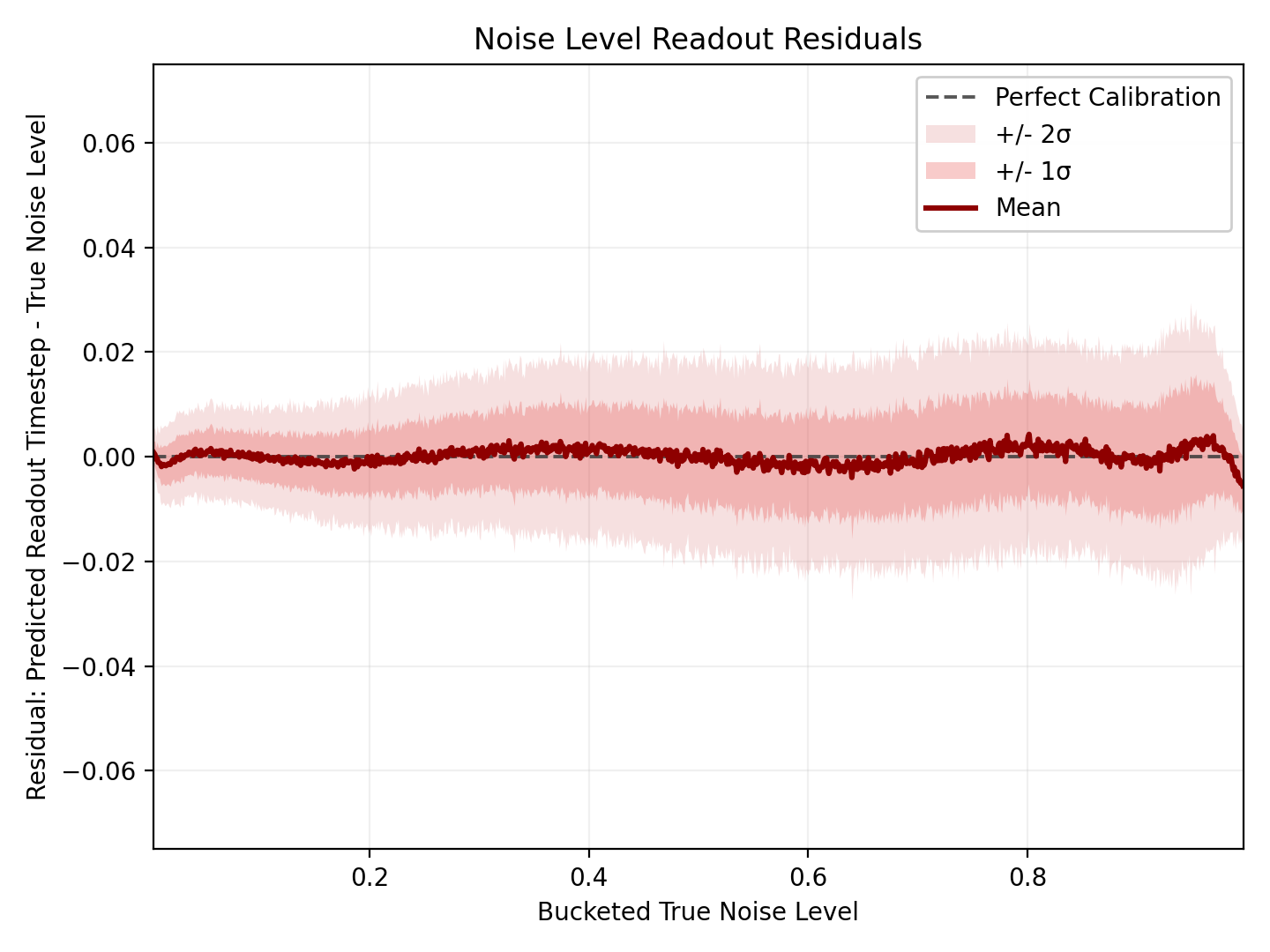}
    \end{minipage}
    \caption{Marginalized residuals of noise level estimates from the framewise noise level predictor $g_\phi$ trained directly on noisy data (left) and the readout predictor $h_\omega$ trained on \EqF activations (right). Both predictors are accurate, but $h_\omega$ is especially accurate near the low-noise boundary.}
    \label{fig:margred}
\end{figure}

\subsection{Training Details of Noise Level Predictor on Re10K and Droid Datasets}
Here, we elaborate on training details for the noise level predictor $h_\omega$ for equilibrium finetuned models. On the 1.3B Re10K model, we train a 460k parameter readout predictor for an additional 5k steps. On the 5B Droid model, we train a 860k parameter readout predictor for an additional 10k steps. Both architectures are the same: the denoising backbone is frozen, we mean pool over the spatial dimensions of the model activations, apply a shared LayerNorm, apply a linear projection down to 256 dimensions, combine the layers using learned softmax mixing weights, and then apply a framewise MLP. The target is the logit of the noise level for each frame, trained with an MSE loss. Similar to the training dynamics on the Minecraft dataset, we find that the model converges quickly on the noise level readout task.

\section{Details on Sampler Warp Schedules}
\label{apd:sampler_warp_schedules}

In this section, we describe the inference-time warp functions used to shape the sampling dynamics of \EqF. The warp in \EqF is used to modulate the magnitude of the learned equilibrium velocity field during sampling without changing the training objective, providing a unified way to compare transport-style ODE reparameterization functions common in the literature with functions that induce attractor-style dynamics that vanish near the data manifold \citep{esser2024scaling}. We first define the relationship between $\rho$, a function that reparameterizes solver time $s$ to noise levels $\sigma$, and a warp $\eta$ for mapping from noise levels to step sizes. The different function definitions are necessary to unify understanding between what is proposed here with other ODE reparameterizations in the literature, as well as to understand normalization factors necessary for making the inference procedure principled. We then define the tradeoffs between warps that are transports and those that induce attractor landscapes. Finally, we give the exact warp definitions used in our experiments. 

\subsection{Sampling Schedules and Warp Functions}

We first define $s \in [0,1]$ as solver time, where $s=0$ corresponds to the beginning of sampling at pure noise, and $s=1$ to the end of sampling at clean data. Note this is the opposite direction of how the noise level $\sigma$ is defined: $\sigma=1$ is full noise, and $\sigma=0$ is clean data. A time-reparameterized noise schedule is a monotonically decreasing function
\begin{align}
\rho : [0,1] \to [0,1],
\qquad
\rho(0)=1,
\qquad
\rho(1)=0,
\end{align}
where $\rho(s)$ denotes the scheduled noise level $\sigma$ at solver time $s$. Evenly spaced solver steps in $s$ can therefore correspond to nonuniform steps in noise level $\rho(s)$, as is common in ODE samplers for Flow Matching \citep{esser2024scaling}.

The corresponding warp function $\eta$ specifies the local update multiplier applied to the denoising field at each noise level. Along a trajectory $\rho(s)$, we define $\eta$ as the positive denoising speed through noise-level space:
\begin{align}
\eta : [0,1] \to \mathbb{R},
\qquad
\eta(\sigma) = \eta(\rho(s))
:=
-\frac{d\rho(s)}{ds}.
\end{align}

For $\rho$ to define a valid transport, it must bring the probability mass from the initial noise level to the final noise level, satisfying 
\begin{align}
\int_0^1 \eta(\rho(s))\,ds
&=
\int_0^1 \left(-\frac{d\rho(s)}{ds}\right)\,ds
\\
&=
\rho(0)-\rho(1)
=
1.
\end{align}

If we instead define $\eta$ directly as a function of $\sigma$, we can recover the induced trajectory over solver time $\rho$ by solving the ODE
\begin{align}
\frac{d\rho(s)}{ds}
=
-\eta(\rho(s)),
\qquad
\rho(0)=1.
\label{eq:eta_induced_rho_trajectory}
\end{align}

Separating variables in Equation~\ref{eq:eta_induced_rho_trajectory} gives
$ds = -d\rho/\eta(\rho)$. Therefore, the solver time required to travel from noise level $\rho=1$ to a terminal noise level $\sigma_{\mathrm{end}}$ is
\begin{align}
T_\eta(\sigma_{\mathrm{end}})
=
\int_{\sigma_{\mathrm{end}}}^{1}
\frac{du}{\eta(u)}.
\end{align}
When specifying a warp directly, we separate its shape from its overall scale by writing
\begin{align}
\eta(\sigma)
=
A\tilde{\eta}(\sigma),
\end{align}
where $\tilde{\eta}$ is an unnormalized warp shape and $A>0$ is a scalar normalization factor. This normalization factor is important to make sure that when we define the shape of a warp $\tilde{\eta}$, that it also defines a valid sampling process. Imposing the unit-time traversal condition $T_\eta(\sigma_{\mathrm{end}})=1$ gives
\begin{align}
1
=
\int_{\sigma_{\mathrm{end}}}^{1}
\frac{du}{A\tilde{\eta}(u)}
=
\frac{1}{A}
\int_{\sigma_{\mathrm{end}}}^{1}
\frac{du}{\tilde{\eta}(u)},
\qquad
\Longrightarrow
\qquad
A
=
\int_{\sigma_{\mathrm{end}}}^{1}
\frac{du}{\tilde{\eta}(u)}.
\end{align}
Thus, $\tilde{\eta}$ determines the relative allocation of computation across noise levels, while $A$ fixes the total traversal time, together defining $\eta$. In the main text, namely Equation~\ref{eq:sampler_update}, we slightly abuse notation and also include a factor inversely proportional to the number of sampling steps $N$ in the $\eta$ function for simplicity of exposition. In this section, it is expressed as the normalized warp without the scaling by the number of steps.

\subsection{Transport and Attractor Warps}
\label{apd:transport_and_attractor_warps}

The behavior of $\eta$ near the data boundary determines whether the induced dynamics are transport-style or attractor-style. Transport-style warps remain nonzero as $\sigma \to 0$, so the sampler reaches the data manifold by following a prescribed trajectory whose endpoint is $\sigma=0$. In this case, the dynamics need not slow down near clean data; successful sampling depends on the solver intersecting the data manifold at the scheduled endpoint.

Attractor-style warps instead satisfy $\eta(\sigma) \to 0$ as $\sigma \to 0$. When using an attractor-style warp with \EqF sampling, the warped field $\eta(\sigma) f_{\EqF}(x)$ vanishes near clean data, turning sampling into a fixed-point search. This attractor interpretation further motivates the use of gradient-based optimization tools such as Nesterov Accelerated Gradient, which are designed for convergent dynamics. We empirically demonstrate NAG works better with attractor warps in Table~\ref{tab:minecraft_inference_expts}. Since attractor warps vanish at the boundary, reaching exactly $\sigma=0$ requires infinite continuous time; in practice, we normalize the trajectory down to a finite terminal noise level $\sigma_{\mathrm{end}}>0$.

Figure~\ref{fig:sampler_warp_conversions} illustrates the three equivalent views of a warp schedule. Subfigure~\ref{fig:sampler_warp_conversions}(a) shows the warp $\eta(\sigma)$ as a noise-dependent update multiplier. Subfigure~\ref{fig:sampler_warp_conversions}(b) shows the induced noise trajectory $\rho(s)$ over solver time. Subfigure~\ref{fig:sampler_warp_conversions}(c) shows the composed update multiplier $\eta(\rho(s))$ actually applied over solver time. In open-loop sampling this composition determines the predetermined step-size sequence, whereas in closed-loop \EqF sampling we evaluate $\eta(\hat{\sigma})$ using the estimated current noise level.

\subsection{Detailed Instantiations of Warp Functions}
\label{apd:detailed_instantiations_of_warp_functions}

In this section we explain the instantiations of the warp functions used in our experiments. The identity warp and the EqM $c$-function are specified directly as noise-dependent warp functions $\eta(\sigma)$. In contrast, SD3 and linear log-SNR are naturally specified as solver-time trajectories $\rho(s)$ and converted into warp functions using $\eta(\rho(s))=-d\rho(s)/ds$.

\paragraph{Transport Warp: Identity.}
The simplest transport warp is the identity warp,
\begin{align}
\eta_{\mathrm{id}}(\sigma)
=
1.
\end{align}
This assigns constant denoising speed at every noise level. Using the induced trajectory conversion in Equation~\ref{eq:eta_induced_rho_trajectory}, we obtain
\begin{align}
\frac{d\rho(s)}{ds}
=
-1,
\qquad
\rho_{\mathrm{id}}(s)
=
1-s.
\end{align}
Since $\eta_{\mathrm{id}}(0)=1$, the update multiplier does not vanish near the data boundary, making this a transport-style warp.

\paragraph{Transport Warp: SD3.}
We also consider the time-reparameterization used in SD3 \citep{esser2024scaling}. SD3's schedule denotes a shift parameter $r > 0$. The SD3 trajectory (where $s=0$ is noise and $s=1$ is data) is 
\begin{align}
\rho_{\mathrm{SD3}}(s)
=
\frac{r(1-s)}
{r-(r-1)s}.
\end{align}
This already satisfies the endpoint constraints
\begin{align}
\rho_{\mathrm{SD3}}(0)
=
1,
\qquad
\rho_{\mathrm{SD3}}(1)
=
0,
\end{align}
so it defines a valid unit-time transport schedule without additional normalization. To express the induced warp, we compute the positive denoising speed through noise-level space:
\begin{align}
\eta_{\mathrm{SD3}}(\rho_{\mathrm{SD3}}(s))
:=
-\frac{d\rho_{\mathrm{SD3}}(s)}{ds}
=
\frac{r}
{\left(r-(r-1)s\right)^2}.
\end{align}
This expression is currently written as a function of solver time $s$. To rewrite it as a function of the current noise level $\sigma$, we invert the schedule:
\begin{align}
\sigma
=
\frac{r(1-s)}
{r-(r-1)s}
\qquad
\Longrightarrow
\qquad
s
=
1
-
\frac{\sigma}
{r-(r-1)\sigma}.
\end{align}
Substituting this inverse into the denoising speed gives
\begin{align}
\eta_{\mathrm{SD3}}(\sigma)
=
\frac{\left(r-(r-1)\sigma\right)^2}{r}.
\end{align}
For $r>1$, the boundary values are
\begin{align}
\eta_{\mathrm{SD3}}(1)
=
\frac{1}{r},
\qquad
\eta_{\mathrm{SD3}}(0)
=
r.
\end{align}
Thus, SD3 slows down near pure noise and increases the update multiplier near the data manifold. Since $\eta_{\mathrm{SD3}}(0)\neq0$, it is a transport-style warp. Consistent with this transport boundary behavior, our experiments find that NAG is less stable with SD3 than with attractor-style warps.

\paragraph{Attractor Warp: EqM $c$-Function.}
We next consider the $c$-function shape introduced by Equilibrium Matching \citep{wang2025equilibrium}. EqM uses this function as a training-time target modulation, whereas in the \EqF framework this is an option for the shape of the inference-time warp. EqM defines
\begin{align}
c_{\lambda,\alpha}(\sigma)
=
\lambda
\min
\left\{
1,
\frac{\sigma}{1-\alpha}
\right\}
=
\begin{cases}
\lambda \dfrac{\sigma}{1-\alpha},
& 0 \leq \sigma \leq 1-\alpha,
\\[1ex]
\lambda,
& 1-\alpha < \sigma \leq 1.
\end{cases}
\end{align}

In \EqF, we use the unscaled shape $c_{1,\alpha}$ as an unnormalized warp shape, with $\alpha=0.8$:
\begin{align}
\tilde{\eta}_{c,\alpha}(\sigma)
=
c_{1,\alpha}(\sigma).
\end{align}
Applying the unit-time normalization derived above to ensure that the induced trajectory finishes in 1 unit of solver time gives
\begin{align}
\eta_{c,\alpha}(\sigma)
=
A_\alpha c_{1,\alpha}(\sigma),
\qquad
A_\alpha
=
\int_{\sigma_{\mathrm{end}}}^{1}
\frac{du}{c_{1,\alpha}(u)}.
\end{align}
Since $c_{1,\alpha}(\sigma)\to0$ as $\sigma\to0$, this warp induces attractor-style dynamics and is normalized over a finite terminal noise level $\sigma_{\mathrm{end}}>0$. In EqM, this normalization was not considered, and instead the magnitude scale $\lambda$ was swept over as an empirical training-time hyperparameter. The empirically optimal $\lambda$ and inference-time learning rate scale (called $\eta$ in their paper) combination that EqM ends up with (equivalent to setting $A_\alpha=1.7$) is in fact close to the analytical value of $A_\alpha=1.85$, when the final noise level is taken to be $\sigma_{\mathrm{end}}=0.001$.

\paragraph{Attractor Warp: Linear log-SNR.}
Finally, we consider a warp induced by linear motion in log-SNR space \citep{lu2022dpm}. This schedule is naturally specified by first defining a trajectory in log-SNR and then converting that trajectory back to noise-level space. For the Rectified Flow interpolation, the log signal-to-noise ratio is
\begin{align}
\lambda(\sigma)
=
\log \mathrm{SNR}(\sigma)
=
2\log\frac{1-\sigma}{\sigma}.
\end{align}
Solving for $\sigma$ in terms of a log-SNR value $\ell$ gives
\begin{align}
\lambda^{-1}(\ell)
=
\frac{1}{1+\exp(\ell/2)}.
\end{align}
Moving from noise to data corresponds to increasing log-SNR: as $\sigma\to1$, $\lambda(\sigma)\to-\infty$, while as $\sigma\to0$, $\lambda(\sigma)\to+\infty$. We therefore define a linear trajectory in log-SNR space between finite endpoints $\lambda_{\min}$ and $\lambda_{\max}$:
\begin{align}
\bar{\lambda}(s)
&=
\lambda_{\min}
+
s(\lambda_{\max}-\lambda_{\min}),
\qquad
s\in[0,1],
\\
\rho_{\log\mathrm{SNR}}(s)
&=
\lambda^{-1}(\bar{\lambda}(s))
=
\frac{1}{1+\exp(\bar{\lambda}(s)/2)}.
\end{align}

The corresponding positive denoising speed is
\begin{align}
\eta_{\log\mathrm{SNR}}(\rho_{\log\mathrm{SNR}}(s))
:=
-\frac{d\rho_{\log\mathrm{SNR}}(s)}{ds}
=
\frac{\lambda_{\max}-\lambda_{\min}}{2}
\rho_{\log\mathrm{SNR}}(s)
\left(1-\rho_{\log\mathrm{SNR}}(s)\right).
\end{align}
Equivalently, as a direct function of the current noise level $\sigma$,
\begin{align}
\eta_{\log\mathrm{SNR}}(\sigma)
=
\frac{\lambda_{\max}-\lambda_{\min}}{2}
\sigma(1-\sigma).
\end{align}
This warp vanishes as $\sigma\to0$, so it induces attractor-style dynamics near the data manifold. It also vanishes as $\sigma\to1$, allocating smaller updates near the pure-noise endpoint.

\subsection{Relation Between $\eta$ Warp Function and Budget-Adaptive Inference.}
The budget-adaptive inference algorithm presented in Algorithm~\ref{alg:budget_adaptive} defines an update rule with $\eta$ taken to be the identity warp. Budget-adaptivity calibrates the update scaling factor so that inference can expect to finish at the data manifold based on the most recent noise level estimate. This achieves a similar effect to an attractor warp, helping prevent inference with acceleration techniques like NAG from overshooting. Further exploration of parameterizations that incorporate the number of remaining steps $N - i$ into more general $\eta$ functions is an interesting direction for future work.

\begin{figure}[t]
    \centering
    \includegraphics[width=1\linewidth]{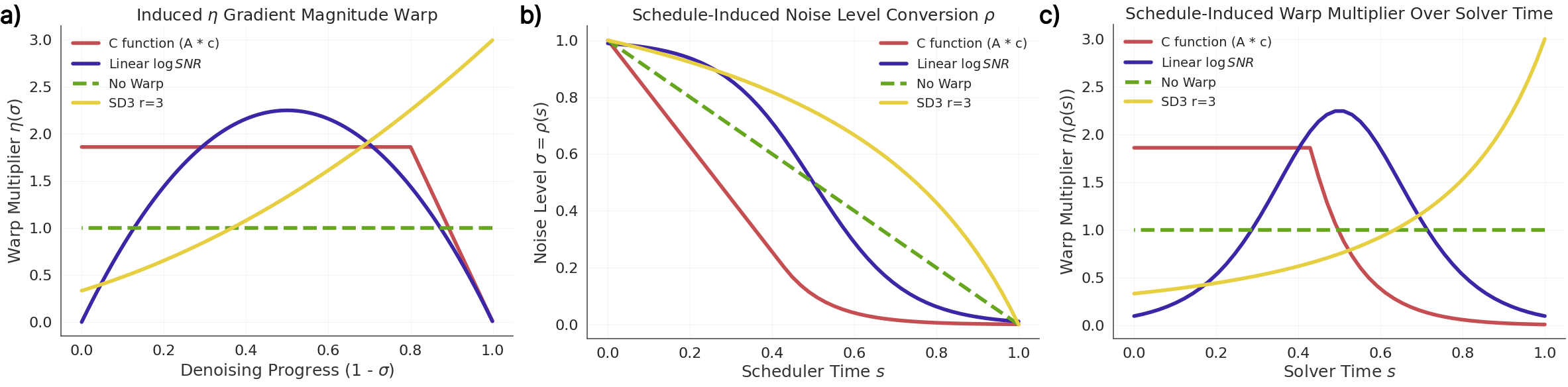}
    \caption{
    \textbf{Conversions between noise level, discretized solver time, and induced warp.}
    \textbf{a)} Conversion between $\sigma$ and update multiplier under different choices of the warp function $\eta$.
    \textbf{b)} Induced conversion between solver time $s$ and noise level $\sigma$ through the trajectory $\rho$. Equally sized steps in solver time can correspond to different changes in noise level.
    \textbf{c)} Warp multiplier as a function of solver time through the composition $\eta(\rho(s))$. Due to the vanishing update multiplier for attractor-style choices of $\eta$, equally sized steps in solver time take increasingly smaller steps as the sampler approaches data.
    }
    \label{fig:sampler_warp_conversions}
\end{figure}

\section{Details on Gradient Based Inference with Denoising Video Generative Models}
\label{apd:inference_details}

Nesterov Accelerated Gradient (NAG) is a method for accelerated function optimization. It is a classical algorithm that leverages momentum and a lookahead step with theoretical guarantees over gradient descent in general smooth convex function optimization \citep{nesterov1983method,pmlr-v28-sutskever13}. The algorithm is for fixed-point iteration, making it a suitable equilibrium landscape optimizer when using an inference warp $\eta$ such that $\eta (\widehat{\seq{\sigma}}) f_{\EqF}(\seq{x}) = \seq{0}$ for $\seq{x}$ on the data manifold. We empirically confirm in Section~\ref{5:experiments} that when the warp is chosen to not lead to such an attractor-style field at the data manifold, NAG produces subpar samples as expected, because the algorithm is prevented from converging. We provide a simple NAG sampling algorithm with \EqF for one sample step in Algorithm~\ref{alg:nag_sampling_step}; note that the signs of gradient terms align with the convention for the training vector $\seq{v}=\seq{\epsilon}-\seq{x}$. Future work may consider accelerated sampling with other fixed point iteration methods. 

\begin{algorithm}[t]
\caption{Nesterov Accelerated Gradient Sampling Step with \EqF}
\label{alg:nag_sampling_step}

\KwIn{
model $f_\EqF$, readout $h_\omega$, sample $\seq{x}^{(i)}$,
momentum $\seq{m}^{(i)}$, warp $\eta(\cdot)$,
momentum coefficient $\mu$
}

$\seq{x}_{\mathrm{look}}^{(i)}
\leftarrow
\seq{x}^{(i)}
-
\mu \seq{m}^{(i)}$\;

$\widehat{\seq{v}}^{(i)}, \seq{z}^{(i)}
\leftarrow
f_\EqF(\seq{x}_{\mathrm{look}}^{(i)})$\;

$\widehat{\seq{\sigma}}^{(i)}
\leftarrow
h_\omega(\seq{z}^{(i)})$\;

$\seq{m}^{(i+1)}
\leftarrow
\mu \seq{m}^{(i)}
+
\eta(\widehat{\seq{\sigma}}^{(i)})
\odot
\widehat{\seq{v}}^{(i)}$\;

$\seq{x}^{(i+1)}
\leftarrow
\seq{x}^{(i)}
-
\seq{m}^{(i+1)}$\;

\Return{$\seq{x}^{(i+1)}, \seq{m}^{(i+1)}$}\;
\end{algorithm}
\section{Autoregressive Inference Procedures for Standard Denoising Video Generative Models}
\label{apd:inference_sched_math}

In this section, we elaborate on the static inference algorithms used by standard Flow Matching and Diffusion video models during inference, such as those used by \cite{chen2024diffusionforcingnexttokenprediction,ruhe2024rolling,song2025historyguidedvideodiffusion,hafner2025training}. \EqF instead relies on its own estimate of progress to the data manifold within a closed loop, leveraging gradient-based solvers to propose data-adaptive steps and possessing unique early stopping capabilities (Section~\ref{4.3:eta_schedule_explanation}, Algorithm~\ref{alg:adaptive_early_stopping}). However, some aspects, such as the bake-in period, are also shared by \EqF. Though the autoregressive inference style we utilize here has fundamental weaknesses in preserving information across repeated inference rounds, we isolate our study to the denoising and inference procedure, and leave the memory issue to future work \citep{lillemark2026flow}.

\paragraph{Framewise Noise Level Training Enables Stable Autoregressive Rollouts.}
Rolling inference algorithms correlate the noise level of a frame with its index in the video, with lower noise levels mapping to earlier frames and higher noise levels to later frames. This capability is enabled by training with different noise levels per frame, instead of having just one noise level for all frames. Diffusion Forcing training assigns independent noise levels per frame \citep{chen2024diffusionforcingnexttokenprediction}, generally preferred by practitioners since it enables a variety of inference schedules to be `within training distribution' \citep{oasis2024, hafner2025training}. Other methods like Rolling Diffusion bake in a specific inference schedule during training, such as one resembling a staircase with increasing noise levels per index \citep{ruhe2024rolling,cachay2025elucidated}.

\paragraph{Inference Schedule Parameter Definitions.}
These inference schedules specify a global geometry that is predetermined prior to inference, and we expand upon this here in order to explain the baseline inference algorithm for easier comparison. It is static, fully defining the number of steps per frame and the noise levels that will be visited. An inference schedule is defined by three fundamental quantities: the length of the prediction horizon at each denoising step $H$; the stride length of a group of frames that will be emitted together, denoted $S$; and the number of denoising steps the model will take for a frame, denoted as the depth $D$. The algorithm, as presented, can be run in an infinite loop so that any amount of frames can be generated.

\paragraph{Global and Local Constraints of Inference Schedules.}
The inference algorithm alternates between an outer loop operation of sliding forward the denoising horizon and the inner loop refinement of the active window. The outer loop can run indefinitely for infinite autoregressive generation. The stride breaks the prediction horizon into $G=H/S$ groups of frames, where each group remains at the same noise level during inference and has $S$ frames. Each iteration of the outer loop symmetrically moves forward by $S$ frames, popping $S$ clean frames that have finished denoising, appending $S$ noisy frames at the end, and shifting the active window forward by $S$ frames to continue the rollout. In between these pop-shifts, the inner loop of the algorithm takes exactly $K=D/G$ denoising steps in order to satisfy the global consistency of $D$ steps per frame; intuitively, splitting the horizon into $G$ groups means that we need to take $K$ steps for each pop-shift.

\paragraph{Noise Level Schedule Matrix.}
The number of denoising steps $D$ determines a set of noise levels that each frame will visit during inference, $\{\sigma^{(d)} \}_{d=0}^D$ where $\sigma^{(0)} = 0$ and $\sigma^{(D)}=1$. These noise levels may be spaced out over the interval $[0, 1]$ in essentially any manner. Common approaches include spacing them out evenly, and extending the EDM inference schedule from static image generation to each chunk of video \citep{cachay2025elucidated}. We refer to choices of these visited noise levels as warps of the solver time $s$ (see Appendix~\ref{apd:sampler_warp_schedules}). In our exposition of the algorithm, for simplicity we can assume the noise levels are evenly spaced. At a given iteration during the steady state of the algorithm (after the bake-in period), the noise level transitions for $K$ inner steps for the active horizon are
\begin{align}
[
\sigma^{(K)} \cdot \mathbf{1},
\dots,
\sigma^{((G-1)K)} \cdot \mathbf{1},
\sigma^{(D)} \cdot \mathbf{1}
]
\quad
\rightarrow
\quad
[
\sigma^{(0)} \cdot \mathbf{1},
\dots,
\sigma^{((G-2)K)} \cdot \mathbf{1},
\sigma^{((G-1)K)} \cdot \mathbf{1}
],
\label{eq:rolling_noise_transition}
\end{align}
where each multiplication by $\mathbf{1}$ yields a group of $S$ frames that all have the same noise level. After that transition, the first group of $S$ frames have reached the clean boundary and are `emitted.' The horizon is then pop-shifted forward and a new group of fully noisy frames with noise level $\sigma^{(D)}$ is appended to recover the same steady-state noise geometry.

\paragraph{Bake-In Period for Initial Conditions.}
The rolling schedule described above is only well-defined once the local horizon has reached its steady-state triangular noise geometry. However, at initialization all $H$ frames in the horizon are sampled from pure noise,
so the noise schedule begins as
\begin{align}
[
\sigma^{(D)} \cdot \mathbf{1},
\sigma^{(D)} \cdot \mathbf{1},
\dots,
\sigma^{(D)} \cdot \mathbf{1}
],
\end{align}

The bake-in period constructs this steady-state geometry through $G-1$ masked bake-in iterations. At bake-in iteration $g \in \{1,\dots,G-1\}$, the model evaluates the full context and prediction horizon, but the denoising update is applied only to the first $g$ groups of the horizon. Thus, the first group receives $(G-1)K$ bake-in steps, the second group receives $(G-2)K$ steps, and so on, while the final group remains at pure noise. After bake-in, the horizon therefore has noise geometry
\begin{align}
[
\sigma^{(K)} \cdot \mathbf{1},
\sigma^{(2K)} \cdot \mathbf{1},
\dots,
\sigma^{((G-1)K)} \cdot \mathbf{1},
\sigma^{(D)} \cdot \mathbf{1}
].
\end{align}
The first ordinary inference iteration then applies $K$ steps to the entire horizon, bringing the first group to the clean boundary before it is emitted.

\paragraph{Autoregressive Context Conditioning and Full Algorithm.}
The final specification of the autoregressive inference algorithm determines how context $\seq{c}$ of length $C$ guides the generation of the active prediction horizon $\seq{x}$. Similar to History Guidance \citep{song2025historyguidedvideodiffusion}, we treat the conditioning frames to be unified within the same self-attention window as $\seq{x}$. The noise level conditioning is injected into the token space through AdaLN on a per frame basis \citep{peebles2023scalable, song2025historyguidedvideodiffusion}. 

The schedules explained in Equation~\ref{eq:rolling_noise_transition} define the transitions of noise levels in the active horizon. Denoting the first token in the active window to have index $1$, the full input to the model is formed from concatenating the context (denoted with indices for clarity) $\seq{c}_{-C+1:0}$ with the active horizon $\seq{x}_{1:H}$. The unified prediction horizon is $[\seq{c}_{-C+1:0}, \seq{x}_{1:H}]$, with the noise levels fixed to $\seq{0}_{-C + 1:0}$ for the context, and to $\seq{\sigma}_{1:H}$ for the active horizon. 

The outer loop and masked inner loop can be found in Algorithms~\ref{alg:outer_rolling_inference} and~\ref{alg:inner_denoising_loop}, respectively. In the algorithm, we define the function $\textsc{Emit}$ to slide the $S$ most recent cleanly generated frames to the context. During the bake-in period, only the prefix selected by the bake-in mask is updated. After the steady state is reached, each clean group is emitted and appended to the context, and the window slides forward by $S$ context frames, thereby replacing the oldest $S$ context frames.

\begin{algorithm}[t]
\caption{Rolling Inference with Masked Bake-In}
\label{alg:outer_rolling_inference}
\KwIn{
Model $f_{\mathrm{FM}}$, context length $C$, horizon $H$, stride $S$,
denoising steps $D$, inner steps $K$, initial context
$\seq{c}_{-C+1:0}$
}
$M \leftarrow H/S$\;
$\seq{x}_{1:H} \sim \mathcal N(0,I)$\;
$\seq{\sigma}_{1:H} \leftarrow \ones$\;
\tcp{Construct the initial denoising staircase, bake-in period}
\For{$b \leftarrow 1$ \KwTo $M-1$}{
    $\seq{m}_{1:H}
    \leftarrow
    [\ones_{bS},\,\zeros_{H-bS}]$\;
    $(\seq{x},\seq{\sigma})
    \leftarrow
    \textsc{MaskedInnerDenoise}
    (f_{\mathrm{FM}},\seq{c},\seq{x},\seq{\sigma},D,K,\seq{m})$\;
}
\tcp{Ordinary rolling generation}
\While{\textnormal{NotDone}}{
    $(\seq{x},\seq{\sigma})
    \leftarrow
    \textsc{MaskedInnerDenoise}
    (f_{\mathrm{FM}},\seq{c},\seq{x},\seq{\sigma},D,K,\ones_H)$\;
    $\textsc{Emit}(\seq{x}_{1:S})$\;
    $\seq{c}
    \leftarrow
    [\seq{c}_{S+1:C},\seq{x}_{1:S}]$\;
    $\seq{x}_{1:H-S} \leftarrow \seq{x}_{S+1:H}$\;
    $\seq{\sigma}_{1:H-S} \leftarrow \seq{\sigma}_{S+1:H}$\;
    $\seq{x}_{H-S+1:H}\sim\mathcal N(0,I)$\;
    $\seq{\sigma}_{H-S+1:H}\leftarrow\ones$\;
}
\end{algorithm}

\begin{algorithm}[t]
\caption{Masked Inner Denoising}
\label{alg:inner_denoising_loop}
\KwIn{
Model $f_{\mathrm{FM}}$, context $\seq{c}$, horizon $\seq{x}$,
noise levels $\seq{\sigma}$, denoising steps $D$, inner steps $K$,
update mask $\seq{m}$
}
\For{$k \leftarrow 1$ \KwTo $K$}{
    $\widehat{\seq{v}}
    \leftarrow
    f_{\mathrm{FM}}
    ([\seq{c},\seq{x}],[\zeros_C,\seq{\sigma}])$\;
    $\Delta\seq{\sigma}\leftarrow\frac{1}{D}\seq{m}$\;
    $\seq{x}\leftarrow
    \seq{x}-\Delta\seq{\sigma}\odot\widehat{\seq{v}}_{1:H}$\;
    $\seq{\sigma}\leftarrow
    \seq{\sigma}-\Delta\seq{\sigma}$\;
}
\Return{$(\seq{x},\seq{\sigma})$}\;
\end{algorithm}

\section{Online Replanning Inference Details}
\label{apd:online_replanning_inference}

\begin{figure}[t]
    \centering
    \includegraphics[width=\linewidth]{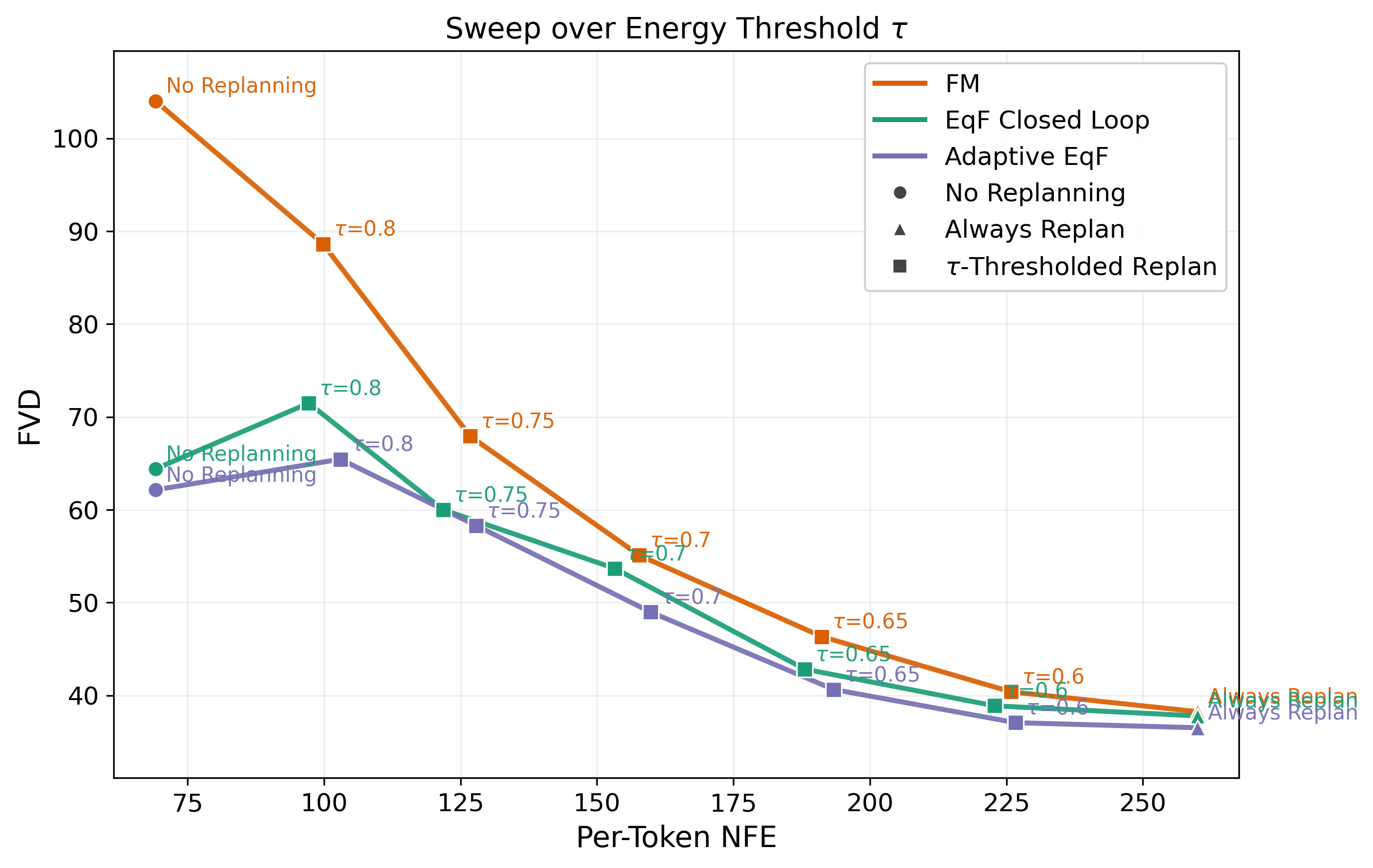}
    \caption{Online replanning on Minecraft. The replanning threshold $\tau$ controls when the retained future horizon is discarded and regenerated. $\tau=\infty$ never replans, while $\tau=0$ always replans. Intermediate thresholds trade off reuse of the current plan against responsiveness to new observations.}
    \label{fig:planning_threshold}
\end{figure}

We introduce an online replanning experiment intended to more closely reflect downstream applications of video generation models, such as world modeling in a closed loop environment. In such settings, an agent predicts a future video trajectory using a video generation model as a world model, turning the imagined future into an executable action \citep{du2023learning,chen2025large}. After execution, the agent receives a new observation from the environment (of a length less than the full predicted trajectory). The remaining prediction was generated before this new observation was made, and may therefore no longer be consistent with the new full ground truth state. An efficient planner should try to optimize the tradeoff between quality and cost for replanning. Reusing the previous plan reduces generation cost, but risks propagating predictions that were generated under stale observations. Regenerating the entire horizon is more responsive to new information, but requires additional sampling compute to regenerate all those frames.

We study this tradeoff with a new inference setting using a denoising-based compatibility score inspired by Adaptive Online Replanning with Diffusion Models, a paradigm in which a diffusion-based state-action plan can be re-evaluated as observations are streamed \citep{zhou2023adaptive}. In our setting, we isolate the video generation component, treating actions as fixed from the dataset. The compatibility score evaluates whether the existing prediction remains consistent with the newly observed context. When the score indicates that the prediction remains compatible, the planner retains the existing future and generates only the newly appended portion of the horizon. Otherwise, it discards the previous prediction and regenerates the entire future. We find that \EqF provides a stronger quality/compute tradeoff than noise-conditional Flow Matching across replanning thresholds. 

\subsection{Clean Buffer, Evaluation Buffer, and Compatibility Score}

The replanning setting differs from the standard rolling inference setting we use in other experiments. Typically, generated frames are both evaluated as predictions, and reused as context for subsequent generation. Here, the model maintains a predicted future video plan, commits a short prefix of that plan to the evaluation buffer, and then receives the corresponding ground-truth observations to use as context for subsequent generation. Below, we describe how the algorithm maintains the state, and how the compatibility score is defined and used to decide whether to regenerate.

\paragraph{Planning State and Observation Feedback.}
We use the notation of Appendix~\ref{apd:inference_sched_math}. For a generation round, let $C$ denote the context length, $H$ the prediction horizon, and $S$ the number of frames observed at each replanning step. $S$ also denotes the number of frames committed to the evaluation buffer at each step. For each outer replanning step $\ell$, let $r_\ell = \ell \cdot S$ denote the number of frames committed before step $\ell$. The algorithm maintains the evaluation buffer $\seq{x}^{\mathrm{eval}}$, which stores the predictions used for downstream evaluation, and a working buffer of clean frames $\seq{b}^{(\ell)}_{-C+1:H}$. The context
$\seq{b}^{(\ell)}_{-C+1:0}$ contains the most recent $C$ ground-truth observations, while
$\seq{b}^{(\ell)}_{1:H}$ contains the current predicted future plan of fully denoised frames. Here, clean means that they are not noisy frames; only the context is ground truth, while the future plan consists of denoised predictions. The conditioning sequence, which for our experiments is the ground truth action sequence, is denoted as $\seq{y}^{\mathrm{data}}$. 

At each step $\ell$, the first $S$ frames of the current working buffer are committed to the evaluation buffer,
\begin{align}
\seq{x}^{\mathrm{eval}}_{r_\ell+1:r_\ell+S}
=
\seq{b}^{(\ell)}_{1:S}.
\end{align}
The corresponding ground-truth observations are then revealed and written into the working buffer,
\begin{align}
\seq{b}^{(\ell)}_{1:S}
=
\seq{x}^{\mathrm{gt}}_{r_\ell+1:r_\ell+S}.
\end{align}
After this update,
$\seq{b}^{(\ell)}_{-C+1:S}$ forms an observed prefix of length $C+S$, while
$\seq{b}^{(\ell)}_{S+1:H}$ is the existing future plan generated before the new observations were available. The replanning decision evaluates whether this existing future prediction remains compatible with the updated observed prefix.

\paragraph{Compatibility Score.}
We measure whether the retained future remains consistent with the updated observed context. Specifically, we temporarily noise the existing prediction and evaluate the model’s denoising error while conditioning on the newly observed prefix.
At replanning step $\ell$, we evaluate at $M$ probe noise levels, $\sigma_m\in[0,1]$. For each probe level, the observed prefix remains clean and only the existing future is noised: 
\begin{align}
\seq{\sigma}^{(m)}_{-C+1:H}
=
\left[
\zeros_{C+S},
\sigma_m\ones_{H-S}
\right].
\end{align}
The clean working buffer $\seq{b}^{(\ell)}$ is used to construct both the temporary noised value and its corresponding velocity target. In the same denoising style as Equation~\ref{eq:data_gen_process}, for sampled Gaussian noise $\seq{\epsilon}^{(m)}$, we construct
\begin{align}
\seq{b}^{(\ell),\seq{\sigma}^{(m)}}_{-C+1:H}
&=
(\ones-\seq{\sigma}^{(m)}_{-C+1:H})\odot\seq{b}^{(\ell)}_{-C+1:H}
+
\seq{\sigma}^{(m)}_{-C+1:H}\odot\seq{\epsilon}^{(m)},
\\
\seq{v}^{(\ell,m)}
&=
\seq{\epsilon}^{(m)}-\seq{b}^{(\ell)}.
\end{align}

We define the compatibility score as
\begin{align}
\mathcal{S}^{(\ell)}
=
\frac{1}{M}
\sum_{m=1}^{M}
\operatorname{MSE}_{S+1:H}
\Bigl(
f_\theta\!\left(
\seq{b}^{(\ell),\seq{\sigma}^{(m)}}_{-C+1:H},
\seq{\sigma}^{(m)}_{-C+1:H}
\,;\,
\seq{y}^{\mathrm{data}}_{r_\ell-C+1:r_\ell+H}
\right),
\seq{v}^{(\ell,m)}
\Bigr),
\label{eq:planning_compatibility}
\end{align}
where $\operatorname{MSE}_{S+1:H}$ averages over all elements of the existing future frames. The denoising model is denoted by $f_\theta$. Lower values of $\mathcal{S}^{(\ell)}$ indicate that the existing prediction is more compatible with the newly observed prefix. In other words, when the denoising loss is low even with the new observed frames, it is a proxy for the model estimating that the frames are consistent with each other. The noise level condition is included in this equation as an input to the denoising model for clarity, as used by noise-conditional baselines like Flow Matching; for \EqF it is ignored.

\paragraph{Plan Reuse and Regeneration.}
A threshold $\tau$ converts the compatibility score into a replanning decision.

\emph{Plan reuse.}
If $\mathcal{S}^{(\ell)}<\tau$, the existing future prediction is considered compatible with the new observations. The existing predicted $H-S$ frames are shifted forward, $S$ fresh noisy frames are appended to restore the full prediction horizon, and only these newly appended frames are denoised. The existing prediction therefore remains unchanged and its next $S$ frames are committed at the next replanning round.

\emph{Plan regeneration.}
If $\mathcal{S}^{(\ell)}\geq\tau$, the existing prediction is considered incompatible with the new observations. The retained $H-S$ frames are discarded and replaced by full noise, $S$ additional noisy frames are appended, and the complete $H$-frame future horizon is regenerated.

The compatibility probes are temporary and do not modify the clean working buffer. The complete procedure is provided in Algorithms~\ref{alg:closed_loop_planning},\ref{alg:planning_energy},\ref{alg:masked_planning_denoise}. Note that the function \textsc{MaskedDenoise} in Algorithm~\ref{alg:masked_planning_denoise} is replaced by a specialized \EqF sampler when appropriate, restricted to the frames indicated by the mask $\seq{m}_{1:H}$.

\subsection{Planning on Minecraft}
\label{apd:planning_minecraft_results}

\paragraph{Experimental Setup.}
We evaluate the online replanning procedure on Minecraft using the dataset and model settings described in Appendix~\ref{apd:mc_expt_details_settings}. We compute the compatibility score using $M=5$ evenly spaced probe noise levels, which we found to provide stable estimates comparable to larger values of $M$. We sweep the replanning threshold $\tau$ over a range calibrated to the observed compatibility scores and include the two boundary values of $\tau$. Setting $\tau=0$ causes the model to replan after every newly observed stride, while $\tau=\infty$ always retains the previous prediction and generates only the newly appended frames. Intermediate thresholds adaptively trade off between these two behaviors.

\paragraph{Quality/Compute Tradeoff.}
Figure~\ref{fig:planning_threshold} reports prediction quality as a function of replanning cost. Specifically, we measure the average per-frame number of function evaluations (NFE) to achieve an overall FVD. Always replanning produces the strongest prediction quality, as the complete future horizon is regenerated after every new observation, but requires the most computation. Never replanning minimizes regeneration cost but can retain predictions that are no longer consistent with the observed trajectory. Intermediate thresholds provide a smooth tradeoff by regenerating the horizon only when its compatibility score exceeds $\tau$. Across this sweep, Budget-Adaptive \EqF achieves the strongest quality/compute tradeoff, while the Flow Matching baseline exhibits the weakest tradeoff.

\paragraph{Mixed-Context Prediction.}
The retained horizon can contain frames generated before the newest observations were available, while newly appended frames are generated after incorporating those observations. The resulting prediction may therefore combine frames produced under different context states and can become internally inconsistent. This represents an out-of-distribution denoising setting because both Flow Matching and \EqF are trained on noisy versions of temporally consistent videos. 
We hypothesize that \EqF is more robust to this mismatch because it does not require an externally supplied noise level to accurately describe the effective state of each frame. In contrast, a supplied noise level that does not match the noisy video state may compound the extent to which the input is out-of-distribution, making Flow Matching's denoising error higher. This interpretation is consistent with the noise condition mismatch analyzed in Appendix~\ref{sec_bowls}.

\paragraph{Alternative Renoising Policies.}
Following prior work on adaptive online replanning \citep{zhou2023adaptive}, we also evaluated policies that partially renoise the existing prediction rather than either preserving it or fully restarting the horizon. These included pyramidal profiles, in which the noise level increases across the retained horizon, and flat profiles, in which all retained frames are renoised to the same intermediate level. We additionally considered hybrid policies interpolating between pyramidal renoising and full restart. These variants did not consistently improve over full restart and were generally less stable across thresholds. We therefore report the simpler restart-based replanning procedure, which provides the clearest quality/compute tradeoff.

We would like to emphasize here that due to the action condition coming from the ground truth data instead of being predicted, the discrepancy between the predicted future and the ground truth observation may be less than what it would be in a more realistic full planning setting. We hope this work can further inspire future work that explores these quality/cost tradeoffs in more realistic world modeling settings.


\begin{algorithm}[t]
\caption{Closed-Loop Planning Inference with Observation Feedback. \textsc{CompatibilityScore} is implemented in Algorithm~\ref{alg:planning_energy}. \textsc{MaskedDenoise} is implemented in Algorithm~\ref{alg:masked_planning_denoise}.}
\label{alg:closed_loop_planning}

\KwIn{
Model $f_\theta$, ground-truth video $\seq{x}^{\mathrm{gt}}_{-C+1:R}$,
conditioning sequence $\seq{y}^{\mathrm{data}}$, context length $C$, horizon $H$,
observation stride $S$, rollout length $R$ (divisible by $S$), denoising steps $D$, compatibility threshold $\tau$,
probe noise levels $\{\sigma_{\mathrm{probe}}^{(m)}\}_{m=1}^M$
}

\KwOut{Evaluation buffer $\seq{x}^{\mathrm{eval}}_{1:R}$}

$r \leftarrow 0$\;
$\seq{x}^{\mathrm{eval}}_{1:R} \leftarrow \emptyset$\;

\tcp{Initialize the working buffer with ground truth context and a noisy future}
$\seq{b}_{-C+1:0} \leftarrow \seq{x}^{\mathrm{gt}}_{-C+1:0}$\;
$\seq{b}_{1:H} \sim \mathcal N(0,I)$\;

\tcp{Bootstrap an initial clean plan over the full horizon, initialize active mask}
$\seq{m}_{1:H} \leftarrow \ones$\;
$\seq{b}_{-C+1:H}
\leftarrow
\textsc{MaskedDenoise}
(f_\theta,\seq{b}_{-C+1:H},\seq{m}_{1:H},D,
\seq{y}^{\mathrm{data}}_{r-C+1:r+H})$\;

\While{$r < R$}{

    \tcp{Commit generated frames to the evaluation buffer}
    $\seq{x}^{\mathrm{eval}}_{r+1:r+S}
    \leftarrow
    \seq{b}_{1:S}$\;

    \tcp{Reveal ground-truth observations and overwrite the clean buffer}
    $\seq{b}_{1:S}
    \leftarrow
    \seq{x}^{\mathrm{gt}}_{r+1:r+S}$\;

    \tcp{Score whether the retained future plan is compatible with the new observations}
    $\mathcal E
    \leftarrow
    \textsc{CompatibilityScore}
    (f_\theta,\seq{b}_{-C+1:H},S,
    \{\sigma_{\mathrm{probe}}^{(m)}\}_{m=1}^M,
    \seq{y}^{\mathrm{data}}_{r-C+1:r+H})$\;

    \tcp{Slide the clean ground-truth context}
    $\seq{b}_{-C+1:0}
    \leftarrow
    [\seq{b}_{-C+1+S:0},\seq{x}^{\mathrm{gt}}_{r+1:r+S}]$\;

    \tcp{Slide the existing future prediction}
    $\seq{b}_{1:H-S}
    \leftarrow
    \seq{b}_{S+1:H}$\;

    $r \leftarrow r+S$\;

    \eIf{$\mathcal E < \tau$}{

        \tcp{Reuse the existing prediction; only generate the newly appended stride}
        $\seq{b}_{H-S+1:H} \sim \mathcal N(0,I)$\;
        $\seq{m}_{1:H}
        \leftarrow
        [\zeros_{H-S},\ones_{S}]$\;

    }{
        \tcp{Discard and regenerate the complete future horizon}
        $\seq{b}_{1:H}\sim\mathcal N(0,I)$\;
        $\seq{m}_{1:H}\leftarrow\ones$\;
    }

    \tcp{Denoise only the frames indicated by the mask}
    $\seq{b}_{-C+1:H}
    \leftarrow
    \textsc{MaskedDenoise}
    (f_\theta,\seq{b}_{-C+1:H},\seq{m}_{1:H},D,
    \seq{y}^{\mathrm{data}}_{r-C+1:r+H})$\;
}

\Return{$\seq{x}^{\mathrm{eval}}_{1:R}$}\;

\end{algorithm}

\begin{algorithm}[t]
\caption{Compatibility Score from a Clean Working Buffer}
\label{alg:planning_energy}

\KwIn{
Model $f_\theta$, clean buffer $\seq{b}_{-C+1:H}$, observation stride $S$,
probe noise levels $\{\sigma_{\mathrm{probe}}^{(m)}\}_{m=1}^M$,
conditions for the active window $\seq{y}_{-C+1:H}$
}
\KwOut{Compatibility score $\mathcal E$}

$\mathcal E \leftarrow 0$\;

\For{$m \leftarrow 1$ \KwTo $M$}{

    \tcp{The clean buffer is not modified; this is only a temporary noised probe}
    $\seq{\epsilon}^{(m)}_{S+1:H}\sim\mathcal N(0,I)$\;

    \tcp{Observed prefix stays clean; only retained predictions are noised}
    $\seq{\sigma}^{(m)}_{-C+1:S}\leftarrow\zeros_{C+S}$\;
    $\seq{\sigma}^{(m)}_{S+1:H}
    \leftarrow
    \sigma_{\mathrm{probe}}^{(m)}\ones_{H-S}$\;

    $(\seq{b}^{\seq{\sigma}^{(m)}})_{-C+1:S}
    \leftarrow
    \seq{b}_{-C+1:S}$\;

    $(\seq{b}^{\seq{\sigma}^{(m)}})_{S+1:H}
    \leftarrow
    (\ones-\seq{\sigma}^{(m)}_{S+1:H})
    \odot
    \seq{b}_{S+1:H}
    +
    \seq{\sigma}^{(m)}_{S+1:H}
    \odot
    \seq{\epsilon}^{(m)}_{S+1:H}$\;

    $\seq{v}^{(m)}_{S+1:H}
    \leftarrow
    \seq{\epsilon}^{(m)}_{S+1:H}
    -
    \seq{b}_{S+1:H}$\;

    $\seq{u}^{(m)}_{-C+1:H}
    \leftarrow
    f_\theta(
    \seq{b}^{\seq{\sigma}^{(m)}}_{-C+1:H},
    \seq{\sigma}^{(m)}_{-C+1:H}
    \,;\,
    \seq{y}_{-C+1:H}
    )$\;

    \tcp{MSE averaged over all temporal, spatial, and channel dimensions in the existing future}
    $\mathcal E
    \leftarrow
    \mathcal E
    +
    \operatorname{MSE}\!\left(
        \seq{u}^{(m)}_{S+1:H},
        \seq{v}^{(m)}_{S+1:H}
    \right)$\;
}

$\mathcal E \leftarrow \frac{1}{M}\mathcal E$\;

\Return{$\mathcal E$}\;

\end{algorithm}

\begin{algorithm}[t]
\caption{Masked future denoising for planning. Initialized as fixed-step Euler schedule for simplicity.}
\label{alg:masked_planning_denoise}

\KwIn{
Model $f_\theta$, buffer $\seq{b}_{-C+1:H}$, binary active mask $\seq{m}_{1:H}$,
denoising steps $D$, conditioning window $\seq{y}_{-C+1:H}$
}

\tcp{Clean context and inactive planned frames have zero noise}
$\seq{\sigma}_{-C+1:0}\leftarrow\zeros_C$\;
$\seq{\sigma}_{1:H}\leftarrow\seq{m}_{1:H}$\;

$\Delta\seq{\sigma}_{1:H}
\leftarrow
\frac{1}{D}\seq{m}_{1:H}$\;

\For{$d \leftarrow 1$ \KwTo $D$}{

    $\seq{u}_{-C+1:H}
    \leftarrow
    f_\theta(
    \seq{b}_{-C+1:H},
    \seq{\sigma}_{-C+1:H}
    \,;\,
    \seq{y}_{-C+1:H}
    )$\;

    \tcp{Only noised future frames are updated}
    $\seq{b}_{1:H}
    \leftarrow
    \seq{b}_{1:H}
    -
    \Delta\seq{\sigma}_{1:H}
    \odot
    \seq{u}_{1:H}$\;

    $\seq{\sigma}_{1:H}
    \leftarrow
    \seq{\sigma}_{1:H}
    -
    \Delta\seq{\sigma}_{1:H}$\;
}

\Return{$\seq{b}_{-C+1:H}$}\;

\end{algorithm}



\end{document}